\documentclass[10pt,twocolumn]{article}
\usepackage[letterpaper,textwidth=6.75in,textheight=9.25in,columnsep=0.25in]{geometry}
\usepackage{amsmath,amssymb,amsthm}
\usepackage{graphicx}
\usepackage{booktabs}
\usepackage{placeins}
\usepackage{xcolor}
\usepackage{xspace}
\usepackage{colortbl, tcolorbox}
\usepackage[round]{natbib}
\usepackage{microtype}
\usepackage{hyperref}
\hypersetup{
  pdftitle={When Repeated Measurements Hurt Tabular Foundation Models},
  pdfauthor={Yixiang Liu, Ruizhe Deng, Bibhas Chakraborty, Qiong Zhang, Yan Shuo Tan}}

\hypersetup{hidelinks}

\newcommand{\E}{\mathbb{E}}

\newcommand{\Var}{\mathrm{Var}}

\newcommand{\Ybar}{\bar{Y}}
\newcommand{\Single}{\textit{Single}\xspace}
\newcommand{\Expanded}{\textit{Expanded}\xspace}
\newcommand{\Mean}{\textit{Mean}\xspace}
\newtheorem{proposition}{Proposition}
\newtheorem{lemma}{Lemma}
\newtheorem{theorem}{Theorem}

\newtheorem{remark}{Remark}

\begin{document}
\title{\bfseries When Repeated Measurements Hurt Tabular Foundation Models}
\author{Yixiang Liu$^{1,*}$ \quad Ruizhe Deng$^{1,*}$ \quad Bibhas Chakraborty$^{2}$ \quad Qiong Zhang$^{3}$ \quad Yan Shuo Tan$^{1}$\\[1.2ex]
{\normalsize $^{1}$National University of Singapore \qquad $^{2}$Duke-NUS Medical School}\\[0.2ex]
{\normalsize $^{3}$Mohamed bin Zayed University of Artificial Intelligence}\\[0.8ex]
{\normalsize $^{*}$Equal contribution}}
\date{}
\maketitle

\begin{abstract}
Tabular foundation models (TFMs) achieve state-of-the-art performance on IID benchmarks, but how do they respond to repeated measurements at the same inputs? Across six TFMs, synthetic and semi-synthetic experiments show that additional measurements can improve predictions at measured inputs while degrading them elsewhere.
Harm is most frequent at low noise, and irrelevant features amplify it in several models; real repeated-measurement datasets show similar patterns. Statistically, the IID rule targeted by PFN pretraining treats each repeated input as fresh evidence about the input distribution. When the prior links this distribution to the response function, the resulting protocol mismatch can outweigh the benefit of additional measurements, even under exact inference. Computationally, interventions identify row-count dependence in TabPFN-v3's attention scaling as a contributor to harm. Resetting the count largely reverses distortion in a noiseless linear diagnostic and attenuates amplification by irrelevant features. Finally, we evaluate ways to retain the benefits of repeated measurements while limiting their harm. Averaging measurements usually outperforms supplying them separately at larger repetition counts, whereas count reset can suppress both harm and useful gains.
Code, saved results and reproduction instructions are available at
\url{https://github.com/drizmiz/PFN-Repeated-Measurement-Supp}.
\end{abstract}

\section{Introduction}
\label{sec:intro}

Learning from a new tabular dataset increasingly need not begin with fitting a new model.
Tabular foundation models (TFMs) instead use the labeled table as context for a pretrained predictor, achieving competitive or state-of-the-art accuracy on small- and medium-sized tasks without necessarily requiring parameter updates or task-specific tuning \citep{hollmann2025nature,tabpfn3}.
Strong results on benchmarks such as TabArena \citep{erickson2025tabarena} have been followed by claims of broader applicability, including the latest TabPFN release's performance under temporal and grouped splits \citep{tabpfn35}.
As these models become alternatives to conventional learners, a central question is \emph{how their predictive performance depends on the data-collection protocol.}
Synthetic pretraining for many TFMs includes tasks whose rows are independent and identically distributed (IID) samples from a latent joint distribution on covariates and responses \citep{hollmann2023tabpfn,qu2025tabicl,tabpfn3}.
Scientific data often follow a different observation protocol: several responses are recorded at the same covariate vector \citep{box2005statistics,ankenman2010stochastic,binois2018hetgp}.
We call each distinct covariate vector a \emph{site} and each recorded response a \emph{reading}.
When entered individually, readings from the same site occupy separate table rows with identical covariates.
Oversampling with replacement also produces repeated rows.

Classical regression methods routinely accommodate repeated measurements, either by directly fitting the individual readings or by fitting weighted site means \citep{draper1998applied,ankenman2010stochastic,binois2018hetgp}.
For a class of squared-loss algorithms, including empirical risk minimization and gradient boosted trees under the conditions in Section~\ref{sec:equivalence}, fitting individual readings and fitting count-weighted site means yield the same predictor.
TFMs do not have weight inputs, let alone any guarantee of equivalence.
In fact, when investigating how TabPFN could be used to predict algorithm-runtime distributions,
\citet{ibrahimli2026runtime} found that supplying additional readings as context rows to TabPFN could decrease its predictive performance.
One possible explanation for this failure mode is that repeated readings can share unobserved unit-specific effects, making them dependent even after conditioning on the recorded covariates.
In this paper, we investigate how the failure may persist even without such dependence.
Specifically, our central question is:
\begin{center}
\begin{tcolorbox}[colframe=black!50!white,
    colback=gray!5!white,
    boxrule=0.5mm,
    arc=5pt,width=\columnwidth]
\centering
\textit{When do additional independent readings help or hurt TFM predictions, and what mechanisms determine that tradeoff?}
\end{tcolorbox}
\end{center}
We first characterize the extent and conditions of repetition harm (Section~\ref{sec:empirical}).
We compare six TFMs with conventional learners on synthetic, semi-synthetic and real datasets, keeping sites fixed as the readings per site increase.
Repeated readings often improve prediction at measured sites while worsening it at fresh inputs.
 
On the illustrative quadratic task, all six TFMs gain from the first few readings and then part ways. TabPFN-v3, TabPFN-v3.5 and LimiX-2 are worse than with one reading from eight readings on and end three to five times worse at 256; TabICLv2 and TabDPT cross only at thirty-two and end less than twice as bad. TabFM
  retains a benefit through 256 readings, demonstrating that the response is model-dependent.
Harm is generally more frequent at low noise, where the absolute reduction in uncertainty about the site means is smaller; increasing site coverage can mitigate it, although the dependence on site count is not uniformly monotone.
Appending irrelevant features amplifies the relative harm for TabPFN-v3, LimiX-2 and TabDPT, even though the regression function gains no active dimension.
Results on ten real datasets broadly support the controlled pattern: repetition often hurts with low within-site variation, while higher-variation datasets show substantial benefits.

We investigate reasons for this failure at two levels.
\emph{Statistically}, we examine the mismatch between the inference rule targeted by pretraining and the repeated-measurement observation scheme (Section~\ref{sec:protocolmismatch}).
Many TFMs are prior-data fitted networks (PFNs), which are trained to approximate Bayesian posterior predictive distributions under a prior over the covariate-response joint distributions $p(X,Y)$ and IID samples conditioned on $p(X,Y)$ \citep{muller2022,hollmann2025nature}.
The posterior appropriate to repeated measurements should count each site's covariates once, whereas the IID posterior targeted by pretraining instead counts them once per reading.
Under the commonly used structural causal model priors on $p(X,Y)$, the covariate distribution $p(X)$ and conditional response distribution $p(Y\mid X)$ can be coupled, so this excess covariate evidence also affects the posterior predictive distribution \citep{scholkopf2012causal}.
We further quantify the tradeoff between this misspecification and the gain in precision from repeated measurements.

\emph{Computationally}, we investigate how count-dependent attention contributes to repetition harm in trained TabPFN-v3 (Section~\ref{sec:computation}). Its learned attention scaling takes the number of context rows as an input, which increases with additional readings even when no sites are added. Resetting this count while holding the table and model weights fixed largely reverses the prediction distortion in a noiseless linear diagnostic. Count changes also propagate through the representations used by subsequent attention operations. With noisy readings, repetition concentrates test-query attention on fewer sites, but restoring the one-reading query-attention weights helps only some targets; separate embedding-stage resets support a contribution from the computations that construct context representations. Count reset also attenuates the amplification of harm by irrelevant features, although a residual dimension effect remains.

Finally, we evaluate ways to retain the benefits of repeated measurements while limiting their harm (Section~\ref{sec:remedies}). Averaging responses at each measured input improves on using one response in the settings examined, and increasingly outperforms supplying all responses as separate rows as repetition grows. Resetting the internal row count reduces low-noise harm but also weakens the benefits of additional measurements at higher noise. Averages alone omit information about unequal measurement precision and variation in future responses, motivating summaries that also retain repetition counts and within-input variation. In addition, in three tested TabPFN-v3 fine-tuning recipes, improvements when using all responses as separate rows come at the expense of accuracy on averaged responses (Section~\ref{sec:discussion}).

The central finding is that information from additional independent readings can be beneficial at measured sites yet harmful to predictions elsewhere.
Aggregation comparisons and count interventions show that how a learner represents and processes repeated readings matters alongside their information content.

\section{Setup and Representation Equivalence}
\label{sec:setup}

\phantomsection
\label{sec:protocol}

\textbf{Observation protocol and prediction target.}
At $m$ distinct sites $X_1,\ldots,X_m\in\mathbb R^d$ we observe
\mbox{$Y_{ij}=f(X_i)+\sigma_i\epsilon_{ij}$} for $j=1,\ldots,r_i$.
Conditional on the function and sites, the errors are independent with mean zero and
variance one. There are $n=\sum_i r_i$ readings; balanced counts $r_i=r$ give $n=mr$.
The noise variance of a site mean is $\sigma_i^2/r_i$.
We use mean squared error (MSE) to measure how additional readings at fixed sites affect prediction of $f(x_*)$ at fresh inputs.

\phantomsection
\label{sec:equivalence}
\textbf{How conventional learners use repeated readings.}
Many conventional learners accept sample weights. Let $\mathcal S=\{(x_k,y_k)\}_{k=1}^{|\mathcal S|}$
be a dataset and $w\in(0,\infty)^{|\mathcal S|}$ a vector of sample weights; an ordinary fit uses
$w=\mathbf 1$.
A learning algorithm that accepts sample weights maps $(\mathcal S,w)$ and a random seed $\omega$
to a predictor $\mathcal A(\mathcal S,w,\omega)$. The readings give two datasets:
$\mathcal D=\{(X_i,Y_{ij})\}_{i\le m,\,j\le r_i}$, which contains every reading, and
$\bar{\mathcal D}=\{(X_i,\Ybar_i)\}_{i\le m}$, which contains the site means
$\Ybar_i=r_i^{-1}\sum_jY_{ij}$. With $\mathbf r=(r_1,\ldots,r_m)$, the inputs
$(\mathcal D,\mathbf 1)$ and $(\bar{\mathcal D},\mathbf r)$ both place total weight $r_i$
at site $X_i$.

Under squared loss, a site's readings enter through their mean and count, up to a constant
independent of the predictor. To make this precise, fix a function $\varphi:\mathbb R^d\to[0,\infty)$ that localizes the fit
and a predictor $g$, and define the weighted squared loss
\mbox{$L_\varphi(g;\mathcal S,w)=\sum_k w_k\,\varphi(x_k)\{y_k-g(x_k)\}^2$}.
A region $A$ of the covariate space corresponds to $\varphi=\mathbf 1_A$, a kernel
smoother at $x$ to $\varphi=K_h(x-\cdot)$, and a global fit to $\varphi\equiv1$.
Expanding about each site mean gives
\begin{equation}
\label{eq:regional}
L_\varphi(g;\mathcal D,\mathbf 1)=L_\varphi(g;\bar{\mathcal D},\mathbf r)+\sum_{i}\varphi(X_i)W_i,
\end{equation}
where $W_i=\sum_j(Y_{ij}-\Ybar_i)^2$.
Since the last sum does not depend on $g$, we immediately see that when $\mathcal{A}$ is empirical risk minimization over a function class, we have the invariance $\mathcal A(\mathcal D,\mathbf 1,\omega)=\mathcal A(\bar{\mathcal D},\mathbf r,\omega)$.
In fact, this algorithmic invariance can be stated more broadly.

\begin{proposition}[Algorithmic invariance]
\label{prop:reduction}
Let $\mathcal A$ accept sample weights, and suppose it depends on its input only through
the total weight at each distinct covariate value and through differences
$L_\varphi(g;\cdot)-L_\varphi(g';\cdot)$ for nonnegative functions $\varphi$ of the
covariates and predictors $g,g'$. Then
$\mathcal A(\mathcal D,\mathbf 1,\omega)=\mathcal A(\bar{\mathcal D},\mathbf r,\omega)$.
\end{proposition}

The conditions hold for squared-error regression trees and gradient boosting with
weighted leaf values, no row subsampling and leaf-size constraints measured in total
weight, and for Nadaraya--Watson and local polynomial regression with a fixed bandwidth
(Appendix~\ref{app:invariance}).
Applied inside the Gaussian likelihood, \eqref{eq:regional} gives a GP counterpart.
With a fixed kernel and known noise variances $\sigma_i^2$, the GP posterior given
$\mathcal D$ equals that given $\bar{\mathcal D}$ with noise variances
$\sigma_i^2/r_i$, which play the role of the weights $\mathbf r$
(Appendix~\ref{app:precision}).
On the other hand, the conditions fail for a forest that bootstraps rows and
for any rule that counts rows, such as an optimizer whose number of steps grows with
the table.

While Proposition~\ref{prop:reduction} does not guarantee that extra readings improve
prediction, it explains why they usually do for these learners: under the reduction,
extra readings reach the learner only as more precise site means.
In fact, for balanced counts, any linear smoother
$\hat f(x)=\sum_i\ell_i(x)\Ybar_i$ whose weights depend only on the sites has nonincreasing
squared-error risk with $r$: its squared bias is unchanged and its variance is
$r^{-1}\sum_i\ell_i(x)^2\sigma_i^2$. The decrease is strict when
$\sum_i\ell_i(x)^2\sigma_i^2>0$. This includes a tree with a fixed partition and a
Nadaraya--Watson estimator with a fixed bandwidth.

A TFM has no weight input, so it receives the readings either as separate rows,
$(\mathcal D,\mathbf 1)$, or as unweighted means, $(\bar{\mathcal D},\mathbf 1)$,
which fails the condition of Proposition~\ref{prop:reduction}. Nothing then ensures
that extra readings act only through more precise site means, and
Section~\ref{sec:empirical} shows that for several TFMs they do not: predictions
improve at the measured sites but worsen between them.

\section{When Repetition Hurts}
\label{sec:empirical}

\phantomsection
\label{sec:representations}
\textbf{Experimental comparisons.}
We compare two fits of the same readings: one reading per site, the first (\Single,
$(\mathcal D_1,\mathbf 1)$), and all readings as separate rows (\Expanded,
$(\mathcal D,\mathbf 1)$). An increase in MSE of \Expanded over \Single is
\emph{repetition harm}. One mean per site (\Mean, $(\bar{\mathcal D},\mathbf 1)$) is
an averaging alternative further investigated in Section~\ref{sec:remedies}.

\textbf{Models and evaluation.}
We study TabPFN-v3 \citep{tabpfn3}, TabPFN-v3.5 \citep{tabpfn35}, TabFM \citep{tabfm}, TabDPT
\citep{tabdpt}, LimiX-2 \citep{limix} and TabICLv2 \citep{tabiclv2}, each in its shipped
configuration with no tuning and with the whole table in context, beside a Gaussian process (GP)
fitted on the rows, gradient boosting, a random forest and an MLP. Synthetic tasks use six one-dimensional functions
and three multivariate families, each standardised to unit variance; semi-synthetic tasks
keep the covariates of five real datasets and generate independent readings around a teacher
that is not a TFM and was fitted on disjoint data. Ten real datasets with repeated
measurements then test the same comparison on actual readings (Section~\ref{sec:real}).
Checkpoints, ensemble sizes and context handling are recorded in Appendix~\ref{app:protocol}.

\subsection{TFM Deterioration under Repetition}
\label{sec:curves}

\begin{figure}[ht]
\centering
\includegraphics[width=0.9\columnwidth]{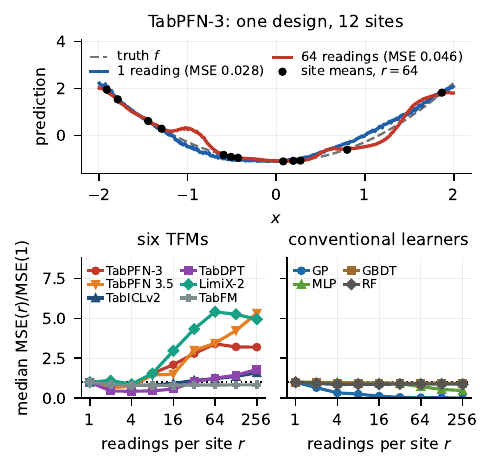}
\caption{More readings, worse predictions between the sites. \textbf{Top:} TabPFN-v3 on one
design (quadratic, twelve sites, noise SD $0.10$) with one reading per site (blue) and $64$
per site as rows (red); dots are the site means. \textbf{Bottom:} median over twenty
seeds of each seed's MSE$(r)/$MSE$(1)$ for the six TFMs and the four conventional learners. Raw errors and
all six TFMs' fits are in Appendix Table~\ref{tab:failure} and Figure~\ref{fig:fits}.}
\label{fig:failure}
\end{figure}

On one quadratic design (Figure~\ref{fig:failure}), every learner receives the same \Expanded
table as $r$ grows from one to $256$; we report the median over twenty seeds of each seed's
MSE$(r)/$MSE$(1)$, measured against the true $f$ on a $400$-point grid
(Appendix~\ref{app:protocol}).
At four readings all six TFMs are below their one-reading error ($0.43$--$0.88$).
TabPFN-v3, TabPFN-v3.5 and LimiX-2 then rise above it from eight readings on and end $3.2$ to
$5.3$ times higher at $256$; TabICLv2 and TabDPT cross only at thirty-two and end at $1.6$ and
$1.8$. TabFM, whose one-reading error is the highest of the ten learners, keeps a gain ($0.84$
at $256$). The conventional learners do not deteriorate: the GP falls to $0.02$ of its
one-reading error and the MLP to $0.48$, while boosting and the random forest gain a little
($0.94$ and $0.88$).

With $64$ readings TabPFN-v3 still passes through the site means but bends away from the truth
between them (Figure~\ref{fig:failure}, top).

\subsection{Noise and Coverage Reduce Repetition Harm; Irrelevant Features Amplify It}
\label{sec:noisecoverage}

\begin{figure*}[t]
  \centering
  \includegraphics[width=\textwidth]{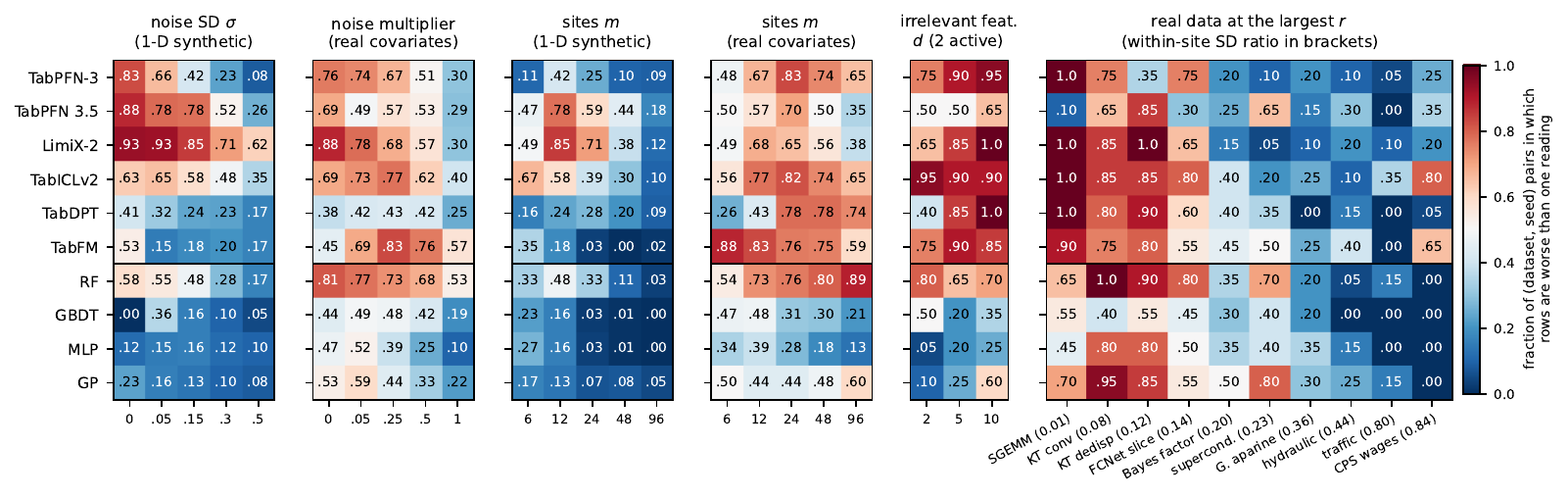}
  \vspace{-3ex}
  \caption{Fraction of the twenty paired seeds, pooled over functions or covariate datasets, in
  which \Expanded at $r=64$ has higher fresh-input MSE than
  \Single (TFMs above the line, conventional learners below; RF: random forest, GBDT: gradient boosting). Synthetic tasks are scored on a grid
  (one dimension) or fresh draws, semi-synthetic tasks on held-out rows. The irrelevant-features
  panel fixes a function of two coordinates and adds ignored ones. Real data: twenty splits at the
  largest count $R$, scored against held-out site means and ordered by within-site SD ratio
  (Section~\ref{sec:real}).}
  \label{fig:overview}
  \vspace{-1ex}
\end{figure*}

We compare \Expanded at $r=64$ with \Single; the one-dimensional default uses twelve random
sites and noise SD $0.15$, and on real covariates the noise SD is $0.25$ times the teacher's
residual SD (Appendix~\ref{app:protocol}). In the default one-dimensional setting, all six TFMs
improve measured-site MSE in all $120$ function--seed pairs, but the same predictors worsen MSE
at gap midpoints in $16$ (TabDPT) to $91$ (LimiX-2) of them (Appendix Table~\ref{tab:query}).

Raising the noise SD from $0$ to $0.5$ lowers the harmed fraction for all six TFMs.
Pooled over the six one-dimensional functions, the six models' fractions range from $0.41$--$0.93$
with noiseless copies to $0.08$--$0.62$ at SD $0.5$; TabPFN-v3's falls from $0.83$ to $0.08$,
while LimiX-2 remains harmed in most pairs at the highest noise tested. The fraction counts
any increase, however small: TabFM's median copy/\Single ratios stay within $1\%$ of one
(Appendix Table~\ref{tab:copy-contrast}).

Starting from the default twelve sites, denser coverage reduces repetition harm on the
quadratic target: the harmed-seed count is nonincreasing across $m=12,24,48,96$ for all
six TFMs (Appendix Table~\ref{tab:coverage}); for TabPFN-v3 the harmed counts fall from $16/20$
to $8/20$, $1/20$ and $0/20$. On real covariates in eight to eleven dimensions,
$96$ sites still leave $0.35$--$0.74$ of dataset--seed pairs harmed across the six TFMs
(Figure~\ref{fig:overview}).

\phantomsection
\label{sec:padding}
Irrelevant features act in the opposite direction. One function of two coordinates is kept
fixed, together with its sites, readings and target, and three and then eight independent
coordinates that it ignores are appended, so the prediction problem gains no active
dimension. TabPFN-v3's harmed fraction rises from $0.75$ to $0.95$ (Figure~\ref{fig:overview}),
and its median paired \Expanded/\Single ratio at $64$ readings from $1.60$ to $3.10$ and $3.15$;
LimiX-2 goes from $1.11$ to $3.06$ and TabDPT from a gain ($0.97$) to a loss ($1.72$), while
TabPFN-v3.5 shows no amplification ($1.00$, $1.00$, $1.04$; Appendix
Table~\ref{tab:paddingfull}). \Mean is not amplified in the same way: its error also
grows with the added columns, but it is worse than \Single in at most four of twenty seeds in
every model, which ties the amplification to the rows.

Conventional learners are harmed less often. Counting a design as harmed when \Expanded loses
in at least $11$ of its $20$ seeds, boosting, the MLP and the GP have no harmed
one-dimensional designs out of $60$ (the forest $10$); across all $114$ designs boosting and
the MLP are harmed in $6$ and $8$, against $25$--$73$ for the TFMs, and the forest, which
bootstraps rows (Section~\ref{sec:equivalence}), in $53$ (Appendix Table~\ref{tab:harmcounts}).
The GP is harmed in $19$ of the $45$ real-covariate designs, within the TFM range of $17$--$34$.

\subsection{Real-data Harm Concentrates at Low Within-site SD Ratios}
\label{sec:real}

We use ten public datasets in which a recorded covariate vector was measured more than once
(Appendix~\ref{app:realdata}): five computer experiments that repeat a run at a fixed
configuration \citep{paredes2017sgemm,vanwerkhoven2019kerneltuner,willemsen2025benchmarkhub,klein2019fcnet,franck2020bayes},
three physical or biological experiments
\citep{cabanne1999galium,helwig2015hydraulic,hamidieh2018superconductor} and two surveys
\citep{davison1997bootstrap,bierens2001cps}.

Since no dataset records a grouping beyond the site, such as a lab, batch or
subject shared across sites, we believe unit-specific effects should be relatively small
(Appendix~\ref{app:realdata}). Each dataset is split by site into
context and test halves twenty times, the readings per context site are equalised by
subsampling, and the target at a test site is the mean of all its readings. The datasets are
ordered by their within-site SD ratio, the pooled within-site SD divided by the SD of all
readings; a low ratio means the readings at a site agree closely relative to the spread of all
readings.

The controlled experiments predict harm where one reading is already precise relative to the
differences between sites, and that is where it appears (Figure~\ref{fig:overview}, final
panel). In the four datasets with the smallest within-site SD ratios, $21$ of $24$
model--dataset pairs are worse on most splits. On SGEMM, whose four runs of a configuration
differ by one percent of the spread, the rows are nearly copies: TabICLv2 goes from $0.112$ to
$0.232$ and TabDPT from $0.086$ to $0.162$. In the other six datasets, whose within-site SD
ratios are at least $0.20$, the rows help on most splits in $32$ of $36$ pairs and hurt in $3$:
TabPFN-v3.5 on superconductivity, and TabICLv2 and TabFM on wages.

Conventional learners provide a reference: on the four lowest-ratio datasets boosting's median
paired \Expanded/\Single ratio stays within $0.99$--$1.02$, while the forest, which bootstraps
rows, the MLP, whose updates grow with the rows, and the row-fitted GP are harmed by up to
$1.30$, $1.11$ and $1.58$, against up to $2.34$ for the TFMs. For the GP, kernel estimation and
dependence between readings are candidate explanations.

\section{Statistical Protocol Mismatch}
\label{sec:explanation}
\label{sec:protocolmismatch}

Many TFMs are prior-data fitted networks (PFNs), which are pretrained to approximate Bayesian prediction under a joint IID observation model
\citep[Section~3]{muller2022}.
We show how even applying this idealized prediction rule, without any approximation error from trained TFMs, can cause harm
when there are repeated measurements.

\textbf{Two observation protocols for the same table.}
The joint IID rule counts each reading as evidence about both its response and its site;
repeated measurement selects the site once and measures only the response again.
Let $\theta\sim\Pi$ index $p_\theta(x)p_\theta(y\mid x)$, and write
$L_X(\theta)=\prod_i p_\theta(X_i)$ and
$L_Y(\theta)=\prod_{i,j}p_\theta(Y_{ij}\mid X_i)$ for $r$ readings at each of $m$ sites.
Let $D_r$ denote this balanced table, the $\mathcal D$ of Section~\ref{sec:equivalence} with $r_i=r$.
Drawing sites once from $p_\theta(x)$ gives
\begin{equation}
\label{eq:reppost}
\Pi_{\mathrm{rep}}(d\theta\mid D_r)\propto\Pi(d\theta)L_XL_Y,
\end{equation}
whereas treating every row as a fresh joint draw gives
\begin{equation}
\label{eq:rowpost}
\Pi_{\mathrm{IID}}(d\theta\mid D_r)\propto\Pi(d\theta)L_X^rL_Y.
\end{equation}
If sites are fixed externally or sampled independently of the task, the appropriate posterior
is instead proportional to $\Pi(d\theta)L_Y$. Thus the same prior and table can give different
posteriors under different observation protocols. We next investigate how these differences can affect prediction.

\textbf{Measurement gain versus protocol mismatch cost.}
Extra readings at each site can improve the precision of site means, but potentially incur
a cost from protocol mismatch. To quantify the tradeoff, we compare the posterior means of both the repeated-measurement protocol
and the IID protocol, computing their Bayesian risks under the former's sampling law, which averages over both the random
function $F(x)=\int y\,p_\theta(y\mid x)\,dy$, with realizations $f$, generated by $\theta\sim\Pi$ and the data collected under that protocol.

\begin{proposition}[Gain--cost identity]
\label{prop:gaincost}
Fix a query point $x_*$ and suppose $\E_{\mathrm{rep}}\{F(x_*)^2\}<\infty$.
For $s\in\{1,r\}$, let $D_s$ follow the repeated-measurement sampling law and define
\mbox{$\widehat f_{\mathrm{rep},s}(x_*)=\E_{\mathrm{rep}}\{F(x_*)\mid D_s\}$}.
Let $\widehat f_{\mathrm{IID},s}(x_*)$ be the posterior mean under the IID protocol,
assumed square-integrable under the repeated-measurement sampling law.
Define the risk, protocol mismatch cost, and measurement gain by
\[
\begin{aligned}
R_{\mathrm{rep}}(g)&=\E_{\mathrm{rep}}\!\left\{(g(x_*)-F(x_*))^2\right\},\\
\mathcal C_s&=\E_{\mathrm{rep}}\!\left\{
(\widehat f_{\mathrm{IID},s}(x_*)-\widehat f_{\mathrm{rep},s}(x_*))^2\right\},\\
\mathcal B_r&=R_{\mathrm{rep}}(\widehat f_{\mathrm{rep},1})
-R_{\mathrm{rep}}(\widehat f_{\mathrm{rep},r}).
\end{aligned}
\]
Then
\[
\begin{aligned}
&R_{\mathrm{rep}}(\widehat f_{\mathrm{IID},r})-R_{\mathrm{rep}}(\widehat f_{\mathrm{IID},1})\\
&\qquad=\underbrace{\mathcal C_r-\mathcal C_1}_{\text{change in protocol mismatch cost}}
-\underbrace{\mathcal B_r}_{\text{measurement gain}}.
\end{aligned}
\]
\end{proposition}
The proof is deferred to Appendix~\ref{app:gaincost}, where we also show that $\mathcal B_r\ge0$.

\textbf{Coupling permits a positive protocol mismatch cost.}
Suppose $\theta=(\theta_X,\theta_Y)$, with $p_\theta(x)=p_{\theta_X}(x)$,
$p_\theta(y\mid x)=p_{\theta_Y}(y\mid x)$ and independent priors on the two indices.
The covariate likelihood then cancels when computing Bayes' rule, giving $\mathcal C_r=0$.
On the other hand, most PFNs utilize structural causal model (SCM) priors, with responses generated from a randomly selected target node.
This makes the prediction rule anticausal, which can couple the randomness in $\theta_X$ and $\theta_Y$
\citep{hollmann2023tabpfn,tabpfn3,scholkopf2012causal}.
Extra covariate evidence can then shift posterior weights between models $p_\theta$ and $p_{\theta'}$ that agree at the measured
sites but predict differently elsewhere.
Under the repeated-measurement law, $\mathcal C_r>0$ precisely when the IID rule and the correct
repeated-measurement rule give different posterior means for the query target with positive probability.

\textbf{A construction exhibiting repetition harm.}
A construction in the appendix makes the tradeoff explicit. The response at a new query
combines a component determined by the function values at the measured sites with a separate
binary task component. Additional readings improve estimation of the first component.
The covariates provide noisy information about the second: the correct posterior accounts
for that noise, whereas the IID rule increasingly overweights the same evidence, making its
prediction too sensitive to chance fluctuations. This increases its average squared error
relative to the correct posterior mean. When reading noise is small and the covariates are
only weakly informative about the binary component, an approximation shows that measurement
gain is proportional to $1-1/r$, while mismatch cost is proportional to $(r-1)^2$, for fixed
repetition counts. Changing the relative strengths of reading noise and covariate information
changes when mismatch cost overtakes measurement gain, allowing harm from the second reading,
initial improvement followed by harm, or improvement throughout the tested range.
These possibilities qualitatively match the different responses to repetition observed across
models and tasks in Section~\ref{sec:curves}.
The appendix specifies the model, approximation, and resulting
thresholds (Proposition~\ref{prop:thresholds}; binary example and numerical checks in
Appendix~\ref{app:out:numerical}).

\section{Computational Mechanisms of Repetition Harm in TabPFN-v3}
\label{sec:computation}

Section~\ref{sec:protocolmismatch} shows that protocol mismatch can produce harm even under
exact inference. Here we investigate how the internal computation in frozen TabPFN-v3 can produce harm,
without assuming that it implements exact Bayesian inference.

\textbf{Architecture and a count-dependence hypothesis.}
In the TabPFN-v3 regressor, count-dependent attention appears in three
feature-distribution embedding modules and $24$ in-context learning (ICL) layers.
In the former, learned inducing queries attend over training rows of a column; in the
latter, both training-row and test-query states attend over training rows.
Query-aware scalable softmax (QASS) uses row count to scale attention queries
\citep{nakanishi2025ssmax,tabpfn3}. With $m$ sites and $r$ readings per site, this count grows
from $m$ to $mr$ without adding sites. We hypothesize that the resulting changes in attention
computation contribute to repetition harm.
To test this hypothesis, we apply a \emph{full count reset}, which sets the count supplied
to all $27$ scaling modules to $m$, leaving all other dataset and model parameters unchanged.

\textbf{Does count reset reverse harm?}
We first use noiseless readings, so repetition duplicates rows without adding information.
For a linear target, copying each row $64$ times flattens TabPFN-v3's fit near measured sites
and steepens it between them. Full count reset largely restores the linear shape and reduces
MSE within the range of measured sites from $1.3\times10^{-3}$ to $6.5\times10^{-5}$
(twelve sites, six site draws, one ensemble member; Table~\ref{tab:linear-count-scopes}).
A combined control additionally matches row identifiers and variance normalization,
recovering the one-reading predictions within a maximum difference of $4.4\times10^{-5}$
in all six cases (Appendix~\ref{app:independent-computation}). Thus count reset reverses much
of the distortion in these settings, while other implementation details account for the
remaining difference.

\textbf{Locating the effect: attention weights and token representations.}
Changing the row count supplied to a QASS module has a direct effect: it changes the
attention weights within that module. It also has indirect effects: the module's changed
output alters the token representations passed to later layers, potentially changing their
queries, keys, values, and attention weights. We
first examine the most visible change, in the attention weights with which test queries
read the training rows, and then ask whether it accounts for the harm.

With noisy readings, repetition concentrates test-query attention on fewer sites. We quantify
this by pooling the normalized attention weights over the readings at each site:
\begin{equation}
\label{eq:neff}
\omega_i(x)=\sum_j a_{ij}(x),\qquad
n_{\mathrm{eff}}(x)=\frac{1}{\sum_i\omega_i(x)^2}.
\end{equation}
This effective number of attended sites equals $m$ for uniform site weights and falls as
attention concentrates \citep{hill1973diversity}. At noise SD $0.05$, the median
$n_{\mathrm{eff}}$ in four late ICL layers falls from $4.7$--$6.8$ with one reading to
$2.3$--$3.8$ with $64$ readings. Full reset reverses this concentration and reduces fresh-input MSE from $0.051$ to $0.010$, compared with $0.017$ for \Single ($120$ cells: six functions $\times$ twenty seeds; Table~\ref{tab:e6}); Figure~\ref{fig:attncontrol} illustrates both changes.

These observations do not establish that attention concentration explains the harm:
concentration occurs both when repetition helps and when it hurts. At noise SD $0.05$,
of the $112$ cells where attention concentrates, $74$ are harmed and $38$ improved
(Table~\ref{tab:e6-joint}). We therefore test whether restoring the one-reading attention
pattern reduces error.
In a separate probe ($72$ cells), we transplant the one-reading
test-query-to-training-row attention weights into all $24$ ICL layers, spreading each site's
weight uniformly over its readings and letting all other computation follow the pass with
all readings. This matches the one-reading site weights in every ICL layer, but reduces error
in only $18$ of $40$ harmed cells and closes at least half the gap in $11$. It helps all twelve
step-function cells but none of the sixteen quadratic cells (Appendix~\ref{app:inside-tabpfn3}).
Restoring query-attention weights is therefore insufficient in many cases. In a separate probe
pooling noiseless and noisy readings, resetting only the counts in the three feature-distribution
embedding modules lowers median quadratic error from $0.183$ to $0.066$. Together, these
results support a contribution from computations that construct the representations read by test queries.

\begin{figure}[t]
  \centering
  \includegraphics[width=\columnwidth]{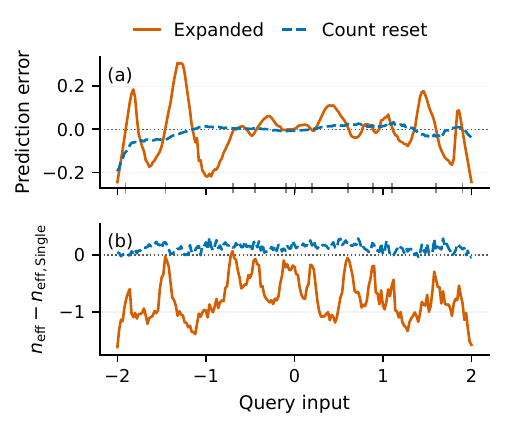}
  \caption{Full count reset on a fixed table of readings: TabPFN-v3, quadratic target,
  $12$ sites, $64$ independent readings, SD $0.05$, one ensemble member, $201$ shared queries.
  (a) Prediction minus truth; zero means correct prediction and grey ticks mark sites.
  (b) Site-pooled $n_{\mathrm{eff}}$ minus \Single, after the median over heads/ICL operations
  at each query; zero matches \Single's concentration. \Single's query median is $11.06$
  (uniform: $12$). Full count reset sets all count inputs to $12$.}
  \label{fig:attncontrol}
\end{figure}

\textbf{Revisiting the irrelevant-feature effect.}
We apply full count reset to Section~\ref{sec:padding}'s experiment, retaining all readings and
features ($m=24$, $r=64$, SD $0.15$). Median paired error ratios relative to \Single fall
from $1.60$, $3.10$, and $3.15$ to $1.045$, $1.051$, and $1.151$ at dimensions $2$, $5$,
and $10$, respectively (Table~\ref{tab:padding-count}). Full count reset substantially attenuates
the amplification with irrelevant features, although a residual dimension effect remains.
Section~\ref{sec:countcontrol} evaluates its practical limits: resetting the count can also
reduce the benefit of additional noisy readings.

\section{Practical Responses and Their Limits}
\label{sec:remedies}

\phantomsection
\label{sec:countcontrol}

\textbf{Count reset suppresses harm and useful gains.}
Full count reset retains \Expanded's readings but fixes all count inputs to $m$.
In TabPFN-v3's standard eight-member ensemble ($m=12$, $r=64$), it reduces low-noise harm
but loses gains at higher noise: at SD $0.30$ and $0.50$, it is worse than \Expanded
(Table~\ref{tab:countctl}).

\begin{table}[ht]
  \caption{Count reset against averaging, TabPFN-v3, $r=64$: median MSE at fresh inputs over
  the six one-dimensional functions and twenty seeds ($m=12$, eight ensemble members),
  and cells (of $120$) in which \Mean beats full count reset. At SD $0$, \Mean equals \Single.}
  \label{tab:countctl}
  \centering\small
  \setlength{\tabcolsep}{3pt}
  \begin{tabular}{lccccc}
    \toprule
    noise SD & \Single & \Expanded & \Mean & \shortstack{full\\reset} & \Mean wins \\
    \midrule
    0    & 0.0088 & 0.053 & 0.0088 & 0.0084 & 36 \\
    0.05 & 0.015  & 0.049 & 0.0094 & 0.0135 & 61 \\
    0.15 & 0.039  & 0.045 & 0.010  & 0.032  & 90 \\
    0.30 & 0.104  & 0.047 & 0.013  & 0.076  & 107 \\
    0.50 & 0.194  & 0.052 & 0.017  & 0.174  & 115 \\
    \bottomrule
  \end{tabular}
\end{table}

On real data, reset returns error near \Single where \Expanded hurts TabPFN-v3
(three datasets among the four lowest within-site SD ratios), but is worse than
\Expanded on most splits in all seven datasets where \Expanded helps (Appendix Table~\ref{tab:e2}). Averaging keeps $m$ rows while more accurate site values carry the measurement gain.

\phantomsection
\label{sec:averaging}

\textbf{Averaging retains precision gains with less bias.}
With balanced counts, \Mean's marginal median MSE is below \Single's for all six models
on one-dimensional synthetic and semi-synthetic tasks at $r=4,64$ ($m=12$, default noise).
\Expanded beats \Mean in six of twelve model--task-type medians at $r=4$, but only one
at $r=64$: one-dimensional TabFM (Appendix Table~\ref{tab:three}).

Both representations reduce prediction variance across independent reading banks;
their biases differ. With eight banks of $64$ readings at fixed sites, \Expanded's
median squared bias exceeds \Mean's for five of six models at both one-dimensional
noise SDs ($0.15$, $0.50$), and all six on semi-synthetic tasks. TabFM is the
one-dimensional exception (Appendix Tables~\ref{tab:banks}--\ref{tab:e3}).
In one dimension, \Mean also lowers squared bias relative to \Single for every model:
the benefit extends beyond reduced sensitivity to reading noise.

On real data, \Mean beats \Expanded on most splits in $20/24$ model--dataset pairs
at the four lowest within-site SD ratios, versus $20/36$ elsewhere (Section~\ref{sec:real}).
Averaging supplies more accurate site values; explicit precision information could
help learners exploit this gain.

\phantomsection
\label{sec:precision}

\textbf{What site means omit: counts and within-site variation.}
\Mean omits counts and within-site variances. Under independent homoscedastic noise,
$\Var(\bar Y_i\mid X_i)=\sigma^2/r_i$. A learner that overestimates this noise could
oversmooth the means, underusing their precision.
Under unequal counts, paired GP fits \citep{binois2018hetgp} use identical means and
pooled reading variance $\widehat\sigma^2$, estimating kernel parameters in each.
Replacing common noise variance $\widehat\sigma^2/\bar r$ with $\widehat\sigma^2/r_i$
lowers median MSE from $0.068$ to $0.023$ at reading-noise SD $0.15$;
$\bar r$ is the mean count (Appendix~\ref{app:three-tables}).

Copying each mean $r_i$ times encodes multiplicity; replacing copies with readings
adds within-site variation.
Copies raise error for TabPFN-v3, TabPFN-v3.5 and LimiX-2 in both task types but help
TabFM in both and TabDPT on semi-synthetic tasks.
Replacing copies with readings helps $4/6$ models in one dimension under the random
and clustered count schemes, but worsens $11/12$ semi-synthetic model--scheme cases
(Appendix Table~\ref{tab:e4enc}).

Future readings carry noise variance $\sigma^2$, versus $\sigma^2/r_i$ for site means.
\emph{Gaussian MeanRestored} (MR) adds pooled reading variance to a Gaussian fit of
\Mean's quantiles:
\begin{equation}
F_{\mathrm{MR}}=\mathcal N\!\left(q_{0.5},
  \left[\frac{q_{0.95}-q_{0.05}}{2\Phi^{-1}(0.95)}\right]^2+\widehat\sigma^2\right),
\label{eq:gaussian-mr}
\end{equation}
where $q_\alpha$ is \Mean's predictive $\alpha$-quantile at fresh $x$, and $\Phi$ the
standard normal CDF; MSE retains \Mean's native point predictions (Appendix~\ref{app:mr}).
For four quantile-output models with Gaussian homoscedastic noise ($r=64$), one-dimensional $90\%$
coverage rises from $43$--$68\%$ to $92$--$96\%$; median paired MR/\Mean
quantile-approximated CRPS ratios \citep{gneiting2007} are $0.906$--$0.988$.
Semi-synthetic coverage starts at $84$--$93\%$, with ratios $0.990$--$1.005$,
including slight deterioration. Median paired MR/\Expanded CRPS ratios are below one
in all eight model--task-type comparisons (Table~\ref{tab:e4cov}); restoration helps
most where \Mean undercovers strongly.

\section{Discussion}
\label{sec:discussion}

\textbf{Interpretation.}
Repeated readings make site estimates more precise but change how predictors use them.
Statistically, counting input evidence once per reading reweights functions that agree at
the sites, a protocol mismatch cost that precise readings leave little measurement gain to
offset, while denser coverage bounds the remaining disagreement. Computationally,
TabPFN-v3's row count alters site attention and internal encodings. Our reference concerns
observation-protocol misspecification; response-model misspecification can likewise make
Bayesian inference worse as data accumulate \citep{grunwald2017misspecified}.

\textbf{Design implications.}
The count-reset tradeoff, the GP precision comparison and the copies-versus-readings
contrast suggest giving TFMs site means, counts and within-site variances as separate
inputs; variances would also inform future-reading uncertainty. We hypothesize that
separating these inputs would retain measurement gains without using row count to convey
precision. None of three exploratory TabPFN-v3 fine-tuning recipes improved \Expanded
without degrading \Mean (Appendix Table~\ref{tab:finetune}). A kernel-weighted prediction
layer like KernelICL's \citep{miftachov2026kernelicl} is one route to such site-aware
weighting.

\textbf{Future directions.}
A direct test of this hypothesis would train these inputs under a prior that draws sites
once and readings within sites, scoring latent-function error and future-reading coverage.
Connecting the protocol mismatch cost to specific count and encoding interventions could
help predict when additional readings are useful. Extending the prior to shared
unit-specific effects would cover correlated readings. For sequential designs that choose
between replicating a site and exploring a new one \citep{binois2019replication}, the
choice would then also depend on how replicates are represented.

\FloatBarrier
\subsubsection*{Acknowledgements}
Qiong Zhang is supported by the National Key R\&D Program of China Grant 2024YFA1015800.
Yan Shuo Tan is supported by NUS Start-up Grant A-8000448-00-00 and MOE AcRF Tier 1 Grant A-8002498-00-00.
Additionally, Bibhas Chakraborty and Yan Shuo Tan are supported by MOE AcRF Tier 1 Grant A-8004458-00-00.
This work was conducted while Qiong Zhang was at the Institute of Statistics and Big Data,
Renmin University of China. Qiong Zhang is now affiliated with Mohamed bin Zayed University
of Artificial Intelligence.

We used generative AI tools to assist with literature review, theoretical development and proof writing, experimental design and implementation, result analysis, and manuscript drafting and editing. The authors take responsibility for the final content of the manuscript.

\bibliographystyle{apalike}
\bibliography{refs}

\clearpage
\appendix
\onecolumn
\begin{center}
{\Large\bfseries When Repeated Measurements Hurt Tabular Foundation Models:\\[0.5ex] Supplementary Materials\par}
\end{center}
\vspace{1ex}

\section{Proofs}\label{app:proofs}

We write $\sum_{i\in A}$ for the sum over sites with $X_i\in A$, and $n_A:=\sum_{i\in A}r_i$.

\subsection{Regional reduction}\label{app:regional}

\textbf{Statement.}
For a region $A$ and a function $g$, write, in the notation of Section~\ref{sec:equivalence} with $\varphi=\mathbf 1_A$, $L_A(g):=L_{\mathbf 1_A}(g;\mathcal D,\mathbf 1)=\sum_{i\in A}\sum_{j}(Y_{ij}-g(X_i))^2$, $\bar L_A(g):=L_{\mathbf 1_A}(g;\bar{\mathcal D},\mathbf r)=\sum_{i\in A}r_i(\Ybar_i-g(X_i))^2$ and $C_A:=\sum_{i\in A}W_i$. Then
\begin{equation}\label{eq:app-regional}
L_A(g)=\bar L_A(g)+C_A .
\end{equation}
(a) $C_A$ does not depend on $g$; hence $L_A(g)-L_A(g')=\bar L_A(g)-\bar L_A(g')$, and for any parametrisation $\beta\mapsto g_\beta$ the derivatives of $L_A(g_\beta)$ and $\bar L_A(g_\beta)$ coincide wherever either exists.
(b) If $A_1,\dots,A_K$ partition $A$, then $\sum_k C_{A_k}=C_A$; hence for two partitions $(A_k)$, $(A'_l)$ of $A$ with predictors $(g_k)$, $(g'_l)$,
\begin{equation}\label{eq:app-partition}
\sum_k L_{A_k}(g_k)-\sum_l L_{A'_l}(g'_l)=\sum_k \bar L_{A_k}(g_k)-\sum_l \bar L_{A'_l}(g'_l).
\end{equation}
No distributional assumption is made.

\begin{proof}
Fix a site $i$ and put $c=g(X_i)$. Writing $Y_{ij}-c=(Y_{ij}-\Ybar_i)+(\Ybar_i-c)$,
\begin{align*}
\sum_j (Y_{ij}-c)^2
&= W_i+2(\Ybar_i-c)\sum_j (Y_{ij}-\Ybar_i)+r_i(\Ybar_i-c)^2\\
&= W_i+r_i(\Ybar_i-c)^2,
\end{align*}
since $\sum_j(Y_{ij}-\Ybar_i)=0$. Summing over $i\in A$ gives \eqref{eq:app-regional}.
(a) $C_A$ involves the readings only, so it cancels in $L_A(g)-L_A(g')$, and $L_A(g_\beta)-\bar L_A(g_\beta)=C_A$ is constant in $\beta$.
(b) Each site of $A$ lies in exactly one $A_k$, so $\sum_k C_{A_k}=\sum_{i\in A}W_i=C_A$. Applying \eqref{eq:app-regional} to each child gives $\sum_k L_{A_k}(g_k)=\sum_k\bar L_{A_k}(g_k)+C_A$, and likewise for $(A'_l,g'_l)$; subtracting gives \eqref{eq:app-partition}.
\end{proof}

\textbf{Localizing functions.}
The statement and proof hold with $\mathbf 1_A$ replaced by any $\varphi\ge0$ on the
covariates. Since $\varphi(X_i)$ is the same for all readings at site $i$, the per-site
identity becomes $\sum_j\varphi(X_i)(Y_{ij}-c)^2=\varphi(X_i)\{W_i+r_i(\Ybar_i-c)^2\}$, so
$L_\varphi(g;\mathcal D,\mathbf 1)-L_\varphi(g;\bar{\mathcal D},\mathbf r)=\sum_i\varphi(X_i)W_i$
does not depend on $g$, and parts (a) and (b) follow as before. A Nadaraya--Watson or local
polynomial fit at $x$ with a fixed bandwidth minimizes $L_\varphi$ with
$\varphi=K_h(x-\cdot)$ over constants or polynomials, a query of type (iv) below, so it
returns the same predictor from $(\mathcal D,\mathbf 1)$ and $(\bar{\mathcal D},\mathbf r)$.
The sections below write regions for concreteness.

\subsection{Algorithmic invariance}\label{app:invariance}

\textbf{Reading the conditions.}
The condition of Proposition~\ref{prop:reduction} applies to every choice the algorithm
makes: which differences it evaluates, which may depend on earlier answers, and its
initialization, penalties, step sizes, stopping and tie-breaking. Split gains, comparisons
between partitions, gradients and least-squares minimizers are all such differences.
Comparing two partitions of $A$, with predictor $g_k$ on cell $A_k$, is the difference
$L_{\mathbf 1_A}(\tilde g;\cdot)-L_{\mathbf 1_A}(\tilde g';\cdot)$ between the piecewise predictors $\tilde g$ and $\tilde g'$; a derivative is
a limit of differences; and a minimizer over a class is fixed by its pairwise differences.
A choice that reads the number of rows in the dataset it is given fails the condition,
since that number is $n$ for $(\mathcal D,\mathbf 1)$ and $m$ for
$(\bar{\mathcal D},\mathbf r)$. The formulation below makes this precise as a sequence of
response queries.

\textbf{Precise formulation of Proposition~\ref{prop:reduction}.}
Let $\Xi:=\{(X_i,r_i)\}_{i\le m}$, the distinct covariate values with the total weight at each, which both inputs $(\mathcal D,\mathbf 1)$ and $(\bar{\mathcal D},\mathbf r)$ determine, and let $\omega$ be the random seed. A \emph{response query} is one of: (i) a difference $L_A(g)-L_A(g')$; (ii) a derivative of $\beta\mapsto L_A(g_\beta)$; (iii) a difference of the form \eqref{eq:app-partition}, which is a difference of type (i) between the piecewise predictors equal to $g_k$ on $A_k$ and to $g'_l$ on $A'_l$; (iv) the set $\arg\min_{g\in\mathcal G}L_A(g)$ for a class $\mathcal G$. On the input $(\bar{\mathcal D},\mathbf r)$ the same query is evaluated with $\bar L$ in place of $L$. Consider an algorithm with states $s_0,s_1,\dots$ such that $s_0$, each transition $s_t\mapsto s_{t+1}$, the stopping decision and the output are fixed functions of $(s_t,\Xi,\omega)$ and of the answers to finitely many response queries, where the arguments of each query (regions, predictors, classes, parametrisations) are fixed functions of $(s_t,\Xi,\omega)$ and of the answers to earlier queries. These functions, which include regularization and tie-breaking, are the same on both inputs. Then the algorithm passes through the same states, stops at the same step and returns the same predictor on both: $\mathcal A(\mathcal D,\mathbf 1,\omega)=\mathcal A(\bar{\mathcal D},\mathbf r,\omega)$.

\begin{proof}
A query with given arguments has the same answer on both tables: (i) and (ii) by \ref{app:regional}(a), (iii) by \ref{app:regional}(b), and (iv) because $L_A$ and $\bar L_A$ differ by a constant on $\mathcal G$. We induct on $t$. Suppose the two runs share $s_t$ (for $s_0$, replace $s_t$ by the empty state). The arguments of the first query depend only on $(s_t,\Xi,\omega)$ and so agree; hence its answers agree. If the first $k$ answers agree, so do the arguments of query $k+1$ and therefore its answer. Thus all answers in the step agree, and $s_{t+1}$, the stopping decision and the output agree.
\end{proof}

\textbf{Corollary (a): regression trees.}
Consider squared-error trees with weighted leaf means, candidate splits determined by the distinct covariate values, and feasibility constraints expressed through total weight $n_A$. Solving $\partial_c\bar L_A(c)=-2\sum_{i\in A}r_i(\Ybar_i-c)=0$ gives the constant minimiser $c_A=n_A^{-1}\sum_{i\in A}r_i\Ybar_i$, which is the row mean on the \Expanded table and the count-weighted mean on $(\bar{\mathcal D},\mathbf r)$. For a split $A=A_1\cup A_2$, the gain on the \Expanded table is
\begin{align*}
&L_A(c_A)-L_{A_1}(c_{A_1})-L_{A_2}(c_{A_2})\\
&\quad=\bar L_A(c_A)-\bar L_{A_1}(c_{A_1})-\bar L_{A_2}(c_{A_2})+C_A-C_{A_1}-C_{A_2},
\end{align*}
and $C_A=C_{A_1}+C_{A_2}$ by \ref{app:regional}(b), so it equals the weighted gain. Candidate splits and feasibility depend on $(\Xi,s_t)$ only, where $s_t$ is the current partition; leaf values are queries of type (iv) and gains of type (iii). The statement applies, and the two trees coincide.

\textbf{Corollary (b): gradient boosting.}
Consider squared-error boosting with $F_0\equiv c_A$ for $A$ the whole table, learning rate $\eta$, a fixed number of rounds, no row subsampling, and a weak learner that itself satisfies the statement above (for instance the tree of Corollary (a)). Suppose $F_t$ agrees on both tables. The row pseudo-residuals are $R_{ij}=Y_{ij}-F_t(X_i)$, with site mean $\bar R_i=\Ybar_i-F_t(X_i)$, which is the pseudo-residual of site $i$ in $(\bar{\mathcal D},\mathbf r)$, carried with weight $r_i$. Since $R_{ij}-\bar R_i=Y_{ij}-\Ybar_i$, identity \eqref{eq:app-regional} applied to the readings $R_{ij}$ gives, for every $h$,
\[
\sum_{i}\sum_j (R_{ij}-h(X_i))^2=\sum_i r_i(\bar R_i-h(X_i))^2+\sum_i W_i .
\]
The row-residual table is therefore the \Expanded table of the weighted site-residual table, the weak learner returns the same $h_t$, and $F_{t+1}=F_t+\eta h_t$ agrees. Induction over rounds completes the proof.

\begin{remark}
The conditions fail in practice under: bootstrap or subsampling of rows; minimum-leaf or minimum-split constraints counted in rows rather than in total weight; stopping rules that threshold the impurity $L_A(c_A)$ itself, which differs from $\bar L_A(c_A)$ by $C_A$; stochastic-gradient training whose number of steps depends on $n$; hyperparameters (for example a noise level) estimated from the table; and floating-point tie-breaking. In the balanced case $r_i=r$, $\bar L_A(g)=r\,L_A(g;\bar{\mathcal D},\mathbf 1)$, $r$ times the unweighted \Mean loss, so $L_A(g)+\lambda P(g)$ has the minimisers of the unweighted loss plus $(\lambda/r)P(g)$: the unweighted \Mean gives the same minimiser if penalties and count-based settings are divided by $r$.
\end{remark}

\textbf{Precision gain at fixed sites.}\label{app:precision}
Fix the function, sites and query, with balanced independent zero-mean reading errors of variance
$\sigma^2$. For a tree partition fixed independently of these errors, let the query's leaf
$A$ contain $k_A\ge1$ sites. Its mean prediction has conditional risk
\[
 \E_Y\!\left[\{\widehat f_r(x_*)-f(x_*)\}^2\right]
 =\left\{\frac{1}{k_A}\sum_{i:X_i\in A}f(X_i)-f(x_*)\right\}^2
   +\frac{\sigma^2}{r k_A}.
\]
The partition bias is fixed and the noise term decreases with $r$. Expanding readings or
weighting their means by $r$ gives the same leaf prediction.

With site-specific variances, the noise term is
$\sum_{i:X_i\in A}\sigma_i^2/(r k_A^2)$, which has the same $1/r$ dependence.

For independent Gaussian readings with known $\sigma_i^2>0$, the likelihood at site $i$ is
\[
 p(Y_{i1},\ldots,Y_{ir_i}\mid f(X_i))
 \propto \exp\!\left[-\frac{r_i\{\Ybar_i-f(X_i)\}^2}{2\sigma_i^2}\right],
\]
where the omitted factor is independent of $f(X_i)$. This is the likelihood of the
mean with variance $\sigma_i^2/r_i$, so a fixed GP prior gives the same latent posterior
under either representation. With known zero noise, copies impose the same exact constraint.

\subsection{Gain--cost identity}\label{app:gaincost}

\textbf{Proof of Proposition~\ref{prop:gaincost}.}
For each $s\in\{1,r\}$, conditional-expectation orthogonality under the repeated-measurement
law gives
\[
R_{\mathrm{rep}}(\widehat f_{\mathrm{IID},s})
=R_{\mathrm{rep}}(\widehat f_{\mathrm{rep},s})+\mathcal C_s.
\]
Subtracting the identities for $s=r$ and $s=1$ proves the proposition with the main-text
definition of $\mathcal B_r$. No coupling or nesting of $D_1$ and $D_r$ is required.
The following additional identities use a nested coupling to express the gain as an
expected squared difference of posterior means.

\textbf{Statement.}
All expectations are under the joint law of task, sites, readings and query. Let $\mathcal F_s:=\sigma(X_1,\dots,X_m,x_*,\,Y_{ij}:j\le s)$, so $\mathcal F_1\subseteq\mathcal F_r$, and read ``$D_s$-measurable'' as $\mathcal F_s$-measurable. Let $G=F(x_*)$ with $\E G^2<\infty$, $z_s:=\E[G\mid\mathcal F_s]$ and $R(g):=\E(g-G)^2$. For any square-integrable $\mathcal F_r$-measurable $a_r$,
\begin{equation}\label{eq:app-gaincost}
R(a_r)-R(z_1)=\E(a_r-z_r)^2-\E(z_r-z_1)^2 .
\end{equation}
Taking $a_r=z_r$ gives $\mathcal B_r=R(z_1)-R(z_r)=\E(z_r-z_1)^2\ge0$.
The nested coupling preserves the marginal laws of $D_1$ and $D_r$, so this is the
measurement gain defined in Proposition~\ref{prop:gaincost}.
For square-integrable $\hat f_r$ and $\hat f_1$, measurable with respect to $\mathcal F_r$ and $\mathcal F_1$,
\begin{equation}\label{eq:app-bridge}
R(\hat f_r)-R(\hat f_1)=-\E(z_r-z_1)^2+\E(\hat f_r-z_r)^2-\E(\hat f_1-z_1)^2 .
\end{equation}
If $\hat f_s$ also uses internal randomness $U_s$, \eqref{eq:app-bridge} holds provided $U_s$ is conditionally independent of $G$ given $\mathcal F_s$, for $s=1$ and $s=r$.

\begin{proof}
Let $g$ be square-integrable and $\mathcal F_s$-measurable. Then
\[
(g-G)^2=(g-z_s)^2+2(g-z_s)(z_s-G)+(z_s-G)^2 .
\]
The cross term is integrable by the Cauchy--Schwarz inequality, and $\E[(g-z_s)(z_s-G)\mid\mathcal F_s]=(g-z_s)(z_s-\E[G\mid\mathcal F_s])=0$. Taking expectations,
\begin{equation}\label{eq:app-pyth}
R(g)=R(z_s)+\E(g-z_s)^2 .
\end{equation}
Apply \eqref{eq:app-pyth} with $s=r$ to $g=a_r$ and to $g=z_1$, which is $\mathcal F_r$-measurable and square-integrable by Jensen's inequality: $R(a_r)=R(z_r)+\E(a_r-z_r)^2$ and $R(z_1)=R(z_r)+\E(z_1-z_r)^2$. Subtracting gives \eqref{eq:app-gaincost}. For \eqref{eq:app-bridge}, \eqref{eq:app-pyth} gives $R(\hat f_r)=R(z_r)+\E(\hat f_r-z_r)^2$ and $R(\hat f_1)=R(z_1)+\E(\hat f_1-z_1)^2$; subtract and insert $R(z_r)-R(z_1)=-\E(z_r-z_1)^2$. With internal randomness, replace $\mathcal F_s$ by $\mathcal F_s\vee\sigma(U_s)$ in the derivation of \eqref{eq:app-pyth}; conditional independence gives $\E[G\mid\mathcal F_s,U_s]=z_s$, so \eqref{eq:app-pyth} is unchanged.
\end{proof}

\begin{remark}
Both identities hold under the joint, prior-averaged law. They say nothing about the risk at a fixed $f$.
\end{remark}

\textbf{Fixed sites and an actual network.}
\label{app:network-bridge}
In protocol comparisons, $a_s=\widehat f_{\mathrm{IID},s}(x_*)$ denotes the IID posterior mean, for $s\in\{1,r\}$. With $z_s$ the conditional mean under the data-generating law as above, write $\mathcal C_s=\E(a_s-z_s)^2$ and $\mathcal B_r=R(z_1)-R(z_r)$ under that same law. Under the repeated-measurement (sampled-once) law, $z_s=\widehat f_{\mathrm{rep},s}(x_*)$, recovering the quantities in Proposition~\ref{prop:gaincost}. Under an externally fixed site protocol, $z_s$ is the fixed-site posterior mean and the IID rule can have $a_1\ne z_1$. Applying
\eqref{eq:app-pyth} at counts one and $r$ gives
$R(a_r)-R(a_1)=\mathcal C_r-\mathcal C_1-\mathcal B_r$.
For an actual predictor, write $d_r=a_r-z_r$ and $e_r=\widehat f_r-a_r$. Then
\begin{equation}\label{eq:bridge}
\begin{aligned}
R(\widehat f_r)-R(\widehat f_1)={}&-\mathcal B_r\\
&+\E(\widehat f_r-z_r)^2-\E(\widehat f_1-z_1)^2,
\end{aligned}
\end{equation}
and $\E(\widehat f_r-z_r)^2=\E d_r^2+\E e_r^2+2\E(d_re_r)$.
The cross term can reinforce or offset the reference shift. The fixed-function
bias/variance estimates in Table~\ref{tab:banks} use a different expectation and do
not identify these terms inside a TFM.

\subsection{Posterior odds under the two protocols}\label{app:odds}

\textbf{Statement.}
Let the task family be finite or countable with prior mass function $\pi$ (the prior $\Pi$ of Section~\ref{sec:protocolmismatch} restricted to such a family), write $p_\theta(x)$ for the input density and $p_\theta(y\mid x)$ for the response density of task $\theta$, and let $L_X(\theta):=\prod_i p_\theta(X_i)$ and $L_Y(\theta):=\prod_i\prod_{j\le r}p_\theta(Y_{ij}\mid X_i)$. Treating the $mr$ rows as independent draws from $p_\theta(x)p_\theta(y\mid x)$ gives the likelihood $\prod_i\prod_{j\le r}p_\theta(X_i)p_\theta(Y_{ij}\mid X_i)=L_X(\theta)^rL_Y(\theta)$; drawing each site once gives $L_X(\theta)L_Y(\theta)$; fixing the sites externally gives $L_Y(\theta)$. By Bayes' rule,
\[
\pi_{\mathrm{IID}}\propto\pi L_X^{r}L_Y,\qquad
\pi_{\mathrm{rep}}\propto\pi L_XL_Y,\qquad
\pi_{\mathrm{fix}}\propto\pi L_Y ,
\]
assuming the normalising sums are finite and positive (the sum for $\pi_{\mathrm{IID}}$ is not automatically finite for a countable family). Write $\mathrm{odds}_\bullet:=\pi_\bullet(\theta_1\mid D_r)/\pi_\bullet(\theta_0\mid D_r)$ and $\Lambda:=L_X(\theta_1)/L_X(\theta_0)$, and assume $\pi(\theta_k)L_X(\theta_k)L_Y(\theta_k)\in(0,\infty)$ for $k=0,1$.
\begin{equation}\label{eq:odds}
\frac{\mathrm{odds}_{\mathrm{IID}}}{\mathrm{odds}_{\mathrm{rep}}}=\Lambda^{r-1}.
\end{equation}
(a) $\mathrm{odds}_{\mathrm{IID}}/\mathrm{odds}_{\mathrm{rep}}=\Lambda^{r-1}$ and $\mathrm{odds}_{\mathrm{IID}}/\mathrm{odds}_{\mathrm{fix}}=\Lambda^{r}$.
(b) Fix the sites and let the readings at site $i$ be i.i.d.\ from $p_{\theta_0}(\cdot\mid X_i)$, independently or not across sites. Assume $\pi(\theta_0),\pi(\theta_1)>0$, $S_X:=\log\Lambda$ finite and $D_Y:=\sum_i\mathrm{KL}(p_{\theta_0}(\cdot\mid X_i)\,\|\,p_{\theta_1}(\cdot\mid X_i))<\infty$. Then, almost surely as $r\to\infty$, $r^{-1}\log\mathrm{odds}_{\mathrm{fix}}\to-D_Y$, $r^{-1}\log\mathrm{odds}_{\mathrm{rep}}\to-D_Y$ and $r^{-1}\log\mathrm{odds}_{\mathrm{IID}}\to S_X-D_Y$. No integrability condition beyond $D_Y<\infty$ is needed.
(c) If $p_\theta$ does not depend on $\theta$ and $0<L_X<\infty$, the three posteriors coincide for every $r$.

\begin{proof}
(a) Normalising constants cancel in odds, so with $\rho:=\pi(\theta_1)L_Y(\theta_1)/\{\pi(\theta_0)L_Y(\theta_0)\}$ we have $\mathrm{odds}_{\mathrm{IID}}=\rho\Lambda^{r}$, $\mathrm{odds}_{\mathrm{rep}}=\rho\Lambda$ and $\mathrm{odds}_{\mathrm{fix}}=\rho$; divide.

(b) Let $Z_{ij}:=\log\{p_{\theta_0}(Y_{ij}\mid X_i)/p_{\theta_1}(Y_{ij}\mid X_i)\}$, which takes values in $(-\infty,\infty]$ almost surely because $p_{\theta_0}(Y_{ij}\mid X_i)>0$ almost surely. Using $\log^+u\le u$, the negative part satisfies
\[
\E Z_{ij}^-\le\E\Big[\frac{p_{\theta_1}(Y_{ij}\mid X_i)}{p_{\theta_0}(Y_{ij}\mid X_i)}\Big]=\int_{\{p_{\theta_0}>0\}}p_{\theta_1}(y\mid X_i)\,dy\le1 ,
\]
so $\E Z_{ij}$ is well defined and equals the $i$th Kullback--Leibler term. These terms are non-negative with finite sum, so each is finite and $\E|Z_{ij}|<\infty$. For each $i$ the $Z_{ij}$ are i.i.d.\ in $j$, and the strong law gives $r^{-1}\sum_{j\le r}Z_{ij}\to\E Z_{i1}$ almost surely; summing over the finitely many sites, $r^{-1}\sum_i\sum_{j\le r}Z_{ij}\to D_Y$. Now
\[
\log\mathrm{odds}_{\mathrm{fix}}=\log\frac{\pi(\theta_1)}{\pi(\theta_0)}-\sum_i\sum_{j\le r}Z_{ij},
\]
and the first term does not depend on $r$, so $r^{-1}\log\mathrm{odds}_{\mathrm{fix}}\to-D_Y$. By (a), $\log\mathrm{odds}_{\mathrm{rep}}=\log\mathrm{odds}_{\mathrm{fix}}+S_X$ and $\log\mathrm{odds}_{\mathrm{IID}}=\log\mathrm{odds}_{\mathrm{fix}}+rS_X$; dividing by $r$ gives the limits $-D_Y$ and $S_X-D_Y$. In particular, if $S_X>D_Y$ the IID-rule odds in favour of the wrong task diverge.

(c) $L_X(\theta)$ is then a constant in $\theta$, and so is $L_X(\theta)^r$; both cancel on normalisation.
\end{proof}

\subsection{Replication in attention}\label{app:replication}

\textbf{Statement.}
Let $s_i:=q^\top k_i$ and $\mathrm{Att}:=\sum_i e^{s_i}v_i/\sum_i e^{s_i}$. If each pair $(k_i,v_i)$ is replaced by $r_i\ge1$ identical copies, the output equals attention over the original pairs with logits $s_i+\log r_i$. If $r_i=r$ for all $i$, the output equals $\mathrm{Att}$.

\begin{proof}
The $r_i$ copies of pair $i$ share the logit $s_i$ and the value $v_i$, so the replicated output is
\begin{equation}
\label{eq:replication}
\frac{\sum_i r_ie^{s_i}v_i}{\sum_i r_ie^{s_i}}=\frac{\sum_i e^{s_i+\log r_i}v_i}{\sum_i e^{s_i+\log r_i}} .
\end{equation}
If $r_i=r$, the factor $r$ cancels between numerator and denominator, leaving $\mathrm{Att}$.
\end{proof}

\begin{remark}
Theorem~\ref{app:out:thm:uniform-copy} gives sufficient operator-level conditions for the uniform-copy identity to extend to an entire attention network.
\end{remark}

\section{Proofs for the statistical reference}
\label{app:reference}

This section develops the gain--cost bounds and repetition thresholds for the Bayesian reference
of Section~\ref{sec:protocolmismatch}, extends the analysis to site coverage and a fixed variance
prior, and proves the network-replication result used in Section~\ref{sec:computation}.
Here $T=(F(X_i))_i$ denotes the vector of site values and $G$ the query target;
$\mathcal B_r$ and $\mathcal C_r$ are the measurement gain and protocol mismatch cost of
Proposition~\ref{prop:gaincost}, and $\mathcal X$ collects common site and query information.

\subsection{Joint likelihoods and a cancellation condition}
\label{app:out:protocol}
\label{app:protocolfull}
All posterior formulas assume a dominating measure and finite, strictly positive normalizers. For independently selected parent sites, the likelihood contains one $p_\theta(X_i)$ per parent. Conditional response independence gives the reading likelihood $L_Y$ at count $r$. Treating each reading as an independent joint row replaces the input exponent $1$ by $r$, proving Equations~\eqref{eq:rowpost}, \eqref{eq:reppost} and~\eqref{eq:odds}. For unequal counts, the discrepancy is $\prod_i p_\theta(X_i)^{r_i-1}$. For externally fixed sites the correct input exponent is zero. A site-selection likelihood cannot be inferred solely from the final table: it is part of the declared observation protocol.

If a query is sampled conditionally on $\theta$ with density $q_\theta(x_*)$, both posteriors additionally contain this factor once. If it is externally specified, neither does. An arbitrary dependence of query selection on the task or observations must likewise be modeled explicitly. The response distribution at the query is then integrated against the appropriate updated posterior.

\textbf{A sufficient cancellation condition.}
Suppose $\Pi=\Pi_X\otimes\Pi_Y$, $p_\theta(x)=p_{\theta_X}(x)$, and $p_\theta(y\mid x)=p_{\theta_Y}(y\mid x)$ for $\theta=(\theta_X,\theta_Y)$. If any informative query-input factor also depends only on $\theta_X$, integrating out $\theta_X$ gives the same posterior for $\theta_Y$ for every input exponent. This proves the factorized-prior claim in the body. More generally, for nuisance coordinates $W_X$ assume
\[
 p_\theta(z,w)=p_\theta(z)q(w),\qquad
 p_\theta(y\mid z,w)=p_\theta(y\mid z),
\]
where $q$ is task-independent, including independence from latent noise parameters. Then $\prod_iq(W_{X,i})^b$ is constant in posterior variables for every exponent $b$ and cancels. The same holds for a common query nuisance factor. This statement requires the factorization to be encoded in the model; it does not follow from a feature's empirical lack of predictive importance.

\subsection{Measurement gain with an unknown variance}
\label{app:out:measurement-gain}

For this measurement bound, let $F$ be a random function, with realizations $f$ as in Section~\ref{sec:protocol}, site values $T_i=F(X_i)$ and query target $G=F(x_*)$.
Condition first on common information $\mathcal X$. Readings follow
$Y_{ij}=T_i+\sqrt S\,\epsilon_{ij}$, where the $\epsilon_{ij}$ are IID standard Gaussian,
independent of $(T,G,\mathcal X,S)$, and the common reading-noise variance $S$ is positive almost surely.
Let $T$ have differentiable density $p$ on $\mathbb R^m$, finite Fisher information
$J=\E\|\nabla\log p(T)\|^2$, and the vanishing boundary terms needed for integration by parts. Let $g(t)=\E[G\mid T=t,\mathcal X]$ have a twice continuously differentiable extension to $\mathbb R^m$ with
$\|\nabla g\|\le L$ and $\|\nabla^2g\|_{\mathrm{op}}\le H_2$.
Assume $\E G^2<\infty$, $S\perp(T,G,\mathcal X)$, and $\E S^2<\infty$. Define
\[
 M_r=\frac1r\sum_{j=1}^rY_{\cdot j},\quad
 W_r=\sum_{i=1}^m\sum_{j=1}^r(Y_{ij}-M_{r,i})^2,\quad
 R_0=\E\operatorname{Var}(G\mid T,\mathcal X),
\]
\[
 K=\E\|\nabla g(T)\|^2,\quad Q_m=m(m+2),\quad
 C^{\mathrm{rem}}=L\sqrt J+\tfrac12H_2\sqrt{Q_m},
\]
where $M_{r,i}=\Ybar_i$ and $W_r=\sum_iW_i$ at count $r$. Define
\begin{equation}
 E_r=
 LH_2\sqrt{Q_m}\,\frac{\E S^{3/2}}{r^{3/2}}
 +\left(\frac{H_2^2Q_m}{4}+(C^{\mathrm{rem}})^2\right)
                  \frac{\E S^2}{r^2}.
 \label{app:out:eq:gain-remainder}
\end{equation}

\begin{lemma}[Finite measurement-gain bound]
\label{app:out:lem:gain}
For the posterior mean $z_r=\E[G\mid\mathcal X,M_r,W_r]$ that integrates $S$,
\begin{equation}
 \left|R(z_r)-R_0-\frac{\E S}{r}K\right|\le E_r,
 \qquad
 \left|\mathcal B_r-(1-1/r)\E[S]K\right|\le E_1+E_r.
 \label{app:out:eq:gain-bound}
\end{equation}
Conditional versions hold with conditional regularity constants and can be averaged whenever the displayed bounds are integrable.
\end{lemma}
\begin{proof}
First fix $v>0$ solely for a calculation. Set $Y=T+\sqrt vN$, with $N\sim N(0,I_m)$ independent, and $\mu_v(Y)=\E[g(T)\mid Y]$. With $\phi_v$ the $N(0,vI_m)$ density, differentiating the convolution gives
\[
 \E[T-Y\mid Y]=v\nabla\log(p*\phi_v)(Y).
\]
Integration by parts also expresses the latter score as
$\E[\nabla\log p(T)\mid Y]$, so conditional Jensen gives
$J(p*\phi_v)\le J$. These are Gaussian score identities \citep{dytso2021derivative}.
Taylor expansion about $Y$, the score bound, and
$\E\|T-Y\|^4=v^2Q_m$ give
\[
 \|\mu_v(Y)-g(Y)\|_2
 \le vL\sqrt J+\tfrac12vH_2\sqrt{Q_m}=vC^{\mathrm{rem}}.
\]
The norm here and below is under the indicated joint experiment.

Given $S$, Gaussian orthogonal decomposition makes $W_r$ independent of $(T,M_r)$, and $G$ is conditionally independent of the reading errors given $T$. Consequently
\[
 z_r=\E[\mu_{S/r}(M_r)\mid M_r,W_r,\mathcal X],
 \qquad
 \E\{z_r-g(M_r)\}^2\le (C^{\mathrm{rem}})^2\E S^2/r^2.
\]
Orthogonal projection then yields
\[
 R(z_r)-R_0
 =\E\{g(T)-g(M_r)\}^2-\E\{z_r-g(M_r)\}^2.
\]
Taylor expansion about $T$ has linear term
$\sqrt{S/r}\nabla g(T)^\top N$, whose mean square is
$\E[S]K/r$. Its remainder has squared expectation at most
$H_2^2Q_m\E S^2/(4r^2)$. Cauchy--Schwarz bounds the absolute cross term by
$LH_2\sqrt{Q_m}\E S^{3/2}/r^{3/2}$, by first conditioning on $S$. Combining these bounds proves the first inequality. Subtracting its versions at $1$ and $r$ proves the second.
\end{proof}

If $S=\nu Z$ with a fixed positive law and $\E Z^2<\infty$, then $E_r=O(\nu^{3/2}+\nu^2)$ for each fixed $r$. The variance used inside the proof is an analytical conditioning variable. The evaluated $z_r$ uses only the observations and the prior for $S$.

\subsection{Weak input cues and finite-grid thresholds}
\label{app:out:local-thresholds}

\textbf{Joint law and regularity conditions.}
At fixed response sites, let $(T,G,H)$ have a proper task law with
$|G|\le M$ and $|H|\le H_0$. Common non-cue site/query information is denoted by $\mathcal X$. Let the cue noise $N_0\sim N(0,1)$, the standard Gaussian reading errors $\epsilon_{ij}$, and $Z>0$ be mutually independent and independent of the task, conditional on $\mathcal X$; the law of $Z$ is fixed, independent of $\mathcal X$, and has $\E Z^2<\infty$. Observe
$V=tH+N_0$ once and readings with variance $S=\nu Z$. Write $\rho_\nu$ for the law of $S$.
Only the cue likelihood is powered; any other input information is conditioned on once in both rules. With $T=(F(X_i))_i$, their joint posterior family for $(F,H,S)$ is
\begin{equation}
 \begin{split}
 \Pi_b(dF,dH,dS\mid\mathcal X,V,M_r,W_r)
 \propto{}&
 \Pi(dF,dH\mid\mathcal X)\rho_\nu(dS)
 e^{b(tVH-t^2H^2/2)} S^{-mr/2}\\
 &{}\times
 \exp\left[-\frac{W_r+r\|M_r-T\|^2}{2S}\right],
 \qquad b=1,r .
 \end{split}
 \label{app:out:eq:cue-posterior}
\end{equation}
Here $a_r$ and $z_r$ are the posterior means of $G=F(x_*)$ under $\Pi_r$ and $\Pi_1$, respectively.
Factors independent of the latent variables have been omitted. Sampling attributes
$U_i=(t/\sqrt m)H+\xi_i$ gives this cue because their aggregate
$m^{-1/2}\sum_iU_i=V$ is sufficient for $H$; the remaining components have task-independent likelihoods.

For each sufficiently small $t$ and almost every $(V,\mathcal X)$, impose the density, differentiability, and boundary assumptions of Lemma~\ref{app:out:lem:gain} on the task law conditional on $(V,\mathcal X)$. Let its constants be $L_t,H_{2,t},J_t$ and $C_t^{\mathrm{rem}}=L_t\sqrt{J_t}+\tfrac12H_{2,t}\sqrt{Q_m}$. Specifically require
\[
 \sup_{|t|\le t_0}\E_t[L_tH_{2,t}]<\infty,\qquad
 \sup_{|t|\le t_0}\E_t[H_{2,t}^2+(C_t^{\mathrm{rem}})^2]<\infty,
\]
and
\[
 K_t=\E_t\left\|\nabla_T\E_t[G\mid T,V,\mathcal X]\right\|^2
       \longrightarrow K_0 .
\]
These are sufficient uniformity assumptions, not consequences of bounded $G,H$ alone. They ensure that the averaged measurement-gain remainder is $o(\nu)$ along $t=\lambda\sqrt\nu$.

The zero-cue covariance constant is
\[
 C_0^2=\E_0\!\left[\operatorname{Cov}(G,H\mid T,\mathcal X)^2\right],\qquad C_0\ge0,
\]
where $\E_0$ denotes expectation under the zero-cue task law.

\begin{proposition}[Local repetition thresholds]
\label{prop:thresholds}
For the model and regularity conditions above, with $t=\lambda\sqrt\nu$, as $\nu\downarrow0$ on a fixed finite count grid,
\begin{equation}
\label{eq:threshold}
\begin{aligned}
&R_{\mathrm{rep}}(\widehat f_{\mathrm{IID},r})-R_{\mathrm{rep}}(\widehat f_{\mathrm{IID},1})\\
&\quad=\nu\{c_{\mathrm{cost}}(r-1)^2
-c_{\mathrm{gain}}(1-1/r)\}+o(\nu),
\end{aligned}
\end{equation}
where $c_{\mathrm{gain}}=\E[Z]K_0$ and $c_{\mathrm{cost}}=\lambda^2C_0^2$ are independent of $\nu$. If both are positive, put
$\eta=c_{\mathrm{gain}}/c_{\mathrm{cost}}$. For sufficiently small $\nu$, $r\ge2$ readings
beat one reading if $r(r-1)<\eta$; adding reading $r+1$ helps if
$r(r+1)(2r-1)<\eta$, with both counts on the grid. Strict reverse inequalities reverse
the comparisons; equality is undetermined.
\end{proposition}

\textbf{A finite bound for the cue cost.}
Let $\mathcal O_r=(M_r,W_r)$ and, under the zero-cue task law, define
\[
 c_r=\operatorname{Cov}(G,H\mid\mathcal O_r,\mathcal X),
 \qquad C_r=\|c_r\|_2.
\]
For $|t|\le t_0$ put
\[
 J_r^\circ=MH_0^2
 \left[2r+4(r^2+1)(3^{1/4}+2t_0H_0)^2\right].
\]
Then
\begin{equation}
 \left[(r-1)|t|C_r-J_r^\circ t^2\right]_+^2
 \le\mathcal C_r
 \le\left[(r-1)|t|C_r+J_r^\circ t^2\right]^2,
 \label{app:out:eq:cue-cost-bound}
\end{equation}
where $[u]_+=\max(u,0)$.

\begin{proof}
Conditional on $(\mathcal O_r,\mathcal X)$, let $\Psi_b(t,v)$ be the mean of $G$ after tilting the zero-cue conditional law by
$\exp\{b(tvH-t^2H^2/2)\}$. Differentiating a normalized expectation gives
$\partial_t\Psi_b(0,v)=bv c_r$.
For a general $t$, set $A=b(vH-tH^2)$. The second derivative is
\[
 \E_b[(G-\E_bG)(A-\E_bA)^2]
       +\operatorname{Cov}_b(G,-bH^2).
\]
Its absolute value is at most
$8Mb^2H_0^2(|v|+t_0H_0)^2+2MbH_0^2$.
Taylor's theorem for $b=r$ and $b=1$ therefore bounds the difference of their remainders. Under the actual experiment, $V=N_0+tH$ and
$\|V\|_4\le3^{1/4}+t_0H_0$. Replacing $V$ in the first-order difference by $N_0$ adds at most
$(r-1)MH_0^2t^2$, using $|c_r|\le MH_0$. Combining terms gives
\[
 \|a_r-z_r-(r-1)tN_0c_r\|_2\le J_r^\circ t^2.
\]
The zero-cue conditional covariance is a function of $\mathcal O_r,\mathcal X$, and $N_0$ is independent of these variables. Thus
$\|N_0c_r\|_2=C_r$. The direct and reverse triangle inequalities prove Equation~\eqref{app:out:eq:cue-cost-bound}.
\end{proof}

\begin{proof}[Proof of Proposition~\ref{prop:thresholds}]
For any bounded $U\in\{G,H,GH\}$, reading-noise independence gives
\[
 \E[U\mid\mathcal O_r,\mathcal X]
 =\E[\E(U\mid T,\mathcal X)\mid\mathcal O_r,\mathcal X].
\]
At fixed $\mathcal X$, approximate the bounded measurable function
$u(T)=\E(U\mid T,\mathcal X)$ in $L^2$ by bounded Lipschitz functions. Since
$\E\|M_r-T\|^2=m\nu\E Z/r\to0$, a Lipschitz function of $M_r$ approximates the same function of $T$. The conditional expectation, being the best $L^2$ predictor from $\mathcal O_r$, has no larger error. The approximation and then the noise limit show
$\E[U\mid\mathcal O_r,\mathcal X]\to\E[U\mid T,\mathcal X]$ in $L^2$.
Boundedness permits averaging over $\mathcal X$. Applying the result to the three listed $U$ proves $C_r\to C_0$.

Equation~\eqref{app:out:eq:cue-cost-bound} with $t=\lambda\sqrt\nu$ now gives
\[
 \mathcal C_r/\nu\longrightarrow
       (r-1)^2\lambda^2C_0^2 .
\]
Lemma~\ref{app:out:lem:gain}, conditional on $(V,\mathcal X)$ and averaged using the stated uniformity assumptions, gives
\[
 \mathcal B_r/\nu\longrightarrow(1-1/r)\E[Z]K_0.
\]
There are only finitely many counts, so convergence is simultaneous on a fixed finite grid.
Write the scaled leading terms in Equation~\eqref{eq:threshold} as
\begin{equation}\label{app:out:eq:local-coefficients}
A:=\nu\,\E[Z]K_0,\qquad B:=t^2C_0^2=\nu\lambda^2C_0^2.
\end{equation}
The main-text coefficients are $c_{\mathrm{gain}}=\E[Z]K_0=A/\nu$ and
$c_{\mathrm{cost}}=\lambda^2C_0^2=B/\nu$.
Substituting the two limits in Proposition~\ref{prop:gaincost} gives that expansion,
with $A/B=\E[Z]K_0/(\lambda^2C_0^2)$ when $B>0$.

Under $K_0C_0>0$ and $B>0$ (equivalently $\lambda\ne0$ here), let
$\eta=\E[Z]K_0/(\lambda^2C_0^2)>0$. For $r\ge2$, the leading contrast divided by $\nu$ factors as
\[
 \lambda^2C_0^2\,\frac{r-1}{r}\{r(r-1)-\eta\}.
\]
The adjacent-count difference divided by $\nu$ tends to
\[
 (2r-1)\lambda^2C_0^2-\frac{\E[Z]K_0}{r(r+1)}.
\]
Its sign is that of $r(r+1)(2r-1)-\eta$. A strict nonzero limit bounds the remainder away from a sign change for sufficiently small $\nu$, proving the claims. At equality the leading expansion does not determine a sign.
\end{proof}

\textbf{Margins on a tested grid.}
For finitely many proposed comparisons, suppose the smallest absolute nonzero leading contrast after division by $\nu$ is $\kappa>0$. Uniform convergence gives a common $\nu_0$ for which every remainder in those comparisons has magnitude less than $\kappa\nu/2$. All their leading signs then hold for $0<\nu<\nu_0$. The value of $\nu_0$ requires a finite remainder calculation; the asymptotic theorem does not supply a universal numerical cutoff.
For two tested counts $s>r$, the leading increment is
\[
 (s-r)\left\{B(s+r-2)-\frac{A}{rs}\right\}.
\]
For $B>0$, this gives the corresponding grid condition $rs(r+s-2)\lessgtr A/B$. These formulas do not justify letting the largest count grow with $1/\nu$.

\textbf{Finite covariance approximation.}
If $g(T)=\E[G\mid T,\mathcal X]$, $h(T)=\E[H\mid T,\mathcal X]$, and
$c(T)=\operatorname{Cov}(G,H\mid T,\mathcal X)$ have uniform Lipschitz constants $L_g,L_h,L_c$, then
\begin{equation}
 |C_r-C_0|
 \le L_c\sqrt{m\E S/r}
       +L_gL_h\sqrt{Q_m\E S^2}/r.
 \label{app:out:eq:covariance-approximation}
\end{equation}
Indeed, total covariance gives
$c_r=\E[c(T)\mid\mathcal O_r,\mathcal X]
 +\operatorname{Cov}(g(T),h(T)\mid\mathcal O_r,\mathcal X)$.
Projection bounds the first term's $L^2$ distance from $c(T)$ by
$L_c\|M_r-T\|_2$. The independent-pair covariance identity bounds the absolute second term by
$L_gL_h\operatorname{tr}\operatorname{Var}(T\mid\mathcal O_r,\mathcal X)$, hence by
$L_gL_h\E[\|T-M_r\|^2\mid\mathcal O_r,\mathcal X]$.
Conditional Jensen and the fourth Gaussian moment prove the displayed bound.
Together with Equations~\eqref{app:out:eq:gain-bound} and
\eqref{app:out:eq:cue-cost-bound}, this gives finite lower and upper intervals
$[\mathcal B_L,\mathcal B_U]$ and $[\mathcal C_L,\mathcal C_U]$.
The inequalities $\mathcal C_L>\mathcal B_U$ and
$\mathcal C_U<\mathcal B_L$ certify harm and benefit, respectively.

\textbf{A common weak cue and varying population noise.}\label{app:out:varying-noise}
Set $t=\delta$ and $\nu=c\delta^2$, keeping the zero-cue task law, the law of $Z$, and $c>0$ fixed. The proposition with $\lambda=1/\sqrt c$ yields
\[
 \frac{R(a_r)-R(z_1)}{\delta^2}
 \longrightarrow
 (r-1)^2C_0^2-c(1-1/r)\E[Z]K_0.
\]
Thus increasing $c$ increases the leading benefit at fixed weak cue. This changes the variance prior used by each reference predictor. It is not the physical-noise comparison with one predictor prior below.

\subsection{Coverage, unresolved values, and stable determination}
\label{app:out:coverage}

\textbf{A pointwise coverage envelope.}
For fixed sites and exact site values $T$, consider the fiber
\[
 \mathcal F_X(T)=\{f:f(X_i)=T_i\text{ for all }i\}.
\]
Assume every function in the prior support satisfies
$|f(x)-f(x')|\le Ld_{\mathcal X}(x,x')^\alpha$, with
$0<\alpha\le1$, in a specified metric. Let
$h=\min_i d_{\mathcal X}(x_*,X_i)$. Two functions in the fiber agree at a nearest site, so the triangle inequality gives
$|f(x_*)-f'(x_*)|\le2Lh^\alpha$.
Every posterior supported on the same fiber places its query mean in an interval of this length. In particular, if $g_1(T,\mathcal X)$ and $g_r(T,\mathcal X)$ are the means under the exact-value correct and powered-input posteriors,
\begin{equation}
 b_r^2:=\E\{g_r(T,\mathcal X)-g_1(T,\mathcal X)\}^2
       \le4L^2\E h^{2\alpha}.
 \label{app:out:eq:coverage-shift}
\end{equation}
This is an exact-value bound for fixed $r$ and a specified averaging law; it is not by itself a finite-noise risk bound.

The same nearest-site value confines $G$ to an interval of radius $Lh^\alpha$ conditional on $(T,\mathcal X)$, and therefore
$\operatorname{Var}(G\mid T,\mathcal X)\le L^2h^{2\alpha}$.
With $|H|\le H_0$, conditional Cauchy--Schwarz gives
\begin{equation}
 C_0^2\le H_0^2L^2\E_0 h^{2\alpha}.
 \label{app:out:eq:coverage-cue}
\end{equation}
The subscript $0$ denotes the zero-cue law defining $C_0$. Equation~\eqref{app:out:eq:coverage-shift} can instead be averaged under the original cue-informed design law. Both inequalities allow correlations among function values.

\textbf{A sampling rate requires conditional coverage.}
Suppose the domain has diameter $D<\infty$ in $d_{\mathcal X}$. The sites are i.i.d.\ exogenous draws, or conditionally i.i.d.\ given the task with a bound uniform over tasks. A random query is independent of those sites conditional on the same information. Assume, uniformly over possible queries and the conditioned information,
\[
 \Pr\{X_i\in B(x,\tau)\}\ge c \tau^{d_{\mathrm{eff}}},
 \qquad 0<\tau\le h_0 ,
\]
with $c,h_0>0$. Then
$\Pr\{h>\tau\}
\le e^{-mc \tau^{d_{\mathrm{eff}}}}$ for $0<\tau\le h_0$.
Integrating the tail of $h^{2\alpha}$ yields
\begin{equation}
 \E h^{2\alpha}
 \le
 \Gamma(1+2\alpha/d_{\mathrm{eff}})
       (mc)^{-2\alpha/d_{\mathrm{eff}}}
 +D^{2\alpha}e^{-mc h_0^{d_{\mathrm{eff}}}}.
 \label{app:out:eq:coverage-rate}
\end{equation}
The last term may be omitted when the bound holds through the domain diameter. To see the integral, extend
$\int_0^{h_0}2\alpha \tau^{2\alpha-1}e^{-mc \tau^{d_{\mathrm{eff}}}}\,d\tau$
to infinity and substitute $v=mc \tau^{d_{\mathrm{eff}}}$. The remaining bounded-domain tail is at most the displayed exponential term.
A marginal input-density bound without the relevant conditional independence is insufficient for the product bound.

The exponent $d_{\mathrm{eff}}$ describes neighborhood probabilities in a metric that controls function variation. It need not equal the number of columns or the number of active coordinates of a particular teacher. These upper envelopes do not order the difference $\mathcal C_r-\mathcal B_r$ across dimensions; such an ordering would require control of both terms under a specified family.

\textbf{A smooth region model with a positive protocol mismatch cost.}
Let $K\ge2$ disjoint plateau regions have equal probability for site selection and query selection. For independent equiprobable signs $H_1,\ldots,H_K$, set $f_H=uH_j$ on region $j$, with $u>0$, and join the plateaus smoothly outside the design support. This defines smooth functions for each fixed separated design support, with target variance $u^2$. Each site has an independently sampled cue vector
$U_i=(t/\sqrt m)H+\xi_i$, $\xi_i\sim N(0,I_K)$, whose aggregate is
$V=tH+N$. Responses depend on the region and $H$ but not additionally on the realized cue coordinates. A proper positive variance prior is independent of all signs, cues, and region labels.

Let $J$ be the query region. With probability
$p_0=(1-1/K)^m$, no measured site lies in region $J$. On this event, no reading or residual statistic contains information about $H_J$. Its posterior odds from the cue are $\exp(2tV_J)$ under the correct protocol and $\exp(2rtV_J)$ under the IID rule. Hence
\[
 z_1=z_r=u\tanh(tV_J),\qquad
 a_r=u\tanh(rtV_J).
\]
Put
\[
 \delta_r(t)=\E\{\tanh(rtV_J)-\tanh(tV_J)\}^2,
 \quad V_J=tH_J+N_J .
\]
For $t>0$ and $r>1$, $\delta_r(t)>0$; it increases strictly with real $r>1$ because the absolute difference increases for every nonzero $V_J$.

\begin{proposition}[Finite-noise and low-noise region comparisons]
\label{app:out:prop:regions}
In this model, let $p_{\mathrm{noise}}=\E\Phi(-u/\sqrt S)$. Then
\begin{equation}
 R(a_r)-R(z_1)
 \ge u^2\{p_0\delta_r(t)-4(1-p_0)p_{\mathrm{noise}}\}.
 \label{app:out:eq:region-finite}
\end{equation}
Thus $p_0>4p_{\mathrm{noise}}/\{\delta_r(t)+4p_{\mathrm{noise}}\}$ suffices for harm. If $S=\nu Z$ for any fixed proper law on $Z>0$, and $u,t,K,m,r$ are fixed, then
\begin{equation}
 \lim_{\nu\downarrow0}\{R(a_r)-R(z_1)\}
 =u^2(1-1/K)^m\delta_r(t).
 \label{app:out:eq:region-limit}
\end{equation}
\end{proposition}
\begin{proof}
On the unobserved-region event, posterior orthogonality makes the excess risk exactly $u^2\delta_r(t)$. On its complement $E$, the sign of one reading in the query region is an available predictor with conditional risk $4u^2p_{\mathrm{noise}}$. \Single's Bayes risk is no larger, while $a_r$ has nonnegative risk. Weighting the two events proves Equation~\eqref{app:out:eq:region-finite}.

For the limit, dominated convergence gives
$p_\nu=\E\Phi(-u/\sqrt{\nu Z})\to0$. Hence
$R(z_1;E)\le4u^2\Pr(E)p_\nu\to0$, and nested projection gives
$R(z_r;E)\le R(z_1;E)$. If $q_1$ denotes the correct posterior probability of the sign opposite the realized $H_J$, averaging also over the true sign gives
$\E[1_Eq_1]=R(z_r;E)/(2u^2)\to0$.

The density ratio of the IID posterior to the correct posterior is at most
$C(V)=\exp\{2(r-1)t\sum_j|V_j|\}$: the quadratic cue term is constant across sign vectors, and the range of the remaining log tilt is bounded by the displayed exponent. If $q_{\mathrm{IID}}$ is its probability of the wrong sign, then for every $c>0$,
\[
 \E[1_Eq_{\mathrm{IID}}]
 \le c\,\E[1_Eq_1]+\Pr\{C(V)>c\}.
\]
The law of $V$ is independent of $\nu$, and $C(V)<\infty$ almost surely. First take $\nu\downarrow0$, then $c\uparrow\infty$. Since
$(a_r-G)^2\le4u^2q_{\mathrm{IID}}$, its observed-region risk also vanishes. Only the exact unobserved-region contribution remains.
\end{proof}

At fixed signal amplitude and aggregate cue strength, this model's limiting harm increases with $K$ and $r$, and decreases with $m$. Comparisons over $m$ keep $t$ fixed by using the per-site cue strength $t/\sqrt m$. These are properties of this constructed family. Increasing its number of separated plateaus does not keep a global H\"older constant fixed on an arbitrary fixed domain, so it is not a universal dimension monotonicity theorem.

\textbf{Stable determination yields a smaller finite-noise cost.}
Condition on the sampled inputs under their correct task law. Let
$\ell=\log L_X$ have finite essential oscillation $\Omega$, and assume for any two functions in the latent support
\[
 |G-G'|\le L_G\|T-T'\|,
 \qquad
 |\ell-\ell'|\le L_\ell\|T-T'\|.
\]
In particular, both the query value and input evidence are determined by exact site values. For independent common variance $S$ with $\E S^2<\infty$,
\begin{equation}
 \mathcal C_r
 \le
 \frac{L_G^2L_\ell^2 Q_m\E S^2}{r^2}
 \left\{\frac{e^{(r-1)\Omega}-1}{\Omega}\right\}^2
 \quad(\Omega>0).
 \label{app:out:eq:stable-cost}
\end{equation}
To prove the bound, tilt the correct posterior by $e^{v\ell}$,
$0\le v\le r-1$. The derivative of its query mean is
$\operatorname{Cov}_v(G,\ell\mid D_r)$.
The independent-pair identity
$\operatorname{Cov}(A,B)=\tfrac12\E[(A-A')(B-B')]$
bounds its absolute value by
$L_GL_\ell\operatorname{tr}\operatorname{Var}_v(T\mid D_r)$.
The tilted density ratio is at most $e^{v\Omega}$, so this is at most
\[
 L_GL_\ell e^{v\Omega}
       \E[\|T-M_r\|^2\mid D_r].
\]
Integrate over $v$ and use conditional Jensen with
$\E\|T-M_r\|^4=Q_m\E S^2/r^2$ to prove the inequality.
When $\Omega=0$, the likelihood is constant and the actual cost is zero.
Under $S=\nu Z$, a fixed $r\ge2$, and Lemma~\ref{app:out:lem:gain} with $\E\|\nabla g(T)\|^2>0$, the cost is $O(\nu^2)$ and the measurement gain is positive of order $\nu$. Thus repetition helps for sufficiently small $\nu$ under these stronger stability assumptions.

\subsection{One fixed variance prior and varying physical noise}
\label{app:out:fixed-noise}

Changing $\rho_\nu$ in a matched Bayes calculation and changing the physical variance with $\rho$ fixed are different experiments. Here all predictions use one fixed proper variance prior. Conditional on sampled sites and any common input information, put
$P_b(dF)\propto\Pi(dF)L_X(F)^b$, $b=1,r$, and let
$p_b$ be the density of $T$ under $P_b$ and
$g_b(t)=\E_{P_b}[G\mid T=t]$. Assume $|G|\le M$, the $p_b$ are bounded and continuous and positive at almost every true $T$, and $g_b$ is continuous at those values. The variance is independent under the predictor prior and has density
\[
 \rho(s)=c\,s^{a-1}(1+o(1))\quad(s\downarrow0),
 \qquad a,c>0.
\]
The physical experiment draws the function under $P_1$ and independent readings with variance $s_0$, which is used only for this analysis. Write $R_{s_0}(g)$ for the squared-error risk under this physical law.

\begin{proposition}[Fixed-prior low-physical-noise limit]
\label{app:out:prop:fixed-noise}
Let $r>1$ be fixed and $k=m(r-1)>2a$. Under the coupling
$M_r=T+\sqrt{s_0/r}N$ and $W_r=s_0\chi$,
$N\sim N(0,I_m)$, $\chi\sim\chi_k^2$ independent, each posterior with exponent $b=1,r$ satisfies
\[
 \mathcal L_b(S/W_r\mid M_r,W_r)
 \longrightarrow \operatorname{InvGamma}(k/2-a,1/2)
\]
in total variation almost surely as $s_0\downarrow0$.
Moreover $a_r\to g_r(T)$ and $z_r\to g_1(T)$ in $L^2$, whereas \Single tends in $L^2$ to
\[
 \mu_1(T),\qquad
 \mu_1(y)=
 \frac{\int_0^\infty\rho(s)[(p_1g_1)*\phi_s](y)\,ds}
      {\int_0^\infty\rho(s)[p_1*\phi_s](y)\,ds}.
\]
Consequently,
\begin{equation}
 \lim_{s_0\downarrow0}\{R_{s_0}(a_r)-R_{s_0}(z_1)\}
 =
 \underbrace{\E_{P_1}\{g_r(T)-g_1(T)\}^2}_{b_r^2}
 -
 \underbrace{\E_{P_1}\{\mu_1(T)-g_1(T)\}^2}_{e_1^2}.
 \label{app:out:eq:fixed-noise-limit}
\end{equation}
The inverse-gamma convention is a density proportional to
$z^{-\gamma-1}e^{-\beta/z}$ for shape $\gamma$ and scale $\beta$.
\end{proposition}
\begin{proof}
Integrating the function out of the reading likelihood gives a variance density proportional to
\[
 \rho(s)s^{-k/2}e^{-W_r/(2s)}
                 [p_b*\phi_{s/r}](M_r).
\]
Fix a coupled realization with $\chi>0$ and a continuity point $T$ where $p_b(T)>0$. Put $\gamma=k/2-a>0$ and substitute $s=W_rz$. Multiplying the unnormalized density in $z$ by $W_r^\gamma$ gives pointwise convergence to
$c\,p_b(T)z^{-\gamma-1}e^{-1/(2z)}$.
For $s<\delta$, the near-zero expansion bounds $\rho(s)$ by a constant times $s^{a-1}$, and the convolution is at most $\|p_b\|_\infty$. This gives an integrable envelope in $z$. The normalized contribution from $s\ge\delta$ is $O(W_r^\gamma)$, since $s^{-k/2}\le\delta^{-k/2}$ and $\rho$ is proper. Dominated convergence, followed by normalization, proves the total-variation limit.

The scaled posterior is tight and $W_r\to0$, so the posterior for $S$ concentrates at zero. Conditional on $S=s$, the query mean is
$[(p_bg_b)*\phi_{s/r}](M_r)/[p_b*\phi_{s/r}](M_r)$.
Continuity at $T$ makes this tend to $g_b(T)$ when $s\to0$ and $M_r\to T$. The posterior concentration and boundedness of $G$ give almost sure and $L^2$ convergence of the mixed mean.

For \Single, $k=0$ and there is no residual-variance factor. Its posterior mean is the fixed mixture $\mu_1(M_1)$. The convolutions are continuous for every $s>0$, bounded by constants depending only on $\|p_1\|_\infty,M$, and integrated against a proper $\rho$. Its denominator is positive at the relevant $T$. Thus $\mu_1(M_1)\to\mu_1(T)$, again in $L^2$. Finally, conditional on $T$ under $P_1$, the residual $G-g_1(T)$ is orthogonal to either limiting prediction minus $g_1(T)$. Subtracting the resulting two risks proves Equation~\eqref{app:out:eq:fixed-noise-limit}.
\end{proof}

The degree condition is sufficient; this proposition does not classify the boundary $k=2a$. With compact upper variance support and $k>2a+2$, the same integral argument weighted by $z$ gives
$\E[S\mid M_r,W_r]\sim W_r/(k-2a-2)$.
The term $e_1^2$ measures \Single's remaining point-prediction error from mixing over noise, rather than coverage of a predictive interval.

\textbf{A finite comparison at a physical variance.}
For a general proper joint law, let $\E_s$ denote expectation under its regular conditional law given $S=s$, and set $R_s(g)=\E_s[(g-G)^2]$. The evaluated predictors still integrate variance. For each count $r'\in\{1,r\}$, define the analytical conditional reference
$z_{r'}^s=\E[G\mid D_{r'},S=s]$ and
\[
 Q_{r'}(s)^2=\E_s(z_{r'}-z_{r'}^s)^2,\quad
 \Delta_{\mathrm{IID}}(s)=\E_s(a_r-z_r)^2,\quad
 B_s=\E_s(z_r^s-z_1^s)^2.
\]
Projection conditional on $s$ gives the exact identity
\begin{equation}
 \begin{split}
 R_s(a_r)-R_s(z_1)
 ={}&\Delta_{\mathrm{IID}}(s)+2\E_s[(a_r-z_r)(z_r-z_r^s)]\\
   &{}+Q_r(s)^2-Q_1(s)^2-B_s .
 \end{split}
 \label{app:out:eq:physical-noise-identity}
\end{equation}
Indeed it first equals
$\E_s(a_r-z_r^s)^2-Q_1(s)^2-B_s$, and expansion gives the display. Direct and reverse triangle inequalities therefore yield
\begin{equation}
 \begin{split}
 (\sqrt{\Delta_{\mathrm{IID}}(s)}-Q_r(s))^2-Q_1(s)^2-B_s
 &\le R_s(a_r)-R_s(z_1)\\
 &\le(\sqrt{\Delta_{\mathrm{IID}}(s)}+Q_r(s))^2-Q_1(s)^2-B_s .
 \end{split}
 \label{app:out:eq:physical-noise-bounds}
\end{equation}
These formulas show explicitly why matched-prior orthogonality does not decompose every term after conditioning on a physical noise value. Only the analytical reference $z_{r'}^s$ is given $s$; no operational predictor or model-selection rule receives it.

\subsection{Numerical verification inside the reference models}
\label{app:out:numerical}

Numerical experiments in the specified Bayesian models with integrated unknown variance verify the approximations and illustrate their scope. The data-generating variance is used for simulation and independent scoring, while each predictor uses its measured sites, counts, means, residual sums of squares, and declared prior.

\textbf{A nonlinear query and non-Gaussian site values.}
Draw $T_i=\operatorname{atanh}(U_i)$ for independent $U_i$ uniform on $[-1,1]$, and an independent equiprobable sign $H$. Define
\[
 G=\frac1{\sqrt m}\sum_i\tanh(T_i)+\frac H2,\qquad
 V=tH+N_0,\qquad
 S\sim\operatorname{Uniform}(\nu/2,3\nu/2).
\]
The query is nonlinear in the site values, whose density is logistic rather than Gaussian. The target variance is $1/3+1/4$. Direct differentiation and $\E U_i^2=1/3$, $\E U_i^4=1/5$ give
\[
 K_0=\frac8{15},\qquad C_0^2=\frac14,\qquad
 L=1,\quad H_2=\frac4{3\sqrt{3m}},\quad J=\frac{4m}{3}.
\]
For example, $\partial_i g(T)=\{1-\tanh^2(T_i)\}/\sqrt m$ and the score of each $T_i$ density is $-2\tanh(T_i)$, giving the displayed gradient energy and Fisher information. Finite site/query values can be embedded by smooth interpolation; the calculation does not require a linear function of covariates.

The count study crosses $m=2$, four population variances
$\nu\in\{.02,.005,.00125,.0003125\}$, eight counts
$r\in\{1,2,3,4,6,8,12,16\}$, and four leading ratios
$\eta=A/B\in\{3,12,56,240\}$. At each $(\nu,\eta)$ the cue strength is
$t=\sqrt{\nu K_0/(\eta C_0^2)}$ because $\E Z=1$. Thus the ratios and variances are crossed, not paired in the listed order. There are 128 related conditions and 16,384 paired tasks per variance scale. Counts share their first reading; additional sufficient statistics are generated conditionally for each count. Variance integration uses Gauss--Legendre quadrature and site-value integration uses Gauss--Hermite quadrature.

At the smallest scale, all 25 nonboundary contrasts have the predicted sign: 13 improvements and 12 deteriorations. Bounds combining the finite measurement remainder with explicit upper and lower sums for the one-dimensional cue integral certify these same signs. Across all 112 nonbaseline conditions, 86 intervals certify a sign and 26 contain zero. An interval containing zero is an inconclusive certificate, not a zero-effect result.

In this factorized example the reading benefit is independent of cue strength. Figure~\ref{app:out:fig:theory-count} therefore combines the measurement-gain estimates with a deterministic cue-cost integral at the common value $t=.02$. For population mean variances
$.0003125,.00125,.005,.02$, the best tested counts are $1,2,3,4$ and the first harmful tested counts are $2,4,6,12$, respectively; the local leading curve gives the same grid choices. All variances remain hidden from the task-specific predictions. Different colors use different population priors and hence different reference predictors.

\begin{figure}[h]
 \centering
 \includegraphics[width=.72\linewidth]{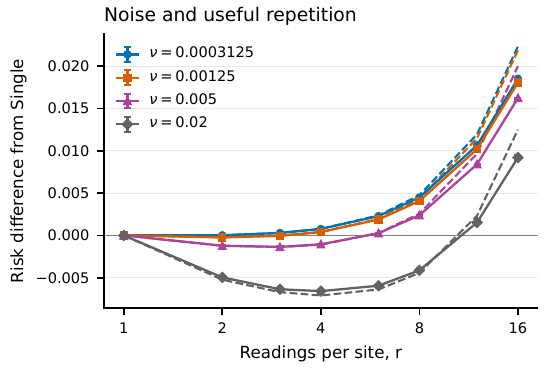}
 \caption{Count thresholds within a specified unknown-noise Bayesian model.
 Markers show the IID-posterior risk minus \Single risk, computed from the integrated reading benefit and cue cost for the nonlinear query model with $m=2$ and fixed cue strength $t=.02$.
 Dashed curves show Equation~\eqref{eq:threshold}.
 Colors vary the population mean variance $\nu$ and the corresponding prior used by the predictor.
 Error bars are pointwise $1.96$ Monte Carlo standard errors of the estimated benefit; cue costs are integrated deterministically.
 The plot verifies this reference calculation and does not report TFM predictions.}
 \label{app:out:fig:theory-count}
\end{figure}

\textbf{Region coverage with inferred variance.}
The region study enumerates all $2^K$ sign functions for
$K\in\{2,4,8\}$, $m\in\{2,4\}$, $r\in\{4,16\}$, and
$\nu\in\{.003,.03,.3\}$, using 8,192 tasks per condition.
Here $u=1$, $t=.25$, and the proper variance prior is
$\operatorname{InvGamma}(4,3\nu)$, whose mean is $\nu$.
Figure~\ref{app:out:fig:theory-coverage} uses all twelve combinations at $\nu=.003$ and compares direct paired risks with Equation~\eqref{app:out:eq:region-limit}. Unobserved-region predictions agree with the exact $\tanh$ expression to within $1.5\times10^{-14}$.
The numerical evaluations of Equation~\eqref{app:out:eq:region-finite} use uniform variance priors; the inverse-gamma posterior experiment above evaluates the low-noise region limit.

\begin{figure}[h]
 \centering
 \includegraphics[width=.75\linewidth]{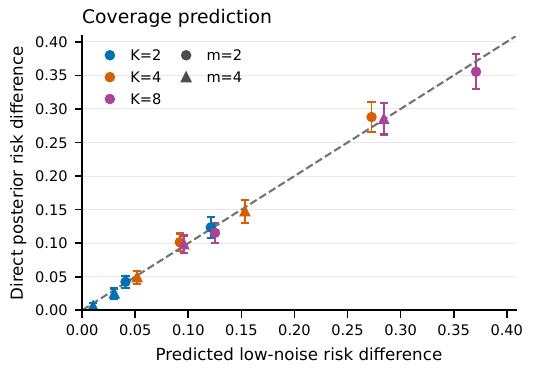}
 \caption{Coverage limit within the unknown-variance region model.
 Direct paired IID-minus-\Single risks at population mean variance $\nu=.003$ are plotted against $u^2(1-1/K)^m\delta_r(t)$, with $u=1$ and $t=.25$.
 All combinations $K=2,4,8$, $m=2,4$, and $r=4,16$ are included.
 The dashed line is equality; bars are pointwise $1.96$ Monte Carlo standard errors over 8,192 tasks.
 Each predictor integrates the same $\operatorname{InvGamma}(4,3\nu)$ prior within its condition.}
 \label{app:out:fig:theory-coverage}
\end{figure}

\textbf{A fixed prior at decreasing physical variance.}
A separate study holds a $\operatorname{Uniform}(0,1)$ variance prior fixed and uses $m=2$, $r\in\{4,8\}$, cue strengths $.03,.2$, and physical variances
$.03,.003,.0003,.00003$, with 2,048 tasks in each of the 16 conditions. It includes both signs of $b_r^2-e_1^2$ in Equation~\eqref{app:out:eq:fixed-noise-limit}. At physical variance $.00003$ and $r=8$, the two risk differences are approximately $-.034563$ and $+.038098$, approaching limits $-.034564$ and $+.038096$. Refining the \Single quadrature changes predictions by less than $7.5\times10^{-8}$. This fixed-prior experiment has a different interpretation from Figure~\ref{app:out:fig:theory-count}.

\subsection{Uniform replication in an entire attention network}
\label{app:out:uniform-copy}

This subsection concerns exact copies of complete observed rows. Fresh noisy readings need not have identical encoded values and do not satisfy the premises.

\begin{lemma}[A fixed-attention identity]
\label{app:out:lem:attention-copy}
For fixed query $q$, keys $k_i$, values $v_i$, and finite logits
$s_i$, uniformly replacing every allowed key/value pair by $r$ identical copies leaves normalized exponential attention unchanged.
Here $s_i$ may include a fixed score function of $q,k_i$ and a fixed pairwise bias. Masked keys are omitted and the allowed pool is nonempty.
\end{lemma}
\begin{proof}
Apply Equation~\eqref{eq:replication} to the allowed pool with $r_i=r$.
Both numerator and denominator have the same positive factor $r$, which cancels.
\end{proof}

Fixed logits are only a local condition: in a network, duplication can alter preprocessing, queries, keys, values, and the set of competing tokens. The following theorem gives sufficient conditions for every operator in a whole computation.

\begin{theorem}[Whole-network uniform-replication invariance]
\label{app:out:thm:uniform-copy}
Consider a finite attention-based computation graph with fixed parameters in exact arithmetic. Run it on a context and fixed queries, and on a second context in which every observed row is copied the same positive integer number of times. Assume the following correspondence is defined at every stage: each context-row state and its row-local feature/token states have identical copies, while query states and any global or inducing states have designated corresponding states. Require all of the following.
\begin{enumerate}
\item \emph{Preprocessing and initial states.}
Preprocessing commutes with this correspondence. In particular, data-derived centering, scaling, feature construction, target transformations, and any metadata used in computation have the same values on corresponding states. Initial copies receive the same embeddings, fingerprints, position or type encodings, and any auxiliary inputs as their originals. Fixed query and inducing states are initially unchanged.
\item \emph{Pointwise and row-local operators.}
Every affine map, activation, featurewise normalization, residual addition, gate, concatenation, and within-row feature operation uses the same parameters and the same corresponding arguments. A normalization over a row or feature embedding has the same coordinates and constants. Any operation aggregating over rows obeys condition 4 instead.
\item \emph{Attention pools, scores, and masks.}
For each corresponding target state and each head, the allowed source pool is either unchanged or obtained by replacing every original allowed source by the same number $c\ge1$ of copies. This number may depend on the attention operation, head, and target, but is common to all sources in that pool. Corresponding allowed source/target pairs have identical key, value, and query transformations, additive biases, temperatures, and other logit transformations. The masks preserve exactly those pairs and their stated multiplicities. Each output is normalized exponential attention over its allowed pool, followed only by the matched maps and head combinations in condition 2.
\item \emph{Other pooling and global statistics.}
Any remaining reduction is either over an unchanged corresponding collection or is invariant to uniform replication of its entire collection. Permitted examples are a mean with its matching normalization, a maximum, and functions of a normalized empirical measure. Unnormalized sums, cardinalities, sample variances with a Bessel correction, and pools that mix copied and uncopied items require a separate verification that their outputs are unchanged; they are not covered merely by permutation invariance.
\item \emph{Randomness and adaptive control.}
The computation is deterministic, or random variables are coupled so that corresponding originals and all their copies receive identical stochastic transformations and preserve conditions 1--4. Adaptive routing, retrieval, truncation, and stopping use the same corresponding decisions and maintain the permitted source pools. Independent dropout masks or row subsampling do not satisfy this condition in general.
\item \emph{Readout.}
The final prediction applies the same map to corresponding fixed query states and any invariant pooled states, followed by the same inverse target transformation. It does not introduce a new dependence on multiplicity.
\end{enumerate}
Then every corresponding query and global state agrees throughout the computation, all copies of each context state agree with the original, and the final query predictions are identical.
\end{theorem}

\begin{proof}
At the initial stage the correspondence holds by condition 1. Order the finite computation graph topologically and suppose all inputs of the next operator satisfy it.
Operators in condition 2 preserve equality because they apply identical functions to identical arguments.
For an attention operation, condition 3 makes corresponding logits and values identical. If its source pool is unchanged, the result is immediate. Otherwise each original numerator contribution and each denominator contribution occurs exactly $c$ times. Lemma~\ref{app:out:lem:attention-copy} gives the same attention output for every corresponding target. This applies head by head, and condition 2 preserves equality after combining heads and residual paths.
Condition 4 preserves every other pooled or global state, while condition 5 supplies the same realized transformations and control decisions when randomness or adaptation is present. Thus the induction advances through every node. Condition 6 then gives identical predictions.
\end{proof}

\textbf{What the conditions allow and exclude.}
An unchanged inducing query may attend all uniformly copied context rows and remain invariant; copied row queries may then attend the unchanged inducing pool. Likewise, a row-local feature aggregator can be repeated once per row without changing its output. In contrast, a softmax pool containing $r$ copies of every context key but only one fixed special key has changed relative multiplicities and generally changes its output. A causal mask based on sequence position, newly sampled row fingerprints, count-dependent attention temperature, or a count-dependent preprocessing statistic can also invalidate the induction. Merely fixing a count input is sufficient only if all the other conditions hold.

For example, for $m>1$ scalars copied $r$ times, the unbiased sample variance is
\[
 s_{\mathrm{copied}}^2
 =\frac{r(m-1)}{mr-1}s_{\mathrm{original}}^2.
\]
This follows by replicating the centered sum of squares and using the two denominators. A scaler based on this quantity changes its scale despite identical empirical proportions; a population-variance scaler with denominator equal to the number of entries is invariant. This distinction explains why preprocessing is part of the theorem.

The theorem provides an architectural sufficient condition. Its fixed-attention lemma alone is not a theorem about a checkpoint. Whether a particular TFM satisfies the conditions, and which violated conditions matter for prediction, are empirical and implementation-specific questions addressed by the network interventions.

\section{Experimental protocol}
\label{app:protocol}

Code, experiment configurations, saved results and reproduction instructions are available at
\url{https://github.com/drizmiz/PFN-Repeated-Measurement-Supp}.

\subsection{Models and learners}

\textbf{Tabular foundation models.} All six use their default configurations with the
full-context settings below, without task-specific tuning or subsampling of rows; the largest
table has $6{,}144$ rows. TabPFN-v3: \texttt{tabpfn} 8.5.0, checkpoint
\texttt{tabpfn-v3-regressor-v3\_default}, eight ensemble members. TabPFN-v3.5: \texttt{tabpfn}
9.0.0. TabICLv2: \texttt{tabicl} 2.2.0, checkpoint \texttt{tabicl-regressor-v2-20260212}.
TabDPT \citep{tabdpt}: \texttt{tabdpt} 1.3.1, eight ensemble members, retrieval context size
set to cover the whole table. LimiX-2: the configuration without retrieval. TabFM: version 1.0.0 with its default
$32$ ensemble members and its row limit raised to $8{,}192$. A fresh model object is built for
every fit.

\textbf{Conventional learners.} Gaussian process: a constant times a Mat\'ern-$5/2$ kernel
plus white noise, kernel and noise level fitted by marginal likelihood from one starting point,
targets normalised, fitted on the rows; on real covariates the inputs are standardised with
the mean and SD of the sites. Gradient boosting: scikit-learn's~\citep{pedregosa2011scikit}
\texttt{GradientBoostingRegressor} at its defaults. Random forest: $200$ trees with bootstrap
resampling of rows, with as many draws as input rows per tree. Under equal counts,
a given site is absent with probability $(1-1/m)^{mr}$ for \Expanded and $(1-1/m)^m$
for \Single. MLP: two hidden layers of $64$ units, at most $3{,}000$ iterations.

\subsection{Designs}

\textbf{Synthetic.} Six one-dimensional functions on $[-2,2]$: quadratic, absolute value,
linear, hyperbolic tangent, step and sine, each standardised to unit variance over the query
grid. Three multivariate families in $d=2,5,10$ with twenty-four sites: additive; an
interaction (Friedman) function; and one function of two coordinates with $d-2$ added
independent coordinates that it ignores, for which the two active coordinates, the readings
and the truth are identical across $d$. Queries are $400$ points (a grid in one dimension,
fresh draws otherwise), the measured sites, and in one dimension the midpoints of the gaps
between adjacent sites.

\textbf{Semi-synthetic.} The covariates of five real regression datasets: diabetes
\citep{efron2004lars}, wine quality \citep{cortez2009wine}, California housing
\citep{pace1997housing}, concrete strength \citep{yeh1998concrete} and energy efficiency
\citep{tsanas2012energy}, in $8$ to $11$ dimensions. A
teacher that is not a TFM (gradient boosting, or a kernel ridge regression for energy
efficiency) is fitted on $70\%$ of the rows; the sites are drawn from those rows, the queries
are up to $300$ held-out rows, and independent Gaussian readings are generated around the
standardised teacher with a noise SD equal to a multiplier times the teacher's residual SD.

\textbf{Axes.} The default design has twelve sites, noise SD $0.15$ (multiplier $0.25$),
one dimension, random sites and equal counts. One property is varied at a time: readings
$r\in\{1,2,4,8,16,32,64,128,256\}$ on the default design, and $r\in\{1,4,64\}$ on every other
axis; sites $m\in\{6,12,24,48,96\}$; noise SD $\in\{0,0.05,0.15,0.30,0.50\}$ or multipliers
$\{0,0.05,0.25,0.5,1\}$; random or evenly spaced sites; dimension $d\in\{2,5,10\}$. This gives
$114$ equal-count designs ($60$ one-dimensional, $45$ on real covariates, $9$ multivariate),
each on twenty seeds. Counts are nested prefixes of one bank of readings per seed, and the
sites and queries of a seed are the same at every count. A further axis draws unequal counts
per site from $\{1,4,16,64\}$, either at random or clustered by position, with a balanced
control of the same total.

The illustrative quadratic design in Figure~\ref{fig:failure} uses noise SD $0.10$,
with the same grid of repetition counts.

\textbf{Summaries.} A design specifies a configuration; a cell is its realization at one
seed (one dataset--split pair for real replicated data). A design is called harmed when the \Expanded table has a larger error
than \Single in at least $11$ of the $20$ seeds. Shares in Figure~\ref{fig:overview} pool the
seeds of the datasets on one axis value. All comparisons are paired within a seed. Designs along an axis share functions and seeds;
these designs are not independent replications.
A median paired ratio is $\mathrm{median}_s(\mathrm{MSE}_{E,s}/\mathrm{MSE}_{S,s})$ over
seeds or splits $s$; medians of errors are computed separately.

\subsection{Real replicated data}
\label{app:realdata}

\begin{table}[h]
  \caption{Real datasets with repeated observations, after retaining sites with at least two
  finite, non-identical responses at one covariate vector. $d$ counts covariates that vary;
  $R$ is the largest count used; the
  context holds half of the sites that have at least $R$ readings, at most $150$.}
  \label{tab:realdesc}
  \centering\small
  \begin{tabular}{llcccc}
    \toprule
    dataset & a reading is & $d$ & sites & $R$ & context sites \\
    \midrule
    SGEMM kernel timings       & one run of a kernel configuration          & 14 & 2{,}994 & 4  & 150 \\
    Kernel Tuner convolution   & one launch of a kernel configuration       & 7  & 3{,}000 & 16 & 150 \\
    Kernel Tuner dedispersion  & one launch of a kernel configuration       & 6  & 3{,}000 & 32 & 150 \\
    FCNet slice localization   & one training run with its own random seed  & 9  & 3{,}000 & 4  & 150 \\
    Bayes-factor surface       & one Monte Carlo replicate                  & 2  & 80      & 4  & 40 \\
    G.\ aparine bioassay       & one pot at a dose and formulation          & 2  & 21      & 8  & 10 \\
    Hydraulic test rig         & one load cycle at a set condition          & 4  & 144     & 8  & 72 \\
    Superconductivity          & one recorded critical temperature          & 81 & 2{,}374 & 8  & 78 \\
    Traffic speeds             & one car in a survey session                & 3  & 84      & 64 & 42 \\
    CPS wages                  & one worker with given covariates           & 6  & 3{,}488 & 32 & 63 \\
    \bottomrule
  \end{tabular}
\end{table}

\textbf{Computer experiments.}
The SGEMM data \citep{paredes2017sgemm} give the running time of a parameterised OpenCL
matrix-multiplication kernel (the kernel of CLTune, \citealp{nugteren2015cltune}) on one
workstation GPU, measured four times for each feasible configuration; a site is one configuration
(its $14$ tuning parameters), a reading is the logarithm of one run time, and we use a fixed random
subset of $3{,}000$ configurations. The Kernel Tuner data are brute-force tuning caches from the
benchmark hub of \citet{willemsen2025benchmarkhub}, produced with Kernel Tuner
\citep{vanwerkhoven2019kerneltuner} for the kernels of \citet{lurati2024hip}; each configuration
provides repeated launches, each recorded as the logarithm of one launch time. The analysed
convolution input contains $31$ launches after dropping the warm-up; dedispersion contains
$32$. The median across sites of lag-one Pearson correlation in source order is $0.19$ and
$0.36$, respectively. The FCNet data \citep{klein2019fcnet} record,
for each configuration of a two-layer network ($9$ hyperparameters) on the slice-localisation
task, the validation error after $100$ epochs of four training runs with different random seeds
on one fixed data split; a reading is the logarithm of one run's validation error, and we use a
fixed random subset of $3{,}000$ configurations. The Bayes-factor surface \citep{franck2020bayes}
gives, for $80$ settings of the two parameters of a Gamma prior, the logarithm of a Bayes factor
estimated by five Monte Carlo replicates on the same data, distributed in the R package
\texttt{hetGP}.

\textbf{Physical and biological experiments.}
The \emph{Galium aparine} bioassay \citep{cabanne1999galium}, distributed in the R package drc
\citep{ritz2019drc}, records the dry matter per pot of plants sprayed with a herbicide alone or
with an adjuvant; a site is one (dose, formulation) pair and a reading is one pot. The hydraulic
data \citep{helwig2015hydraulic} come from a test rig that repeats constant load cycles while the
conditions of its cooler, valve, pump and accumulator are set to $144$ combinations; a site is one
combination and a reading is the mean efficiency factor of one stable load cycle; each
site holds a block of consecutive cycles. The
superconductivity data of \citet{hamidieh2018superconductor} list critical temperatures with $81$
features computed from each material's chemical formula; a site is one distinct feature vector
with at least two different recorded temperatures, and a reading is the temperature of one
sample, so the spread within a site includes differences between samples.

\textbf{Surveys.}
The traffic data, distributed with \citet{davison1997bootstrap}, give the speeds of cars at pairs
of road locations before and after a warning sign was erected; a site is one (location, sign,
period) survey session and a reading is one car. The wage data are the 1988 Current Population
Survey sample of \citet{bierens2001cps}; a site is one combination of the six recorded covariates
and a reading is the wage of one worker.

These ten datasets cover five computer experiments, three physical or biological experiments
and two surveys, with separate runs, units or samples sharing the recorded covariates at a
site. We retain sites with at least two finite responses and discard sites whose responses
are all identical. Table~\ref{tab:realdesc} reports the retained sites; for SGEMM the filter
leaves $2{,}994$ of the $3{,}000$ sampled configurations. Categorical covariates are ordinal-coded.

\textbf{Counts and site splits.}
Let $r_i^{\mathrm{avail}}$ be the available count at retained site $i$. We choose the largest power of two $R$
for which at least $\min(24,\text{number of retained sites})$ sites have $r_i^{\mathrm{avail}}\ge R$.
For superconductivity we use $R=8$, giving $78$ context sites from $156$ eligible sites,
compared with $22$ from $44$ at the rule's $R=16$. For each of twenty seeded splits, the eligible group
($r_i^{\mathrm{avail}}\ge R$) and the lower-count group ($2\le r_i^{\mathrm{avail}}<R$) are independently permuted and each
divided in half, with an odd extra site assigned to test. The context is the first eligible
half, capped at $150$ sites. Test sites are sampled uniformly, capped at $150$, from the
pooled second halves of both groups; the first lower-count half is unused. A site thus never
contributes to both context and test. The context and test sites are shared across counts,
representations and learners within a split. The four lowest-ratio datasets use $150$
context sites; the other six use $10$--$78$. The controlled designs vary noise at fixed
site count and site count at fixed noise (Section~\ref{sec:noisecoverage}).

\textbf{Readings, targets and scaling.}
At each context site, one seeded random ordering of its readings supplies nested prefixes of
$r\in\{1,2,4,\ldots,R\}$. \Single uses the first reading; \Expanded and \Mean use the same first
$r$ readings. All counts and representations, including \Single, use a common response mean
and sample SD (divisor $mR-1$), computed from the context's first $R$ readings per site.
Reported MSE is in units of this common variance. The test target is the mean of \emph{all}
available readings at that site, an empirical mean rather than an observed latent truth.
The within-site SD ratio is the pooled within-site SD divided by the SD of all readings in the
retained dataset. The twenty splits share sites and are not independent replications.

\section{Representation and count-reset comparisons}
\label{app:side}

These experiments isolate one comparison each on the designs, seeds and model configurations of the
main grid (Appendix~\ref{app:protocol}); unless stated otherwise $m=12$, the noise SD is $0.15$
(multiplier $0.25$ on real covariates) and every model runs in its shipped configuration.

\FloatBarrier
\subsection{Rows against means for conventional learners}

\begin{table}[htbp]
  \caption{Proposition~\ref{prop:reduction} in practice. Fits (of $440$: six one-dimensional
  functions and five real-covariate sets, twenty seeds, two count configurations) in which the
  prediction from the means equals the prediction from the rows to $10^{-6}$ at every query, and
  in parentheses at the sites only. Equal counts are $r=4$ and $64$; unequal counts are the grid's
  random and clustered schemes. Weighted: count-weighted means $(\bar{\mathcal D},\mathbf r)$ (for the Gaussian processes,
  each mean's noise variance divided by its count).}
  \label{tab:e1}
  \centering\small
  \begin{tabular}{lcccc}
    \toprule
    & \multicolumn{2}{c}{equal counts} & \multicolumn{2}{c}{unequal counts} \\
    learner & weighted & unweighted & weighted & unweighted \\
    \midrule
    least squares               & 440 (440) & 440 (440) & 440 (440) & 40 (40) \\
    kernel ridge, fixed kernel  & 440 (440) & 0 (0)     & 440 (440) & 0 (0) \\
    GP, fixed kernel and noise  & 440 (440) & 0 (0)     & 440 (440) & 0 (0) \\
    regression tree             & 241 (440) & 241 (440) & 240 (440) & 240 (438) \\
    gradient boosting           & 240 (440) & 240 (440) & 240 (440) & 0 (0) \\
    forest without bootstrap    & 240 (440) & 240 (440) & 240 (440) & 240 (438) \\
    \midrule
    forest with bootstrap       & 0 (0)     & 0 (0)     & 0 (0)     & 0 (0) \\
    MLP                         & 220 (220) & 0 (0)     & 103 (103) & 0 (0) \\
    GP, noise estimated         & 16 (133)  & 0 (0)     & 0 (0)     & 0 (0) \\
    \bottomrule
  \end{tabular}
\end{table}

The count-weighted tree-based learners agree everywhere in one dimension and at the sites on
real covariates. Their between-site differences in eight to eleven dimensions are consistent
with split ties: different covariates can partition the twelve sites into the same groups,
allowing floating-point tie-breaking to select different continuations. The MLP
agrees exactly at $r=4$, when its table fits in one batch, and not at $r=64$. With equal counts,
the GP that estimates its noise level has a median maximum prediction difference of
$10^{-4}$ from the means-with-variances fit, against $0.02$ from means alone. With unequal
counts, these differences rise to $0.060$ and $0.169$, respectively.

\FloatBarrier
\subsection{Count reset on the main designs}

\begin{table}[!htbp]
  \caption{Count reset in TabPFN-v3: median error at fresh inputs of \Single, \Expanded,
  \Mean, full count reset (all $27$ count arguments set to $m$) and embedding-only
   reset (the three embedding modules), and cells in which \Mean beats full count reset. One-dimensional synthetic: six functions, $120$ cells per row; dimension axis: three
   families, $m=24$, $60$ cells; semi-synthetic: five datasets, $100$ cells; real data: twenty
  splits at the largest count.}
  \label{tab:e2}
  \centering\small
  \setlength{\tabcolsep}{4pt}
  \begin{tabular}{@{}llcccccc@{}}
    \toprule
    & setting & \Single & \Expanded & \Mean & \shortstack{full\\reset} & \shortstack{embedding\\only} & mean better \\
    \midrule
    1-D, $r=4$  & noise 0    & 0.0088 & 0.028 & 0.0088 & 0.0089 & 0.021 & 36 \\
                & noise 0.05 & 0.015 & 0.021 & 0.016 & 0.015 & 0.018 & 59 \\
                & noise 0.15 & 0.039 & 0.028 & 0.031 & 0.028 & 0.029 & 64 \\
                & noise 0.30 & 0.104 & 0.047 & 0.050 & 0.069 & 0.040 & 77 \\
                & noise 0.50 & 0.194 & 0.084 & 0.082 & 0.176 & 0.076 & 94 \\
    1-D, $r=64$ & noise 0    & 0.0088 & 0.053 & 0.0088 & 0.0084 & 0.038 & 36 \\
                & noise 0.05 & 0.015 & 0.049 & 0.0094 & 0.0135 & 0.037 & 61 \\
                & noise 0.15 & 0.039 & 0.045 & 0.010 & 0.032 & 0.042 & 90 \\
                & noise 0.30 & 0.104 & 0.047 & 0.013 & 0.076 & 0.114 & 107 \\
                & noise 0.50 & 0.194 & 0.052 & 0.017 & 0.174 & 0.382 & 115 \\
    \midrule
    dimension, $r=64$ & $d=2$  & 0.062 & 0.103 & 0.023 & 0.072 & 0.052 & 60 \\
                      & $d=5$  & 0.281 & 0.347 & 0.254 & 0.263 & 0.269 & 45 \\
                      & $d=10$ & 0.354 & 0.490 & 0.349 & 0.356 & 0.367 & 33 \\
    \midrule
    semi-synthetic, $r=64$ & mult.\ 0    & 0.405 & 0.542 & 0.405 & 0.424 & 0.440 & 44 \\
                           & mult.\ 0.25 & 0.439 & 0.528 & 0.416 & 0.410 & 0.438 & 50 \\
                           & mult.\ 0.5  & 0.513 & 0.522 & 0.414 & 0.449 & 0.456 & 64 \\
                           & mult.\ 1    & 0.698 & 0.524 & 0.424 & 0.561 & 0.647 & 83 \\
    \midrule
    real data & SGEMM ($r=4$)              & 0.105 & 0.161 & 0.104 & 0.102 & 0.136 & 2 \\
              & G.\ aparine ($r=8$)        & 0.388 & 0.199 & 0.320 & 0.302 & 0.159 & 9 \\
              & hydraulic rig ($r=8$)      & 0.178 & 0.074 & 0.071 & 0.177 & 0.066 & 20 \\
              & traffic ($r=64$)           & 0.338 & 0.218 & 0.243 & 0.323 & 1.412 & 19 \\
              & CPS wages ($r=32$)         & 0.898 & 0.675 & 0.565 & 0.724 & 0.903 & 18 \\
              & KT convolution ($r=16$)    & 0.223 & 0.259 & 0.218 & 0.223 & 0.258 & 10 \\
              & KT dedispersion ($r=32$)   & 0.303 & 0.290 & 0.295 & 0.307 & 0.593 & 17 \\
              & FCNet slice ($r=4$)        & 0.245 & 0.281 & 0.243 & 0.236 & 0.296 & 11 \\
              & Bayes factor ($r=4$)       & 0.025 & 0.016 & 0.017 & 0.022 & 0.013 & 17 \\
              & superconductivity ($r=8$)  & 0.226 & 0.185 & 0.213 & 0.209 & 0.192 & 7 \\
    \bottomrule
  \end{tabular}
\end{table}

Full count reset brings exact-copy error close to \Single: at $r=64$ the medians are $0.0084$
and $0.0088$, against \Expanded's $0.053$. Full count reset beats \Single in $84$ of $120$ cells and loses
in $36$. At noise SD $0.15$ its error stays between $0.025$ and $0.039$ for counts from two to
$256$, while \Mean falls from $0.034$ to $0.010$.

At noise SDs $0.30$ and $0.50$, embedding-only reset helps at $r=4$ but reverses at $r=64$
in the one-dimensional designs. It lowers error from $0.047$ to
$0.040$ and from $0.084$ to $0.076$ at $r=4$, but raises it from $0.047$ to $0.114$ and from
$0.052$ to $0.382$ at $r=64$. On real data it lowers error on G.\ aparine ($0.199$ to
$0.159$), SGEMM ($0.161$ to $0.136$), the hydraulic rig ($0.074$ to $0.066$) and Bayes factors
($0.016$ to $0.013$).
Its largest losses are on traffic ($0.218$ to $1.412$), KT dedispersion ($0.290$ to $0.593$)
and CPS wages ($0.675$ to $0.903$); it also raises FCNet error from $0.281$ to $0.296$ at $r=4$.
Resetting only the embedding count therefore leaves substantial downstream responses to
repetition. Full count reset also gives up useful gains: it is worse than \Expanded on $13$--$20$
of $20$ splits in each of the seven datasets where the rows help TabPFN-v3 on most splits,
and on at most $5$ splits in the three where they hurt. On traffic, full count reset is worse
than \Expanded on all $20$ splits.

\textbf{Count reset with irrelevant features.}
For the family with irrelevant features ($m=24$, $r=64$, noise SD $0.15$, twenty seeds per
dimension), Table~\ref{tab:padding-count} compares arms within each cell before taking
the median. Full count reset reduces most of the extra error at fixed signal dimension.

\begin{table}[htbp]
  \centering\small
  \caption{TabPFN-v3 with irrelevant features: median paired fresh-input MSE ratios to
  \Single, and cells in which full count reset beats \Expanded. Each dimension has twenty
  seed cells; all arms share the sites, readings and queries within a cell.}
  \label{tab:padding-count}
  \begin{tabular}{rccc}
    \toprule
    $d$ & \Expanded/\Single & full/\Single & full wins \\
    \midrule
    2  & 1.600 & 1.045 & 15/20 \\
    5  & 3.100 & 1.051 & 19/20 \\
    10 & 3.146 & 1.151 & 20/20 \\
    \bottomrule
  \end{tabular}
\end{table}

Across the three dimension families pooled, full count reset beats \Expanded in $47$, $51$
and $52$ of $60$ cells at $d=2,5,10$; these counts pool additive and interaction targets
with designs using irrelevant features, whereas Table~\ref{tab:padding-count} isolates the added irrelevant features.

\FloatBarrier
\subsection{Bias and variance over independent reading banks}

Independent reading banks separate precision gain from changes in between-site prediction.
For $B=8$ banks at fixed sites and fresh queries $q$, write
$\widehat\mu_q=B^{-1}\sum_b\widehat f_{bq}$ and
$s_q^2=(B-1)^{-1}\sum_b(\widehat f_{bq}-\widehat\mu_q)^2$. With $Q$ queries and truth $f_q$,
the finite-bank estimates are
\[
 \widehat{\mathrm{Bias}^2}=\frac1Q\sum_q
 \left\{(\widehat\mu_q-f_q)^2-\frac{s_q^2}{B}\right\},\qquad
 \widehat{\mathrm{Var}}=\frac1Q\sum_q s_q^2.
\]
Their sum equals the mean MSE over banks within each design. The corrected squared-bias
estimate can be negative; medians of the two components need not sum to the median total MSE.

\begin{table}[htbp]
  \caption{What averaging removes. Squared bias and variance of the prediction at fresh inputs,
  estimated from eight independent banks of $64$ readings at the same sites; median over $60$ cells
  (six one-dimensional functions times ten site seeds, $m=12$, noise SD $0.15$).}
  \label{tab:banks}
  \centering\small
  \setlength{\tabcolsep}{3pt}
  \begin{tabular}{lcccccc}
    \toprule
    & \multicolumn{2}{c}{\Single} & \multicolumn{2}{c}{\Expanded} & \multicolumn{2}{c}{\Mean} \\
    model & bias$^2$ & var & bias$^2$ & var & bias$^2$ & var \\
    \midrule
    TabPFN-v3   & 0.023 & 0.013 & 0.041 & 0.0005 & 0.007 & 0.0005 \\
    TabPFN-v3.5 & 0.100 & 0.016 & 0.240 & 0.0013 & 0.035 & 0.0006 \\
    TabICLv2   & 0.072 & 0.018 & 0.107 & 0.0004 & 0.063 & 0.0005 \\
    TabDPT     & 0.082 & 0.015 & 0.064 & 0.0006 & 0.036 & 0.0004 \\
    LimiX-2    & 0.041 & 0.014 & 0.200 & 0.0037 & 0.027 & 0.0013 \\
    TabFM      & 0.241 & 0.008 & 0.150 & 0.0007 & 0.220 & 0.0003 \\
    \bottomrule
  \end{tabular}
\end{table}

\begin{table}[htbp]
  \caption{Squared bias and variance at fresh inputs over eight independent banks of $64$ readings
  at the same sites, finite-bank corrected; median over cells (one-dimensional: six functions
  times ten site seeds, noise SD $0.50$; semi-synthetic: five datasets times ten site
  seeds, noise multiplier $0.25$).}
  \label{tab:e3}
  \centering\small
  \setlength{\tabcolsep}{4pt}
  \begin{tabular}{llcccccc}
    \toprule
    & & \multicolumn{2}{c}{\Single} & \multicolumn{2}{c}{\Expanded} & \multicolumn{2}{c}{\Mean} \\
    designs & model & bias$^2$ & var & bias$^2$ & var & bias$^2$ & var \\
    \midrule
    1-D, noise 0.5 & TabPFN-v3   & 0.112 & 0.110 & 0.043 & 0.0033 & 0.012 & 0.0036 \\
                   & TabPFN-v3.5 & 0.273 & 0.073 & 0.278 & 0.0028 & 0.049 & 0.0040 \\
                   & TabICLv2   & 0.121 & 0.126 & 0.147 & 0.0031 & 0.062 & 0.0045 \\
                   & TabDPT     & 0.273 & 0.098 & 0.138 & 0.0033 & 0.047 & 0.0032 \\
                   & LimiX-2    & 0.118 & 0.093 & 0.318 & 0.0073 & 0.029 & 0.0046 \\
                   & TabFM      & 0.302 & 0.067 & 0.196 & 0.0041 & 0.225 & 0.0018 \\
    \midrule
    semi-synthetic & TabPFN-v3   & 0.418 & 0.018 & 0.508 & 0.0011 & 0.405 & 0.0004 \\
                   & TabPFN-v3.5 & 0.465 & 0.011 & 0.538 & 0.0008 & 0.475 & 0.0004 \\
                   & TabICLv2   & 0.434 & 0.025 & 0.603 & 0.0003 & 0.439 & 0.0005 \\
                   & TabDPT     & 0.501 & 0.013 & 0.531 & 0.0006 & 0.500 & 0.0002 \\
                   & LimiX-2    & 0.384 & 0.034 & 0.516 & 0.0074 & 0.381 & 0.0016 \\
                   & TabFM      & 0.498 & 0.009 & 0.666 & 0.0007 & 0.508 & 0.0002 \\
    \bottomrule
  \end{tabular}
\end{table}

\Expanded captures much of the precision gain while increasing between-site bias. At
one-dimensional noise SD $0.50$, its variance medians are $0.0028$--$0.0073$, compared with
\Mean's $0.0018$--$0.0046$ and $13$--$41$ times smaller than \Single's. On the semi-synthetic
designs its squared-bias median exceeds \Single's for all six TFMs, by $0.03$--$0.17$, whereas
\Mean's stays within about $0.013$ of \Single's (Table~\ref{tab:e3}). \Expanded has higher squared
bias than \Mean for five of six TFMs at one-dimensional noise SD $0.50$ (TabFM is the exception),
and for all six on the semi-synthetic designs.
The one-dimensional designs at noise SD $0.15$ are in Table~\ref{tab:banks}. At the measured
sites every table's bias is near zero with $64$ readings and the variance is the noise of a site
mean. On the semi-synthetic designs bias dominates every table: TabPFN-v3's median total MSE over banks is $0.408$ for \Mean against $0.478$ for \Single, and nearly all of it is bias, the error of carrying the site
values between sites.

\FloatBarrier
\subsection{Predictive coverage and the count encoding}
\label{app:mr}

Gaussian MeanRestored (Equation~\eqref{eq:gaussian-mr}) and the unequal-count GP precision comparison use the pooled reading-variance estimate, for $\sum_i(r_i-1)>0$,
\begin{equation}
\widehat\sigma^2=\frac{\sum_{i=1}^{m}\sum_{j=1}^{r_i}(Y_{ij}-\bar Y_i)^2}
  {\sum_{i=1}^{m}(r_i-1)},
\label{eq:mr-noise}
\end{equation}
where $\bar Y_i$ is the sample mean at site $i$. For balanced counts $r_i=r>1$, the denominator is $m(r-1)$. MR postprocesses the frozen \Mean output.
The experiment retains the model's original point prediction
for MSE; MR changes the distributional output, with the predictive median as its center.

Table~\ref{tab:e4cov} evaluates future readings at fresh inputs under Gaussian,
homoscedastic reading noise. For each model there are $120$ synthetic cells
($6$ functions $\times$ $20$ seeds) and $100$ semi-synthetic cells
($5$ real-covariate datasets $\times$ $20$ seeds). Each cell is one function--seed
or dataset--seed pair. Each has $m=12$ sites; the synthetic noise SD is $0.15$, and
the semi-synthetic noise SD is $0.25$ times the residual SD used to construct the task.
Scores average over $400$ synthetic or $300$ semi-synthetic fresh queries.
For each method, let $Q_\alpha(x)$ denote its predictive quantile, including the quantiles
of $F_{\mathrm{MR}}$ for MR. At a query $x$ with true mean $f(x)$ and generating
reading-noise SD $\sigma$, coverage is
$\Phi((Q_{0.95}(x)-f(x))/\sigma)-\Phi((Q_{0.05}(x)-f(x))/\sigma)$,
then averaged over queries within a cell. The generating
$\sigma$ is used for scoring; MR estimates its added variance from the context readings.
CRPS is approximated by twice the mean pinball loss at the $19$ quantile levels
$0.05,0.10,\ldots,0.95$, averaged over queries and $64$ generated future readings per
query. The same future-reading bank is shared across methods and counts within a
cell. Table entries are medians across cells; each CRPS ratio is formed within
a cell before taking its median. MR/\Mean compares the Gaussian approximation plus
added variance with the original \Mean distribution. MR/\Expanded also changes the
representation of the same readings; MR/\Single additionally uses $64$ readings per site
instead of one. In model order TabPFN-v3, TabPFN-v3.5, TabICLv2, TabDPT, MR has strictly
lower CRPS than \Mean in $87,70,106,65$ of $120$ synthetic cells and $47,43,55,63$ of
$100$ semi-synthetic cells.

\begin{table}[htbp]
  \caption{Future-reading prediction at fresh inputs, $r=64$, under Gaussian,
  homoscedastic reading noise. Coverage entries are median probabilities for nominal
  $90\%$ intervals; CRPS entries are median paired ratios of Gaussian MR to the named
  comparator, using the $19$-quantile approximation (lower is better).
  Each model has $120$ function--seed and $100$ real-covariate dataset--seed pairs,
  with $m=12$ sites, noise SD $0.15$ (synthetic) or residual-SD multiplier $0.25$
  (semi-synthetic), and $400$ or $300$ fresh queries, respectively.
  MR is defined in Equation~\eqref{eq:gaussian-mr}; \Single uses $r=1$.
  Exp. denotes \Expanded.
  LimiX-2 and TabFM return point predictions only.}
  \label{tab:e4cov}
  \centering\small
  \setlength{\tabcolsep}{3pt}
  \begin{tabular}{lccccccc}
    \toprule
    & \multicolumn{4}{c}{Coverage} & \multicolumn{3}{c}{CRPS: MR/comparator} \\
    \cmidrule(lr){2-5}\cmidrule(lr){6-8}
    model & \Single & Exp. & \Mean & MR & Exp. & \Single & \Mean \\
    \midrule
    \multicolumn{8}{l}{One-dimensional} \\
    TabPFN-v3 & 0.87 & 0.91 & 0.43 & 0.92 & 0.926 & 0.810 & 0.906 \\
    TabPFN-v3.5 & 0.96 & 0.99 & 0.68 & 0.94 & 0.684 & 0.808 & 0.983 \\
    TabICLv2 & 0.67 & 0.98 & 0.57 & 0.92 & 0.819 & 0.835 & 0.948 \\
    TabDPT & 0.93 & 0.90 & 0.65 & 0.96 & 0.966 & 0.853 & 0.988 \\
    \midrule
    \multicolumn{8}{l}{Semi-synthetic} \\
    TabPFN-v3 & 0.91 & 0.86 & 0.89 & 0.91 & 0.914 & 0.981 & 1.001 \\
    TabPFN-v3.5 & 0.93 & 0.94 & 0.92 & 0.94 & 0.964 & 0.987 & 1.005 \\
    TabICLv2 & 0.82 & 0.90 & 0.84 & 0.87 & 0.869 & 0.957 & 0.996 \\
    TabDPT & 0.92 & 0.91 & 0.93 & 0.97 & 0.987 & 0.967 & 0.990 \\
    \bottomrule
  \end{tabular}
\end{table}

\begin{table}[htbp]
  \caption{The count encoding under unequal counts (the grid's random and clustered schemes): median
  error at fresh inputs of \Mean, \Mean with the row of site $i$ copied $r_i$
  times, and \Expanded, and cells (of $120$ one-dimensional, $100$ semi-synthetic) in which the
  copies are worse than \Mean.}
  \label{tab:e4enc}
  \centering\small
  \setlength{\tabcolsep}{4pt}
  \begin{tabular}{llcccc@{\hskip 1.5em}cccc}
    \toprule
    & & \multicolumn{4}{c}{one-dimensional} & \multicolumn{4}{c}{semi-synthetic} \\
    model & scheme & mean & copies & \Expanded & worse & mean & copies & \Expanded & worse \\
    \midrule
    TabPFN-v3   & random    & 0.027 & 0.054 & 0.040 & 89  & 0.444 & 0.490 & 0.489 & 65 \\
               & clustered & 0.026 & 0.083 & 0.039 & 97  & 0.429 & 0.491 & 0.501 & 69 \\
    TabPFN-v3.5 & random    & 0.078 & 0.169 & 0.163 & 83  & 0.473 & 0.491 & 0.524 & 61 \\
               & clustered & 0.070 & 0.182 & 0.128 & 90  & 0.473 & 0.482 & 0.492 & 54 \\
    LimiX-2    & random    & 0.054 & 0.156 & 0.096 & 108 & 0.410 & 0.477 & 0.483 & 66 \\
               & clustered & 0.055 & 0.107 & 0.080 & 96  & 0.394 & 0.460 & 0.476 & 72 \\
    TabICLv2   & random    & 0.090 & 0.100 & 0.127 & 70  & 0.450 & 0.522 & 0.556 & 60 \\
               & clustered & 0.088 & 0.080 & 0.093 & 55  & 0.449 & 0.488 & 0.561 & 63 \\
    TabDPT     & random    & 0.056 & 0.064 & 0.076 & 64  & 0.534 & 0.447 & 0.459 & 30 \\
               & clustered & 0.061 & 0.048 & 0.067 & 54  & 0.521 & 0.448 & 0.476 & 36 \\
    TabFM      & random    & 0.259 & 0.201 & 0.142 & 43  & 0.507 & 0.462 & 0.599 & 32 \\
               & clustered & 0.237 & 0.226 & 0.212 & 63  & 0.500 & 0.489 & 0.616 & 25 \\
    \bottomrule
  \end{tabular}
\end{table}

\FloatBarrier
\subsection{Attention concentration on the main designs}

\begin{table}[htbp]
  \caption{Site-pooled attention in TabPFN-v3 (one ensemble member, $r=64$, six functions, twenty
  seeds): median $n_{\mathrm{eff}}$ over $120$ cells (each the median over $400$ fresh queries of the
  median jointly over heads and layers; uniform attention gives $12$), cells in which it falls with the
  rows and rises under full count reset, and the error of the same fits. Full count reset's $n_{\mathrm{eff}}$ exceeds
  \Single's in $114/120$ cells at SD $0.05$ and $120/120$ at SD $0.15$.}
  \label{tab:e6}
  \centering\small
  \setlength{\tabcolsep}{4pt}
  \begin{tabular}{lcccccccc}
    \toprule
    & \multicolumn{3}{c}{$n_{\mathrm{eff}}$} & & & \multicolumn{3}{c}{fresh-input MSE} \\
    noise SD & \Single & \Expanded & full & Exp.\ $<$ \Single & full $>$ Exp. & \Single & \Expanded & full \\
    \midrule
    0.05 & 10.91 & 10.56 & 11.12 & 112/120 & 120/120 & 0.017 & 0.051 & 0.010 \\
    0.15 & 10.89 & 10.70 & 11.19 & 94/120  & 120/120 & 0.042 & 0.051 & 0.029 \\
    \bottomrule
  \end{tabular}
\end{table}

Full count reset raises site-pooled $n_{\mathrm{eff}}$ above \Expanded in $120/120$ cells at each
noise level, and above \Single in $114/120$ at SD $0.05$ and $120/120$ at SD $0.15$.
The indices are zero-based physical ICL blocks. The change is concentrated in layers $18$, $19$, $21$ and $22$, where the
median falls from $6.3$, $4.7$, $6.8$ and $5.5$ with one reading to $3.0$, $2.3$, $3.8$ and $2.7$
with the rows at noise SD $0.05$; full count reset restores $6.5$, $4.9$, $7.7$ and $6.2$. The step
function is the exception at noise SD $0.15$, where the rows lower $n_{\mathrm{eff}}$ in only $8$ of
$20$ cells.

\begin{table}[htbp]
  \centering\small
  \caption{Attention concentration and prediction error in the same one-member
  fits as Table~\ref{tab:e6}: counts over $120$ function--seed cells per noise level. A down-arrow denotes \Expanded's
  site-pooled $n_{\mathrm{eff}}$ below \Single; harm denotes \Expanded's fresh-input MSE
  above \Single. Cells summarize $400$ fresh queries each.}
  \label{tab:e6-joint}
  \begin{tabular}{lrrrr}
    \toprule
    & \multicolumn{2}{c}{$n_{\mathrm{eff}}\downarrow$}
    & \multicolumn{2}{c}{$n_{\mathrm{eff}}\not\downarrow$} \\
    noise SD & harm & gain & harm & gain \\
    \midrule
    0.05 & 74 & 38 & 6  & 2  \\
    0.15 & 43 & 51 & 11 & 15 \\
    \bottomrule
  \end{tabular}
\end{table}
The resulting harm counts are $80/120$ and $54/120$; these are the same one-member fits
as Table~\ref{tab:e6}, rather than the standard-ensemble risk grid. No query is counted as
an independent replication.

\FloatBarrier
\section{Additional results}
\label{app:results}

\FloatBarrier
\subsection{The design of Figure~\ref{fig:failure} for all six models}

\begin{figure*}[h]
  \centering
  \includegraphics[width=\textwidth]{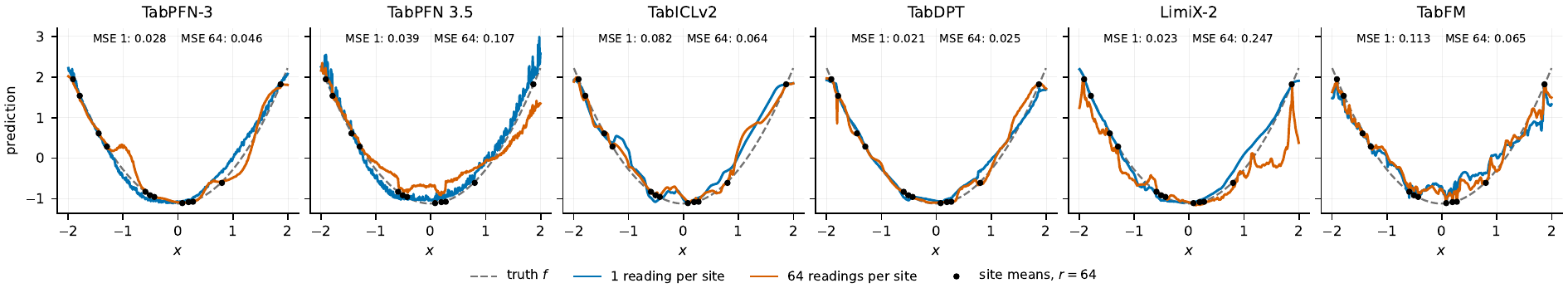}
  \caption{Fitted functions on the left design of Figure~\ref{fig:failure} (quadratic, twelve
  sites, noise SD $0.10$) with one reading per site (blue) and $64$ readings per site entered as
  rows (orange); the numbers in each panel are the two errors. The points are the site means at $r=64$,
  which every model passes through; the changes are between them.}
  \label{fig:fits}
\end{figure*}

\begin{table}[htbp]
  \caption{Raw MSE at $r=1$ and $r=256$ on the design of Figure~\ref{fig:failure} (noise SD
  $0.10$), for all ten learners: median over $20$ seeds, with the number of seeds in which more readings made the
  error worse.}
  \label{tab:failure}
  \centering\small
  \setlength{\tabcolsep}{5pt}
  \begin{tabular}{lccc}
    \toprule
    learner & MSE$(1)$ & MSE$(256)$ & worse \\
    \midrule
    Gaussian process  & 0.014 & 0.0002 & 0/20 \\
    Gradient boosting & 0.172 & 0.158 & 6/20 \\
    Random forest     & 0.182 & 0.158 & 4/20 \\
    MLP               & 0.043 & 0.018 & 1/20 \\
    \midrule
    TabPFN-v3          & 0.017 & 0.070 & 18/20 \\
    TabPFN-v3.5        & 0.062 & 0.379 & 20/20 \\
    TabICLv2          & 0.134 & 0.145 & 13/20 \\
    TabDPT            & 0.066 & 0.089 & 15/20 \\
    LimiX-2           & 0.047 & 0.306 & 20/20 \\
    TabFM             & 0.365 & 0.327 & 2/20 \\
    \bottomrule
  \end{tabular}
\end{table}

\FloatBarrier
\subsection{Where the error changes}

\begin{figure}[htbp]
  \centering
  \includegraphics[width=\textwidth]{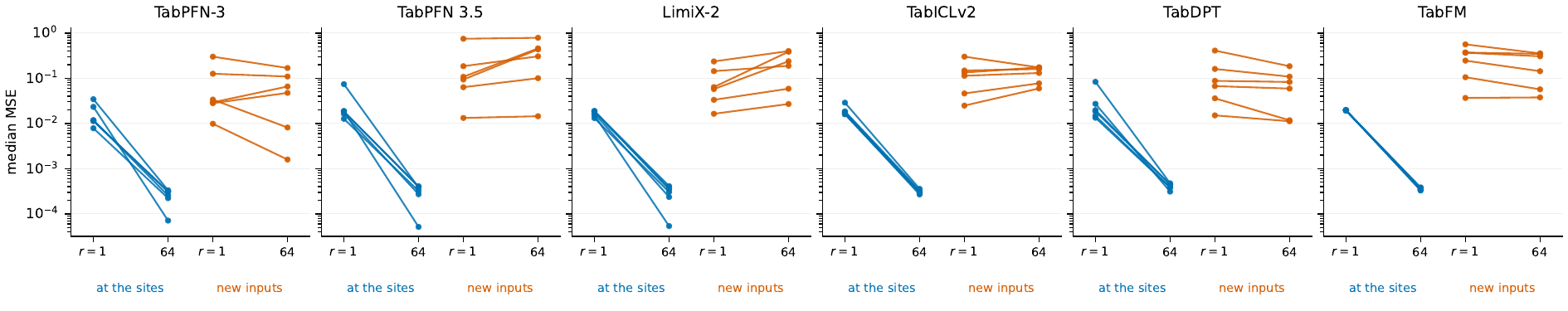}
  \caption{Error at the measured sites (blue) and at fresh inputs (orange) with one reading and
  with $64$ readings as rows; one line per synthetic function, $m=12$, $\sigma=0.15$, medians
  over twenty seeds. At the sites the error of every model falls by a factor of $35$ or more.}
  \label{fig:site}
\end{figure}

Table~\ref{tab:harmcounts} counts the harmed designs per learner, and
Table~\ref{tab:query} gives the comparison of Section~\ref{sec:noisecoverage} under the three
query laws.

\begin{table}[htbp]
  \caption{Designs in which \Expanded at $r=64$ is worse than \Single at fresh inputs
  in at least $11$ of $20$ seeds. Parentheses give the number of tested designs in
  each family; designs sharing a function or seed are not independent replications.}
  \label{tab:harmcounts}
  \centering\small
  \setlength{\tabcolsep}{2pt}
  \begin{tabular}{lrrrrrr|rrrr}
    \toprule
    family & LimiX-2 & TabICLv2 & TabPFN-v3.5 & TabPFN-v3 & TabFM & TabDPT & forest & GP (rows) & MLP & boosting \\
    \midrule
    1-D (60) & 39 & 26 & 35 & 14 & 8 & 4 & 10 & 0 & 0 & 0 \\
    Real covariates (45) & 25 & 34 & 24 & 30 & 33 & 17 & 34 & 19 & 8 & 6 \\
    Multivariate (9) & 9 & 9 & 4 & 8 & 7 & 4 & 9 & 3 & 0 & 0 \\
    \midrule
    All (114) & 73 & 69 & 63 & 52 & 48 & 25 & 53 & 22 & 8 & 6 \\
    \bottomrule
  \end{tabular}
\end{table}

\begin{table}[htbp]
  \caption{Query location. Function--seed pairs (of $120$: six one-dimensional functions,
  twenty seeds, default design) in which \Expanded at $r=64$ is worse than \Single, by where
  the prediction is scored.}
  \label{tab:query}
  \centering\small
  \begin{tabular}{lccccccc}
    \toprule
    scored at & TabPFN-v3 & TabPFN-v3.5 & LimiX-2 & TabICLv2 & TabDPT & TabFM & forest \\
    \midrule
    measured sites & 0 & 0 & 0 & 0 & 0 & 0 & 0 \\
    gap midpoints  & 37 & 78 & 91 & 59 & 16 & 30 & 111 \\
    fresh inputs   & 50 & 93 & 102 & 70 & 29 & 22 & 57 \\
    \bottomrule
  \end{tabular}
\end{table}

\FloatBarrier
\subsection{Coverage and spacing}

\begin{table}[htbp]
  \caption{Left: quadratic, $r=64$, seeds (of $20$) in which \Expanded is worse than \Single,
  by the number of sites. Right: share of function--seed pairs harmed with evenly spaced and
  with random sites ($m=12$, six one-dimensional functions).}
  \label{tab:coverage}
  \centering\small
  \begin{tabular}{lccccc@{\hskip 2em}cc}
    \toprule
    & \multicolumn{5}{c}{sites $m$} & \multicolumn{2}{c}{spacing} \\
    model & 6 & 12 & 24 & 48 & 96 & even & random \\
    \midrule
    TabPFN-v3   & 5 & 16 & 8 & 1 & 0   & 0.23 & 0.42 \\
    TabPFN-v3.5 & 6 & 18 & 11 & 9 & 7  & 0.78 & 0.78 \\
    LimiX-2    & 11 & 20 & 20 & 10 & 3 & 0.88 & 0.85 \\
    TabICLv2   & 9 & 12 & 6 & 5 & 2   & 0.32 & 0.58 \\
    TabDPT     & 1 & 9 & 8 & 5 & 2    & 0.07 & 0.24 \\
    TabFM      & 2 & 3 & 0 & 0 & 0    & 0.12 & 0.18 \\
    \bottomrule
  \end{tabular}
\end{table}

\FloatBarrier
\subsection{Irrelevant features, all models}

\begin{table}[!htbp]
  \caption{Irrelevant features ($m=24$, noise $0.15$, $r=64$, twenty seeds): median error of
  \Single, \Expanded and \Mean, median paired \Expanded/\Single ratio, and seeds in which
  \Expanded is worse than \Single.}
  \label{tab:paddingfull}
  \centering\small
  \resizebox{\textwidth}{!}{%
  \begin{tabular}{llccccc}
    \toprule
    model & $d$ & \Single & \Expanded & \Mean & ratio & worse \\
    \midrule
    TabPFN-v3   & 2 / 5 / 10 & 0.051 / 0.069 / 0.084 & 0.076 / 0.204 / 0.259 & 0.012 / 0.022 / 0.026 & 1.60 / 3.10 / 3.15 & 15 / 18 / 19 \\
    TabPFN-v3.5 & 2 / 5 / 10 & 0.066 / 0.099 / 0.124 & 0.070 / 0.101 / 0.133 & 0.028 / 0.045 / 0.065 & 1.00 / 1.00 / 1.04 & 10 / 10 / 13 \\
    LimiX-2    & 2 / 5 / 10 & 0.040 / 0.050 / 0.046 & 0.045 / 0.088 / 0.175 & 0.014 / 0.016 / 0.015 & 1.11 / 1.61 / 3.06 & 13 / 17 / 20 \\
    TabICLv2   & 2 / 5 / 10 & 0.083 / 0.193 / 0.243 & 0.172 / 0.279 / 0.424 & 0.060 / 0.136 / 0.202 & 1.89 / 1.80 / 1.69 & 19 / 18 / 18 \\
    TabDPT     & 2 / 5 / 10 & 0.079 / 0.113 / 0.120 & 0.072 / 0.175 / 0.282 & 0.046 / 0.060 / 0.084 & 0.97 / 1.46 / 1.72 & 8 / 17 / 20 \\
    TabFM      & 2 / 5 / 10 & 0.138 / 0.182 / 0.211 & 0.172 / 0.256 / 0.297 & 0.118 / 0.145 / 0.164 & 1.11 / 1.31 / 1.30 & 15 / 18 / 17 \\
    Forest & 2 / 5 / 10 & 0.174 / 0.243 / 0.315 & 0.220 / 0.288 / 0.387 & 0.171 / 0.255 / 0.328 & 1.16 / 1.08 / 1.18 & 16 / 13 / 14 \\
    \bottomrule
  \end{tabular}}
\end{table}

\FloatBarrier
\subsection{The Gaussian process under three representations}

\begin{table}[htbp]
  \caption{The Gaussian process: designs (of those fitted) in which the table at $r=64$ is
  worse than \Single in at least $11$ of $20$ seeds.}
  \label{tab:gp}
  \centering\small
  \begin{tabular}{lccc}
    \toprule
    & on the rows & means with variances & means only \\
    \midrule
    One-dimensional synthetic & 0 of 60 & 0 of 60 & 0 of 54 \\
    Real covariates           & 19 of 45 & 18 of 45 & 5 of 40 \\
    Multivariate synthetic ($d=2,5,10$) & 3 of 9 & 0 of 9 & 0 of 9 \\
    \bottomrule
  \end{tabular}
\end{table}

``On the rows'' is the fit used throughout the paper; its kernel and noise parameters
are estimated by marginal likelihood. ``Means with variances'' uses the site means
and fixes each mean's noise variance to the pooled reading variance of the readings
divided by $r$, estimating only kernel parameters by marginal likelihood.
``Means only'' gives the process \Mean and lets it estimate one noise level
from the means; without the within-site scatter this is a different estimator, and its
designs exclude the noise-free ones, where \Mean equals \Single.

The median absolute difference between the first two fits' fresh-input MSEs, relative
to the MSE of the means-with-variances fit, is $0.07\%$ across $3{,}120$ paired
one-dimensional fits pooled over the tested counts $r>1$.
The three harmed multivariate designs use ten-dimensional inputs. In a diagnostic on
the ten-dimensional additive target, trying five initial length-scales on three seeds reached
higher marginal likelihoods and changed all three seed comparisons with \Single from losses to gains.
On real covariates, comparisons against \Single yield $18$ of $45$ harmed designs
for the GP given means and estimated variances and $19$ for the row fit
(Table~\ref{tab:gp}). Deterioration thus persists with an aggregated representation
that explicitly supplies the estimated precision of each mean.

\begin{table}[htbp]
  \caption{GP aggregation on the four actual datasets with the lowest within-site SD ratios,
  at their largest count.
  The row fit estimates kernel and noise parameters from \Expanded; the means+variance fit
  uses site means with pooled empirical within-site variance divided by $r$, and estimates
  only kernel parameters. Harm counts compare each fit with the same \Single GP on $20$
  paired splits. The last column is the median of
  $|\mathrm{MSE}_{\rm rows}-\mathrm{MSE}_{\rm means+var.}|/\mathrm{MSE}_{\rm means+var.}$.}
  \label{tab:gp-real}
  \centering\small
  \setlength{\tabcolsep}{3pt}
  \begin{tabular}{lccc}
    \toprule
    & \multicolumn{2}{c}{harmed splits} & relative MSE gap \\
    dataset ($r$) & rows & means+var. & median \\
    \midrule
    SGEMM ($4$)           & 14/20 & 14/20 & $5.0\times10^{-7}$ \\
    KT convolution ($16$) & 19/20 & 19/20 & $7.2\times10^{-7}$ \\
    KT dedispersion ($32$)& 17/20 & 17/20 & $1.4\times10^{-6}$ \\
    FCNet slice ($4$)     & 11/20 & 11/20 & $3.4\times10^{-5}$ \\
    \bottomrule
  \end{tabular}
\end{table}

\FloatBarrier
\subsection{The three tables}
\label{app:three-tables}

\begin{figure*}[t]
  \centering
  \includegraphics[width=\textwidth]{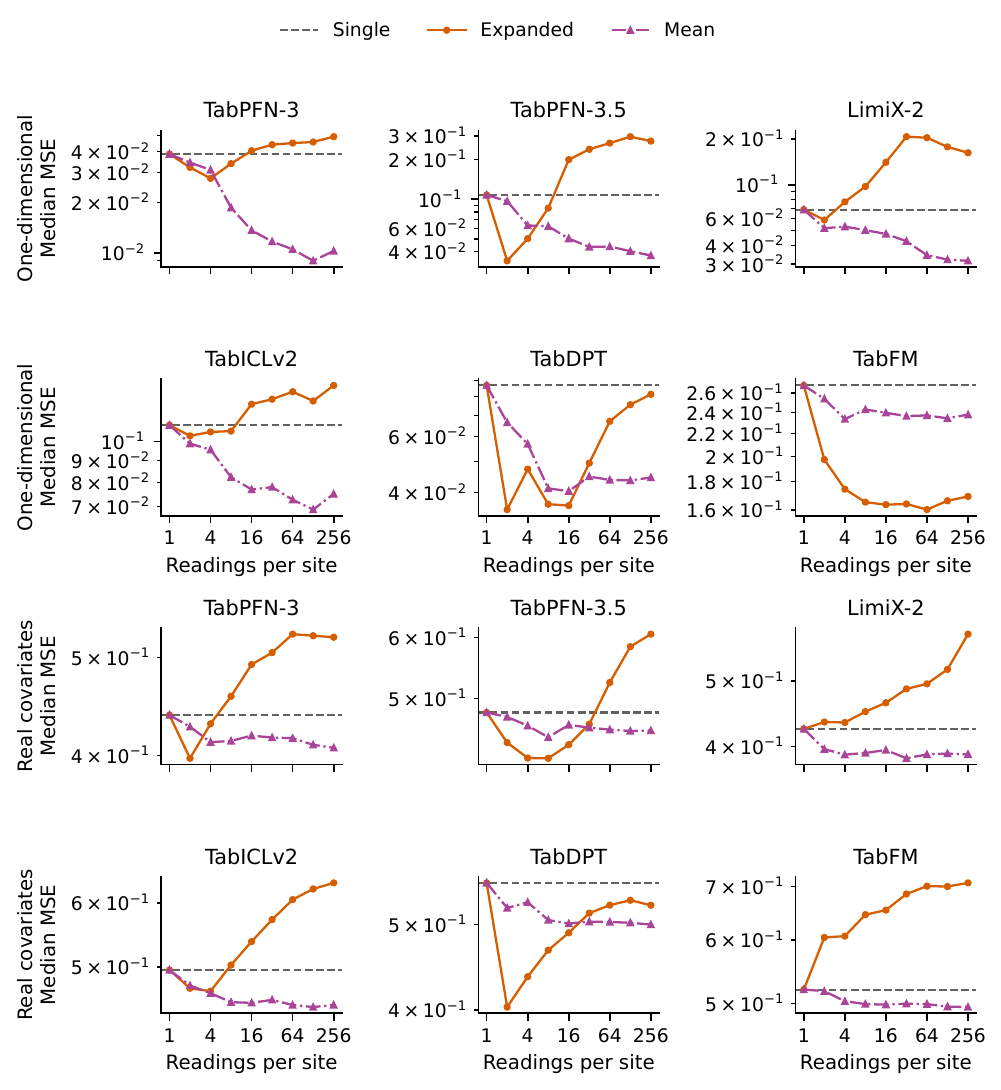}
  \caption{\Single, \Expanded and \Mean against readings per site. Marginal median
  fresh-input MSE over six one-dimensional functions (top, noise SD $0.15$, $120$
  cells) and five real covariate datasets (bottom, noise multiplier $0.25$, $100$
  cells), twenty seeds each, $m=12$. \Single is constant; curves start at its $r=1$
  value. Table~\ref{tab:three} lists the $r=4,64,256$ medians from the same fits.}
  \label{fig:three}
\end{figure*}

\begin{table}[htbp]
  \caption{Marginal median fresh-input MSE on the default designs of Figure~\ref{fig:three}:
  six one-dimensional functions at noise SD $0.15$ ($120$ cells) and five real covariate
  datasets at noise multiplier $0.25$ ($100$ cells), twenty seeds each, $m=12$.
  \Single, and \Expanded and \Mean at $4$, $64$ and $256$ readings.}
  \label{tab:three}
  \centering\small
  \begin{tabular}{llccccccc}
    \toprule
    & & & \multicolumn{3}{c}{\Expanded} & \multicolumn{3}{c}{\Mean} \\
    & model & \Single & 4 & 64 & 256 & 4 & 64 & 256 \\
    \midrule
    one-dimensional & TabPFN-v3   & 0.039 & 0.028 & 0.045 & 0.049 & 0.031 & 0.010 & 0.010 \\
    synthetic       & TabPFN-v3.5 & 0.107 & 0.050 & 0.265 & 0.274 & 0.063 & 0.044 & 0.037 \\
                    & LimiX-2    & 0.069 & 0.077 & 0.204 & 0.162 & 0.053 & 0.034 & 0.032 \\
                    & TabICLv2   & 0.109 & 0.105 & 0.131 & 0.136 & 0.096 & 0.073 & 0.075 \\
                    & TabDPT     & 0.087 & 0.047 & 0.067 & 0.081 & 0.057 & 0.044 & 0.045 \\
                    & TabFM      & 0.268 & 0.174 & 0.160 & 0.169 & 0.233 & 0.237 & 0.238 \\
    \midrule
    real            & TabPFN-v3   & 0.439 & 0.430 & 0.528 & 0.524 & 0.412 & 0.416 & 0.407 \\
    covariates      & TabPFN-v3.5 & 0.479 & 0.418 & 0.525 & 0.606 & 0.461 & 0.455 & 0.454 \\
                    & LimiX-2    & 0.424 & 0.434 & 0.495 & 0.587 & 0.388 & 0.389 & 0.389 \\
                    & TabICLv2   & 0.496 & 0.467 & 0.606 & 0.635 & 0.464 & 0.449 & 0.449 \\
                    & TabDPT     & 0.558 & 0.436 & 0.526 & 0.526 & 0.530 & 0.503 & 0.500 \\
                    & TabFM      & 0.520 & 0.607 & 0.701 & 0.707 & 0.503 & 0.499 & 0.495 \\
    \bottomrule
  \end{tabular}
\end{table}

The positive-noise sweep contains $480$ paired cells per model: six functions, four
noise SDs ($0.05$, $0.15$, $0.30$, $0.50$) and twenty seeds, with $m=12$ random sites.
At $r=64$, \Mean beats \Single in $430$, $418$, $426$, $425$, $422$ and $376$ cells
(TabPFN-v3, TabPFN-v3.5, TabICLv2, TabDPT, LimiX-2, TabFM). At $r=4$, \Expanded beats
\Mean in $243$, $315$, $214$, $338$, $128$ and $354$ of these same cells.
Counts share sites and nested reading banks; cells sharing a seed across noise levels
are dependent.

\textbf{Response to noise in the site values.} With noisy means (noise SD $0.5$, four
readings, so the noise variance of a mean is $0.0625$) the squared error of \Mean at
the sites is $0.063$ for TabFM and $0.049$ to $0.059$ for LimiX-2, which reproduce the noisy
means, and $0.023$ to $0.039$ for TabPFN-v3, which shrinks them as a Gaussian process does
($0.021$ to $0.033$). With nearly exact means (noise SD $0.05$, $64$ readings) LimiX-2 and
TabFM are within $10^{-4}$ of them, while TabDPT misses by $0.060$ on the sine and TabPFN-v3.5
by $0.063$.

\textbf{Unequal counts.} \Mean is worse on the unequal designs than on the
balanced controls with comparable reading budgets in $80$ to $96$ of $120$ cells for every model (TabPFN-v3: $0.0130$ to
$0.0265$); a Gaussian process fitted to the same means goes from $0.0141$ to $0.0272$. A
paired GP reference uses the same means and pooled reading variance in two fits, with
kernel parameters estimated in each. Giving every mean a common noise variance
$\widehat\sigma^2/\bar r$ yields $0.0678$ median MSE; giving site $i$ variance
$\widehat\sigma^2/r_i$ yields $0.0234$. The ordinary learned-noise GP above is a separate
comparison. Balanced controls round their per-site count to the nearest integer to match
the total budget; the totals can differ by at most $m/2$ readings.

The six-model count encoding is in Table~\ref{tab:e4enc}. A separate TabPFN-v3 probe with
counts alternating between $64$ and one reading also tried two sub-means per site, which changed
little, and adding $\log r_i$ to the attention logits, which raised the error to $0.08$ and
$0.75$ for alternating and clustered counts, respectively.

\FloatBarrier
\subsection{Real data, per model}

\begin{table}[htbp]
  \caption{Three real datasets at the largest count: median error over twenty splits for
  \Single, \Expanded and \Mean, and splits in which \Expanded is worse than \Single.}
  \label{tab:realmodels}
  \centering\small
  \begin{tabular}{lcccc@{\hskip 1.35em}cccc@{\hskip 1.35em}cccc}
    \toprule
    & \multicolumn{4}{c}{SGEMM ($150\times4$)} & \multicolumn{4}{c}{Kernel Tuner conv.\ ($150\times16$)} & \multicolumn{4}{c}{Hydraulic rig ($72\times8$)} \\
    & S & E & M & worse & S & E & M & worse & S & E & M & worse \\
    \midrule
    TabPFN-v3      & 0.105 & 0.161 & 0.104 & 20 & 0.223 & 0.259 & 0.218 & 15 & 0.178 & 0.074 & 0.071 & 2 \\
    TabPFN-v3.5    & 0.097 & 0.089 & 0.097 & 2  & 0.211 & 0.216 & 0.203 & 13 & 0.199 & 0.110 & 0.078 & 6 \\
    LimiX-2       & 0.091 & 0.116 & 0.093 & 20 & 0.207 & 0.222 & 0.200 & 17 & 0.141 & 0.057 & 0.045 & 4 \\
    TabICLv2      & 0.112 & 0.232 & 0.111 & 20 & 0.208 & 0.252 & 0.198 & 17 & 0.102 & 0.051 & 0.047 & 2 \\
    TabDPT        & 0.086 & 0.162 & 0.086 & 20 & 0.229 & 0.250 & 0.223 & 16 & 0.185 & 0.097 & 0.102 & 3 \\
    TabFM         & 0.151 & 0.151 & 0.151 & 18 & 0.221 & 0.240 & 0.214 & 15 & 0.151 & 0.100 & 0.043 & 8 \\
    \midrule
    Random forest & 0.349 & 0.381 & 0.350 & 13 & 0.256 & 0.347 & 0.251 & 20 & 0.133 & 0.053 & 0.044 & 1 \\
    Gradient boosting & 0.286 & 0.291 & 0.291 & 11 & 0.246 & 0.236 & 0.237 & 8 & 0.156 & 0.038 & 0.038 & 0 \\
    MLP           & 0.407 & 0.405 & 0.404 & 9  & 0.326 & 0.367 & 0.334 & 16 & 0.220 & 0.136 & 0.136 & 3 \\
    \bottomrule
  \end{tabular}
\end{table}

The Gaussian process fitted on the rows is worse than with one reading on $14$, $19$ and $17$
of $20$ splits on SGEMM and the two Kernel Tuner sets, and on $5$ of $20$ on the hydraulic
rig. On the wage data two models do badly with the rows through a shift in the level of
their predictions that we have not diagnosed (TabICLv2 $0.930$ to $2.888$, worse on $16$ of
$20$ splits; TabFM $1.011$ to $1.280$, $13$ of $20$).

\FloatBarrier
\subsection{Mechanism diagnostics on separate designs}
\label{app:independent-computation}

\paragraph{Count-dependent attention scaling.}
Query-aware scalable softmax uses the following learned query transformation
\citep{nakanishi2025ssmax,tabpfn3,tabiclv2}:
\begin{equation}
\label{eq:qassmax}
\widetilde q_n=q\odot b(\log n)\odot[1+\tanh h(q)].
\end{equation}
Here $q$ is an attention-query vector, $n$ the row count, $b,h$ learned maps and
$\odot$ elementwise multiplication.
\paragraph{Separate diagnostic designs.}

These diagnostics use designs of their own, stated in each paragraph; they are separate from the
main grid and from the representation and count-reset comparisons of Appendix~\ref{app:side}. The linear exact-copy study uses twelve noiseless sites, $64$ copies, six site
seeds and one TabPFN-v3 ensemble member. Near/middle slopes average the outer/central
fifths of each of eleven gaps; MSE is over the observed hull. Arithmetic means over the six cells
are reported. With row fingerprints disabled, native slopes remain $0.695/1.343$,
so flattening persists without that route.

\begin{table}[htbp]
\centering\small
\caption{Linear exact-copy interventions on the same table and model weights.
Twelve sites, $64$ copies, six site seeds, one ensemble member; arithmetic means over six cells.}
\label{tab:linear-count-scopes}
\begin{tabular}{lrrr}
\toprule Scope & Near slope & Middle slope & MSE\\\midrule
Native & 0.688 & 1.354 & 0.001345\\
Embedding-only & 0.803 & 1.141 & 0.000203\\
ICL only & 0.943 & 1.069 & 0.000196\\
Full & 0.996 & 0.999 & 0.000065\\\bottomrule
\end{tabular}
\end{table}

Shared context keys and values carry embedding-only reset: replacing both reproduces
that reset's result, whereas query-only re-encoding raises MSE to $0.006453$.
With query/key states fixed, a numerical count change raises mean site-attention entropy
from $2.147$ to $2.305$; this entropy is distinct from inverse-Simpson $n_{\rm eff}$.
For uniform copies, unbiased sample variance changes by the factor
$r(m-1)/(mr-1)$. Jointly fixing numerical count, fingerprints and this scaling restores
predictions to within $3.6$--$4.4\times10^{-5}$ of \Single across the six cells.

\begin{figure}[htbp]
  \centering
  \includegraphics[width=\linewidth]{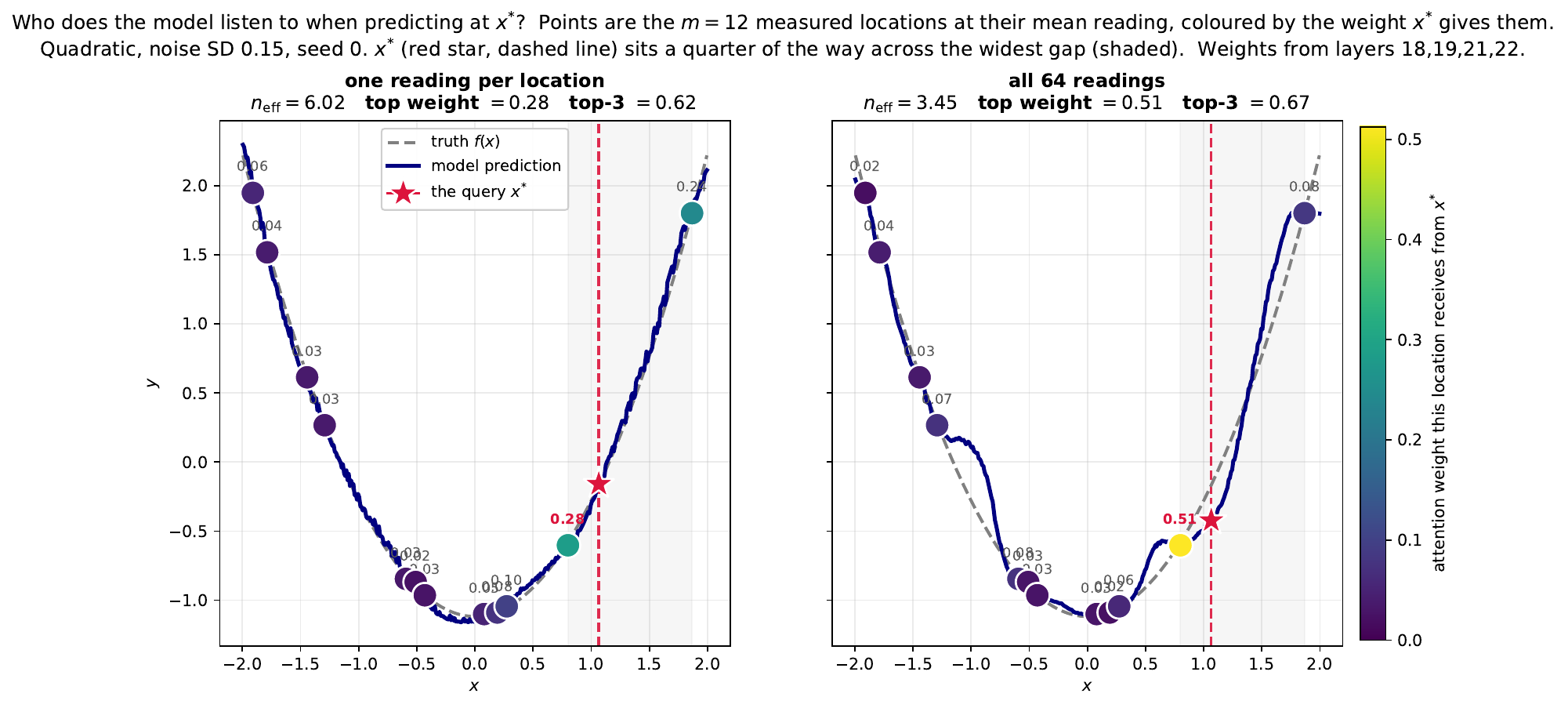}
  \caption{Inside the table of Figure~\ref{fig:failure}. The twelve sites of its left design
  (here at the grid's default noise SD $0.15$),
  drawn at their mean reading and coloured by the pooled attention weight one query (star)
  places on each. \textbf{Left:} one reading per site. \textbf{Right:} $64$ readings. One
  site's weight rises from $0.28$ to $0.51$ and the prediction at the query moves from
  $-0.16$ to $-0.42$ against a truth of $-0.17$.}
  \label{fig:weightmap}
\end{figure}

The attention illustration under full count reset in Figure~\ref{fig:attncontrol} uses independent readings, quadratic
target, SD $0.05$, seed 4102, one CPU ensemble member and $201$ queries. Context attention
is renormalized before site pooling; inverse-Simpson values are computed per captured
head/ICL operation before their median. Capture indices in this illustration denote
forward operations, rather than the zero-based physical block indices used with Table~\ref{tab:e6}. The separate window summaries use four functions
and three shared site seeds; their medians are $11.01/10.13/11.06$ for
\Single/\Expanded/full reset at SD $0.05$ and $10.98/10.59/11.45$ at SD $0.5$.
Windows and functions sharing a seed are dependent. These statistics use a different
batch and reduction from the $240$-cell attention study below.

The single-query example in Figure~\ref{fig:weightmap} belongs to the primary quadratic
design, not to the independent-reading attention illustration. It illustrates a change without
identifying attention weights as the sole cause of that change.
For the independent-reading surface in Figure~\ref{fig:surfaces}, local response is
$S_r(A,h)=\E_{x\in A,v}[\widehat f_r(x+hv)-\widehat f_r(x)]^2/
\E_{x\in A,v}[f(x+hv)-f(x)]^2$, where $v$ is a unit direction and $h$ a step relative
to local spacing. Constant-target directions use absolute prediction energy.
Transition width is the shortest contiguous neighboring-site path fraction containing
$80\%$ of total prediction variation. In the displayed design, near-site response is
$2.075/8.584/3.810$ and width $0.631/0.822/0.863$ for \Expanded/\Single/\Mean;
all response ratios exceed one. Fresh MSE is $0.410/0.467/0.564$. \Expanded has less local
variation than the other representations, not less than the target.
\begin{figure*}[t]
  \centering
  \includegraphics[width=\textwidth]{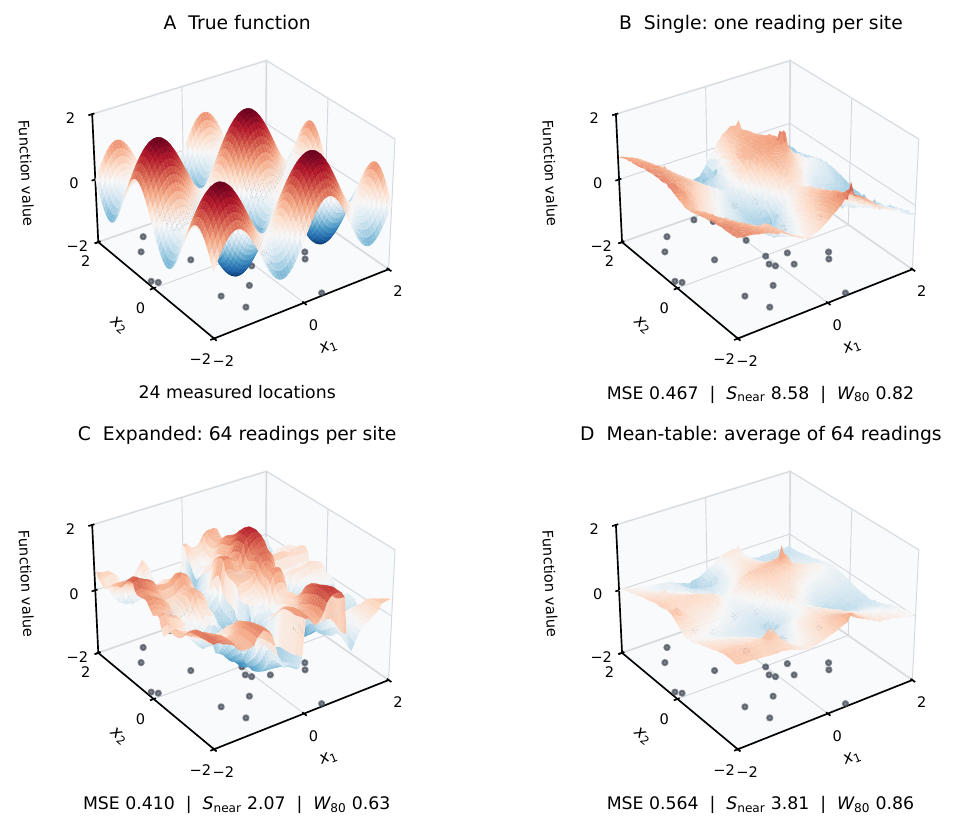}
  \caption{Two-input prediction surfaces with independent readings. Additive target, $24$ sites,
  noise SD $0.5$, $r=64$; \Single uses the first reading. Height and colour share scales on a
  $61\times61$ grid; fresh errors use $512$ separately drawn queries. $S_{\rm near}$ denotes $S_r(A,h)$ on the near-site probes; $W_{80}$ is the mean $80\%$-variation path fraction over eight neighboring-site paths. One design and one bank of
  readings, eight-member TabPFN-v3 (v3 checkpoint, \texttt{tabpfn} 9.0.0).}
  \label{fig:surfaces}
\end{figure*}

The high-dimensional exact-copy controls use additive and interaction targets, $48$ sites,
five or twenty active inputs, and $r=64$ noiseless copies. Full count reset was fixed on
six development designs before evaluation on twelve new designs. The TabDPT scale control
uses a count of $m+64$: $m$ site keys plus TabDPT's $64$ thinking rows. On twelve held-out five-dimensional designs
($48$ sites, $r=64$ noiseless copies, eight ensemble members), full count reset in TabPFN-v3 raises the
additive error from $0.599$ to $0.764$ and lowers the interaction error from $1.209$ to $1.072$;
TabDPT's scale control moves both the other way ($0.712$ to $0.565$ and $1.230$ to $1.381$), and
narrower transitions accompany lower additive and higher interaction error in both models. On
the noisy designs of the main grid's dimension axis, full count reset lowers the pooled median
fresh-input MSE relative to \Expanded at each tested dimension (Table~\ref{tab:e2}).
These interventions and the two-input illustration have different sampling
units; neither is an independent confirmation cohort.

\FloatBarrier
\subsection{Exact-copy contrasts on the main grid}
\label{app:copy-contrast}
At noise SD zero, \Expanded supplies $64$ identical latent values at each of twelve random
sites. Table~\ref{tab:copy-contrast} pairs \Expanded and \Single within each function--seed
cell, then takes the median of the twenty ratios. Both backends use their main-grid
package configurations, including TabDPT's default ensemble and full retrieved context
and TabFM's full context (Appendix~\ref{app:protocol}); these differ from the one-member
attention probes and the high-dimensional scale intervention above.

\begin{table}[htbp]
  \centering\small
  \caption{Main-grid exact copies: median paired fresh-input MSE ratio
  \Expanded/\Single, twelve sites, $64$ copies, twenty seed cells per function.}
  \label{tab:copy-contrast}
  \begin{tabular}{lrr}
    \toprule
    function & TabDPT & TabFM \\
    \midrule
    quadratic & 1.243 & 1.000 \\
    absolute value & 0.712 & 0.999 \\
    linear & 1.420 & 1.009 \\
    hyperbolic tangent & 1.023 & 1.007 \\
    step & 0.932 & 1.003 \\
    sine & 0.484 & 0.999 \\
    \bottomrule
  \end{tabular}
\end{table}

\FloatBarrier
\subsection{Inside TabPFN-v3}\label{app:inside-tabpfn3}

Prior work links spectral capacity to context sample count in TabPFN~2 and~2.5 \citep{zheng2025} and uses
activation patching and head ablation to localize computation in TabPFN~2.5 \citep{gupta2026}.
Our TabPFN-v3 diagnostics hold input sites fixed while changing repetition or internal count inputs.

\begin{figure}[htbp]
  \centering
  \includegraphics[width=0.75\textwidth]{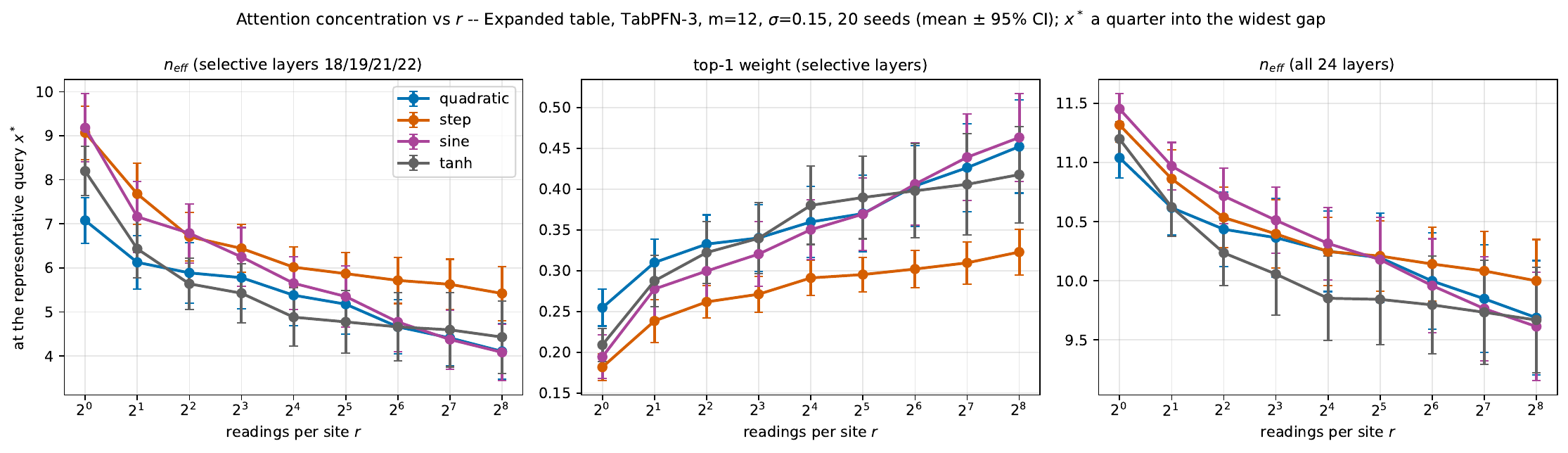}
  \caption{Effective number of attended sites and top pooled weight at a representative query
  as $r$ grows, TabPFN-v3, \Expanded table, $m=12$, $\sigma=0.15$; mean and $95\%$ interval over
  twenty seeds. Selective layers are $18,19,21,22$; the all-layer curves pool all $24$ ICL layers.}
  \label{fig:attn_r}
\end{figure}

\textbf{Readout.} With $\omega_i(x)$ the pooled attention on site $i$ and $\Ybar_i$ its
mean reading, the reconstruction $\sum_i\omega_i(x)\Ybar_i$ is correlated with the model's
output over the query grid at $0.996$ on the one-dimensional designs and at $0.974$ to
$0.983$ on real covariates in eight to eleven dimensions, using every layer in which rows
meet. Weights shuffled across sites, or taken from another query, reconstruct at
$|\rho|<0.02$. The relation is affine and not exact: the slope of output on reconstruction is
$2.0$ with one reading and $1.27$ with $64$, because attention averages value vectors and
not site means.

\textbf{Concentration.} The $240$-cell summaries (three functions, eight noise levels, ten seeds, $r=64$) are medians over cells using ICL layers $18,19,21,22$, selected for spatial selectivity under \Single. The weight on the most attended site rises from $0.26$ to $0.44$ and
$n_{\mathrm{eff}}$ falls from $6.6$ to $4.1$. On twenty seeds the fall is monotone in
$\log r$ (quadratic: $7.1$ to $4.1$; lower at $r=256$ than at $r=1$ in $77$ of $80$
function--seed pairs), and the largest step is the first (Figure~\ref{fig:attn_r}). Under
\Mean the same query moves from $6.8$ to $6.2$.

\textbf{Interventions on the attention.}
The transplant probe uses TabPFN-v3 (\texttt{tabpfn} 8.5.0, the default v3 regressor
checkpoint), one ensemble member, fingerprints on and float32 inference. It compares
$r=1$ and $r=64$ at $m=12$ sites: four functions (quadratic, tanh, step, sine), noise SDs
$0.05,0.15,0.30$ and six seeds give $72$ cells. MSE is averaged over $400$ fresh query
inputs against the true function. The all-layer intervention replaces the test-query
attention at each of the $24$ ICL layers with \Single's per-head site weights, spread
uniformly over each site's $64$ rows; the other computations use the \Expanded pass.
Transplanting each pass's own weights changes predictions by at most
$6.7\times10^{-5}$ over these cells.

\Expanded is worse than \Single in $40/72$ cells. All-layer transplant improves $18$ of
those $40$: $12/12$ on step, $0/16$ on quadratic, $2/4$ on tanh and $4/8$ on sine.
Restricting transplant to layers $18,19,21,22$ improves $29/40$, but closes at least half
the error gap in only $2/40$; all-layer transplant does so in $11/40$.
These layer indices are zero-based. A separate attention-temperature experiment that
fully undoes concentration beats plain averaging in $5$ of $20$ cells.

\textbf{Values and the embedding stage.} Within one site the value vectors of the rows
spread $1.6$ times the spread between sites at the readout layers. The embedding stage
does not create this: rows of one site enter it $0.05$ apart, in units of the between-site
spread, and leave it $0.003$ apart; the in-context layers re-separate them, to $0.72$ by the
last layer (Figure~\ref{fig:embedding}). With identical rows, the three count modules of the
embedding stage displace each site by $0.52$ of the between-site spread. Under embedding-only
reset, the in-context modules rebuild the displacement from $0.23$ at the first layer to
$0.95$ at the last.

A separate risk probe pools $144$ cells: four functions, noise SDs $0$, $0.05$ and $0.30$,
six seeds, with and without fingerprints. It uses one ensemble member and float32 inference,
giving $576$ fits across four arms. Embedding-only reset lowers error in $121/144$ cells,
with median paired MSE ratio $0.68$ to \Expanded. Function medians fall from $0.183$ to
$0.066$ on quadratic and $0.011$ to $0.003$ on tanh, but rise from $0.280$ to $0.283$ on
sine. Full count reset has median paired ratio $0.95$ to \Single and improves on \Single
in $84/144$ cells. Table~\ref{tab:e2} evaluates these resets on the main designs.

\begin{figure}[htbp]
  \centering
  \includegraphics[width=\textwidth]{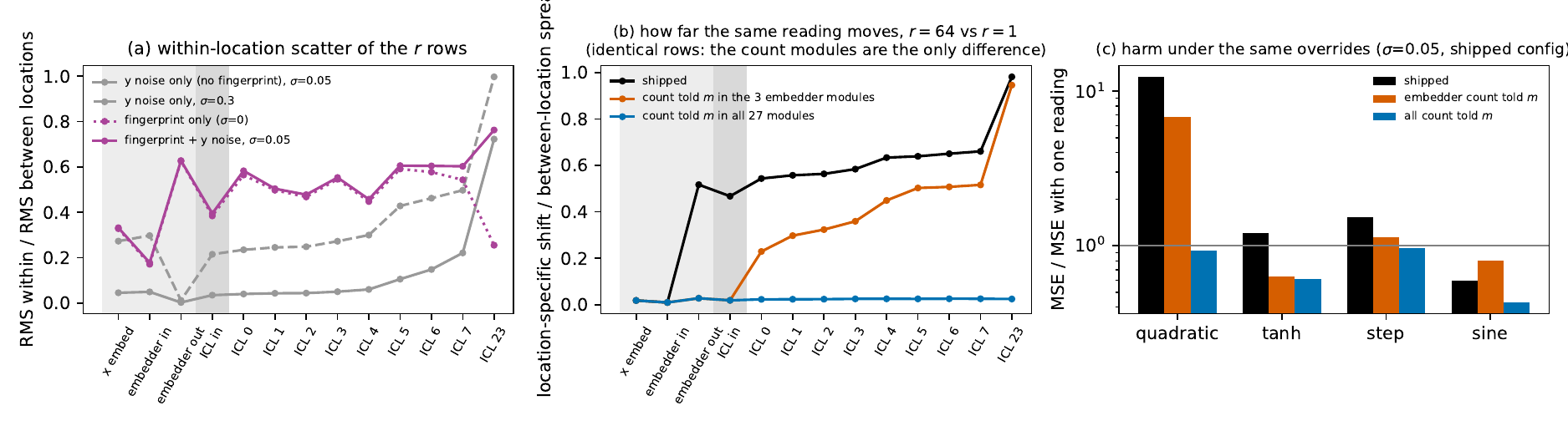}
  \caption{The embedding stage of TabPFN-v3 under repetition ($m=12$, $r=64$, four functions,
  six seeds, medians). Spread of one site's rows relative to the spread between sites along
  the forward pass; displacement of a site between the $r=64$ and $r=1$ contexts under count
  overrides; and the resulting error relative to \Single.}
  \label{fig:embedding}
\end{figure}

\textbf{Bias and variance.} The decomposition over independent reading banks for all six
models is in Table~\ref{tab:banks} and Table~\ref{tab:e3}.

\FloatBarrier
\subsection{Fine-tuning trials}
\label{app:finetune}

The best fine-tuned \Expanded result improves modestly while losing much of the model's
benefit from averaging. In the first trial \Expanded MSE falls from $0.0648$ to $0.0621$,
with gains in $23$ of $36$ cells, but \Mean MSE rises from $0.0079$ to $0.0481$. Even this best
trained \Expanded median remains about eight times the shipped \Mean baseline. The all-layer
trial raises \Expanded MSE to $0.0730$ and \Mean MSE to $0.0307$; the values-only adapter raises
\Expanded MSE from its own shipped baseline of $0.0655$ to $0.0727$.

\begin{table}[htbp]
  \caption{Three fine-tuning trials on TabPFN-v3, each with one training seed. Median fresh-input
  MSE over $36$ paired cells: six one-dimensional functions times six evaluation seeds
  ($0$--$5$), $m=12$, $r=64$, noise SD $0.05$. Each trial uses its own shipped baseline under
  the same inference settings; wins compare trained \Expanded with shipped \Expanded.}
  \label{tab:finetune}
  \centering\small
  \setlength{\tabcolsep}{4pt}
  \begin{tabular}{lccccc}
    \toprule
    & \multicolumn{2}{c}{\Expanded} & \multicolumn{2}{c}{\Mean} & \\
    trial & shipped & trained & shipped & trained & wins \\
    \midrule
    Four ICL layers + count + head & 0.0648 & 0.0621 & 0.0079 & 0.0481 & 23/36 \\
    All ICL layers + count + head  & 0.0648 & 0.0730 & 0.0079 & 0.0307 & 17/36 \\
    Values-only adapter           & 0.0655 & 0.0727 & 0.0076 & 0.0076 & 17/36 \\
    \bottomrule
  \end{tabular}
\end{table}

\textbf{Training modules and budgets.}
The first trial trains in-context layers $18$, $19$, $21$ and $22$, all $27$
\texttt{SoftmaxScalingMLP} count modules, and the output head
(\texttt{output\_norm}, \texttt{output\_projection}): $15.08$ million parameters,
$8{,}000$ tables, learning rate $5\times10^{-6}$, $100$ warm-up tables and weight decay $0.01$.
The second trains all $24$ in-context layers, all count modules and the same head
($57.04$ million parameters), with a richer prior, low-noise anchor and curriculum:
$12{,}000$ tables, learning rate $10^{-6}$, $200$ warm-up tables and weight decay $0.01$.
Both freeze other encoder parameters. The third trains a zero-initialised values-only
adapter ($1{,}176$ parameters) on a frozen base for $10{,}000$ tables. Each row from site $i$ receives $\log r_i$. Its training configuration specifies AdamW with default learning rate $10^{-4}$, $200$ warm-up
tables and zero weight decay. Paired evaluation uses eight ensemble members in all trials,
with float32 for the adapter and each earlier trial's matched shipped inference settings.

\textbf{Validation and evaluation.}
Ordinary validation tables contain one noisy response per site ($r=1$);
\Mean performance is evaluated separately using averages of $64$ readings.
Checkpoints are validated every $2{,}000$ tables. The first trial uses $30$ \Expanded and
$30$ ordinary tables, selecting the lowest mean \Expanded MSE subject to ordinary-table
mean MSE at most $1.05$ times shipped. The second uses $40$ short-length-scale GP \Expanded
and $40$ ordinary tables, with the same rule plus improved \Expanded MSE. The adapter uses
$40$ \Expanded and $40$ ordinary tables, selecting the lowest mean \Expanded MSE.
The selected checkpoints are at $8{,}000$, $12{,}000$ and $10{,}000$ tables, respectively;
the $36$ evaluation cells are separate from these banks. Each recipe uses one training seed.

The values-only adapter preserves one-reading predictions to $3\times10^{-6}$.
A logit-bias-only adapter was not trained because the fused-attention kernel failed
at its first backward pass. The training target in all three is the latent mean, so a narrower predictive band after
training does not indicate calibration for a future reading.

\end{document}